\documentclass[letterpaper]{article} 
\usepackage{aaai2026_submission}  
\usepackage{times}  
\usepackage{helvet}  
\usepackage{courier}  
\usepackage[hyphens]{url}  
\usepackage{graphicx} 
\usepackage{natbib}  
\usepackage{caption} 
\usepackage{microtype}
\usepackage{graphicx}
\usepackage{booktabs} 
\usepackage{xargs}
\usepackage[utf8]{inputenc} 
\usepackage[T1]{fontenc}    

\usepackage{caption}
\usepackage{amsmath}
\usepackage{amssymb}
\usepackage{mathtools}
\usepackage{amsthm}
\usepackage{enumerate}
\usepackage{enumitem}
\usepackage{threeparttable}
\usepackage{subcaption}
\usepackage{tcolorbox}
\usepackage{pifont}
\usepackage[table,xcdraw]{xcolor}
\usepackage{nicefrac}
\usepackage{booktabs}
\usepackage{algorithm}
\usepackage{algorithmic}
\usepackage{adjustbox}
\usepackage{caption}
\usepackage{array} 
\usepackage{multirow}
\usepackage{adjustbox}

\usepackage[capitalize,noabbrev]{cleveref}

\floatstyle{ruled}
\restylefloat{algorithm}

\theoremstyle{plain}
\newtheorem{theorem}{Theorem}[section]

\newtheorem{lemma}[theorem]{Lemma}
\newtheorem{corollary}[theorem]{Corollary}
\theoremstyle{definition}

\newtheorem{assumption}[theorem]{Assumption}
\newlist{assumlist}{enumerate}{1}
\setlist[assumlist,1]{
    label=(\alph*),
    ref=\theassumption(\alph*),
    align=left,
    leftmargin=0.4cm
}
\theoremstyle{remark}
\newtheorem{remark}[theorem]{Remark}

\definecolor{mygreen}{rgb}{0.75,1,0.75}
\definecolor{darkgreen}{rgb}{0.0, 0.5, 0.0}
\definecolor{mydarkred}{RGB}{192,47,25}
\definecolor{lightpurple}{rgb}{0.7, 0.5, 0.9}

\title{Bant: Byzantine Antidote via Trial Function and Trust Scores}

\author{
    Gleb Molodtsov\textsuperscript{\rm 1}\equalcontrib
    ,
    Daniil Medyakov\textsuperscript{\rm 1}\equalcontrib
    ,
    Sergey Skorik\textsuperscript{\rm 2},
    Nikolas Khachaturov\textsuperscript{\rm 3},
    Shahane Tigranyan\textsuperscript{\rm 4},
    Vladimir Aletov\textsuperscript{\rm 1},
    Aram Avetisyan\textsuperscript{\rm 2},
    Martin Tak\'a\v{c}\textsuperscript{\rm5},
    Aleksandr Beznosikov\textsuperscript{\rm 1}
}
\affiliations{
    \textsuperscript{\rm 1}Basic Research of Artificial Intelligence Laboratory (BRAIn Lab)\\
    \textsuperscript{\rm 2}Trusted AI Research Center, RAS\\
    \textsuperscript{\rm 3}Russian-Armenian University\\
    \textsuperscript{\rm 4}Institute for Informatics and Automation Problems, NAS RA\\
    \textsuperscript{\rm 5}Mohamed bin Zayed University of Artificial Intelligence\\
    gleb.molodcov@gmail.com 
}

\usepackage{bibentry}

\begin{document}

\maketitle

\begin{abstract}
Recent advancements in machine learning have improved performance while also increasing computational demands.
While federated and distributed setups address these issues, their structures remain vulnerable to malicious influences. In this paper, we address a specific threat: Byzantine attacks, wherein compromised clients inject adversarial updates to derail global convergence. 
We combine the concept of trust scores with trial function methodology to dynamically filter outliers. 
Our methods address the critical limitations of previous approaches, allowing operation even when Byzantine nodes are in the majority. 
Moreover, our algorithms adapt to widely used scaled methods such as \textsc{Adam} and \textsc{RMSProp}, as well as practical scenarios, including local training and partial participation. We validate the robustness of our methods by conducting extensive experiments on both public datasets and private ECG data collected from medical institutions. Furthermore, we provide a broad theoretical analysis of our algorithms and their extensions to the aforementioned practical setups. The convergence guaranties of our methods are comparable to those of classical algorithms developed without Byzantine interference.
\end{abstract}

\begin{links}
    \link{Code}{https://github.com/brain-lab-research/Bant}
\end{links}

\section{Introduction} \label{sec:intro}


As the field of machine learning expands, researchers are confronted with challenges stemming from increasingly complex models and larger computational demands. To address these issues, distributed and federated learning scenarios have been developed \citep{kairouz2021advances, konevcny2016federated, li2020federated}. These approaches are crucial for a wide range of tasks \citep{smith2017federated, mcmahan2017communication, verbraeken2020survey}, yet, they introduce several complications.
Making the training process multi-node leads to threats related to data storage and transmission. These vulnerabilities manifest as device malfunctions, incorrect data relays, or even initial data corruption \citep{biggio2012poisoning}. \citep{wang2023adversarial} demonstrated how adversarial attacks on data integrity compromise training. In the midst of them are Byzantine attacks. They occur in networks where certain workers, known as Byzantines, may corrupt data \citep{blanchard2017machine}.

\noindent This paper specifically examines the threat of Byzantine attacks. We highlight the limitations of existing protection mechanisms, particularly their reliance on strong assumptions. In response, we propose an approach that leverages diverse concepts to develop a universal method, free from these constraints and applicable to practical scenarios.




\noindent \textbf{Setup.}
We examine the problem often encountered in distributed machine learning, with $\mathcal{D}_i$ representing an unknown distribution of the training sample data on the $i$-th device:
\begin{align}
\label{eq:setting_dist}
\underset{x \in \mathbb{R}^d}{\min} \bigl[f(x) = \frac{1}{n}\sum\limits_{i = 1}^n \mathbb E_{\xi_i\sim \mathcal{D}_i}\left[f_{\xi_i} (x)\right]\bigr].
\end{align}
We consider a setup involving $n$ workers connected to a central server. These workers are divided into two categories: Good (or Honest) workers, indicated by $\mathcal{G}$, and Byzantines, indicated by $\mathcal{B}$. At each iteration $t$, the sets $\mathcal{G}(t)$ and $\mathcal{B}(t)$ are redefined. This allows the composition of honest and Byzantine workers to vary dynamically over time. 
At every step, we assume that the set of honest workers $\mathcal{G}(t)$ is nonempty, i.e., $G(t) \mathrel{\mathop:}= |\mathcal{G}(t)| \geqslant 1$.
During training, the number of Byzantines at each iteration remains unknown. This number is used only in the theoretical analysis of the worst-case scenario.

\section{Related work} 
\textbf{Sophisticated attacks.} Methods resistant to Byzantine attacks are crucial for solving optimization problems. Classical methods for distributed optimization, such as \textsc{SGD} \citep{robbins1951stochastic, bottou2012stochastic, recht2013parallel, mcmahan2017communication}, \textsc{Adam} \citep{kingma2014adam, reddi2020adaptive}, and \textsc{Scaffold} \citep{karimireddy2020scaffold}, average the received gradients or models. However, they \textit{cease to operate when even a single Byzantine worker appears in the network} \citep{blanchard2017machine}. 
Given the critical importance of this problem, numerous publications have addressed it \citep{feng2014distributed, damaskinos2019aggregathor}. Initial approaches proposed robust aggregation rules for data from devices, such as \textsc{Coordinate-wise median, Trimmed Mean} \citep{yin2018byzantine}, \textsc{Krum} \citep{blanchard2017machine}, \textsc{Bulyan} \citep{mhamdi2018hidden}. However, sophisticated attacks, such as \textsc{ALIE} \citep{baruch2019little} or \textsc{Inner Product Manipulation} \citep{xie2020fall}, managed to \textit{circumvent these aggregation rules} by shifting the mean they sought to find. 

\noindent \textbf{Clipping, Momentum, and Variance Reduction.} Aggregation rules are \textit{non-robust even in the absence of attacks}. For example, this occurs in cases of imbalanced classes. This issue was addressed in \citep{karimireddy2021learning}, where \textsc{Centered Clip (CC)} technique was revealed.
Additionally, the authors highlighted that the aforementioned Byzantine-robust methods cannot converge with any predetermined accuracy. Given the importance of this issue, they added client momentum, effectively addressing the problem. Another approach to combat the Byzantines is the application of the variance reduction technique. Originally introduced to eliminate irreducible errors in stochastic methods, it was subsequently proposed as an effective way to mitigate the presence of noise in computing stochastic gradient estimators \citep{gorbunov2023variance}. Then this idea was developed in \citep{malinovsky2023byzantine}. 

\noindent These methods demonstrated significant improvements in resilience against Byzantine attacks. However, they also contain notable limitations. First, variance reduction methods exhibit \textit{moderate convergence} in deep learning applications and are prone to \textit{overfitting} \citep{defazio2019ineffectiveness}. Additionally, all aforementioned approaches \textit{suffer from the requirement that the majority of devices must be honest}.

\noindent \textbf{Validation tests.} Another approach to achieving solutions with any specified accuracy involves techniques such as validation tests \citep{alistarh2018byzantine, allen2020byzantine} or computation checks \citep{gorbunov2022secure}. Nevertheless, they still require a \textit{majority of honest devices} and rely on \textit{strict assumptions} in the analysis \citep{alistarh2018byzantine, allen2020byzantine}.

\noindent \textbf{Heterogeneous setup.} While much work on the Byzantines focused on the distributed case with homogeneous data, a different series of papers allowed for data heterogeneity \citep{wu2020federated, el2021collaborative, data2021byzantine, nguyen2022flame}. This corresponds, for example, to the federated setting. Methods were primarily built around robust aggregation \citep{karimireddy2021byzantine, chang2019cronus, data2021byzantine, allouah2024robust, dorfmandynamic, allouah2024boosting}, and variance reduction techniques \citep{allouah2023fixing}. One of the most advanced approaches assigned coefficients (\textsc{trust scores}) to clients based on their reliability, using these scores to perform gradient steps \citep{cao2020fltrust, yan2024recess}. These studies provided a foundation for Byzantine-robust optimization in the federated setup, yet suffered from \textit{the aforementioned drawbacks}.

\noindent \textbf{Majority of attackers -- trial function.} To avoid requiring a majority of honest workers, several methods leverage server-held ground-truth data to filter compromised updates. \textsc{Zeno} \citep{xie2019zeno} implements this idea by maintaining a small validation set on the server and scoring each incoming gradient against the gradient computed on that set. We define the function evaluated on the validation set as the \textsc{trial function}. As for \textsc{Zeno}, it uses the trial function to compute trust scores but primarily aggregates updates through simple averaging, down-weighting or excluding devices with very low trust. As a result, its performance depends strongly on the devices that are treated as trusted. \citep{cao2019distributed} proposed an alternative that filters updates by comparing them to a noisy gradient approximation computed on a small dataset, which is effectively equivalent to using a validation set as in \textsc{Zeno}.
The idea of using a small server-side dataset to validate local updates was also utilized in \textsc{SageFlow} \citep{park2021sageflow}, which handles majority adversarial clients. However, this method \textit{performs poorly} under an attack ratio of 60\%. Moreover, all aggregation schemes that discard clients based on trust scores remain \textit{sensitive to hyperparameters}, as they require threshold choices to label devices as malicious. It requires tuning, and we discuss such instability regarding \textsc{Zeno} in Appendix, Table \ref{tab:ecggmean_zeno}.

\noindent The authors of \citep{xie2019zeno, cao2019distributed, park2021sageflow} extended their results to address data heterogeneity by requiring the server to possess a representative sample of all device data. However, this assumption is \textit{unrealistic} in real-world scenarios, undermining the fundamental achievements of federated learning with respect to privacy. Approaches \citep{guo2021siren, guo2024siren+} also \textit{accumulate user data} on the server, raising doubts about applicability.

\noindent \textbf{Other shortcomings of previous research.} In addition to aforementioned limitations, many studies in this field are predominantly \textit{heuristic and lack rigorous theoretical analysis} \citep{yan2024recess, guo2021siren, guo2024siren+, chang2019cronus, xu2020reputation, rodriguez2022dynamic, nguyen2022flame, zhang2022fldetector, huangself}. Moreover, in some studies, the \textit{practical component seems flawed} due to the absence of experiments assessing test accuracy \citep{gorbunov2023variance}. Furthermore, theoretical settings often do not align with practical aspects. For instance, in \citep{cao2020fltrust},\textit{ homogeneous data sampling is assumed} while focusing on the federated learning. Besides, the analysis of \textsc{Sageflow} is confined to an \textit{unrealistic strongly convex setup}.

\vspace{0.5mm} \noindent Furthermore, when addressing problems in the distributed or federated setups, local methods \citep{woodworth2020minibatch, khaled2020tighter, gorbunov2021local, nguyen2022flame}, as well as the partial participation scenario \citep{yang2021achieving, kairouz2021advances, sadiev2022decentralized, nguyen2022flame}, is typically assumed. While these options improve computational efficiency and reduce data transmission overhead, only a few studies \citep{data2021byzantine, malinovsky2023byzantine, allouahbyzantine, dorfmandynamic} address these aspects, whereas the majority of works do not. In addition, research is often limited to \textsc{SGD}-like methods, \textit{neglecting adaptive algorithms} such as \textsc{Adam} \citep{kingma2014adam} and \textsc{RMSProp} \citep{tieleman2012lecture}, which are widely used in machine learning.
Given the challenges and gaps identified in the existing literature, we aim to advance the trust score methodology and trial function concept to enhance defense mechanisms.



\noindent \subsection*{Contributions} Our main results are summarized as follows.\\
$\bullet$ \textbf{Combine trust scores with the trial function approach.} Trial loss is based on a subset of the training data stored on the server. 
Weights are assigned to the gradients sent from each device based on the extent to which these gradients reduce the trial function in each iteration. In real networks, honest stochastic gradients may increase the target loss. We account for this by incorporating weights from the previous epoch and a momentum parameter for more stable convergence.\\
$\bullet$ \textbf{Milder assumptions.} Unlike most existing studies, our approach requires only one reliably honest worker instead of a majority. Moreover, unlike previous trial function-based methods that assume data homogeneity \citep{cao2020fltrust, gorbunov2022secure}, our algorithms operate under the more realistic assumption of data similarity in federated learning.\\
$\bullet$ \textbf{Extensions.} 
We adapt our algorithms to important scenarios that are often overlooked in research.
    \begin{itemize}[leftmargin=0.8cm]
    \item[$(a)$] \textbf{Local methods.}
     In our work, we propose utilizing Local SGD to address the high communication costs typically associated with distributed training.
     \item[$(b)$] \textbf{Partial participation.} Our algorithms incorporate the option for partial participation. Thus, devices may not participate in every learning step, and the attackers may vary across iterations.
     \item[$(c)$] \textbf{Adaptive methods.} In this work, we extend our analysis to adaptive algorithms (e.g., \textsc{Adam} and \textsc{RMSProp}), which are widely used in machine learning.
\end{itemize}
$\bullet$ \textbf{Convergence guarantees.} We prove upper bounds on the convergence rates of the main methods and extensions presented for the smooth problem under various assumptions regarding the convexity of the target function (strong convexity, convexity, non-convexity).\\
$\bullet$ \textbf{Experiments.} We demonstrate the superiority of our method in both previously studied attacks and scenarios where alternative methods fail. Our experiments are performed on the CIFAR-10 dataset and real ECG data, utilizing \textsc{ResNet-18} and \textsc{ResNet-1D18} neural networks, respectively. Additionally, we validate our approach to Learning-to-Rank tasks by training a Transformer-based ranking model.
 
\section{Methodology}

\noindent To tackle Byzantine attacks, we introduce a pivotal component of our methodology -- a trial loss function $\hat{f}$. In the homogeneous setting \eqref{eq:setting_dist} with $\mathcal{D}_i = \mathcal{D}$, we take a separate sample from $\mathcal{D}$, but in a smaller volume than the entire dataset.
The trial function calculated on this data forms $\hat f$. 
Under the heterogeneous data scenario, we sample from the distribution $\mathcal{D}_1$ on the server to obtain delayed data for $\hat f$ (indexing the server does not violate generality; we further consider it a device with index 1). Formally, we can write the trial function as $\hat{f}(x) = \frac{1}{N} \sum_{i = 1}^N f_{1}(x, \xi _i),$ where $N$ is the number of samples in $\hat{f}$. This function is stored on the server (obviously, an honest device). The sample distribution of the trial function is similar to the entire distribution $\mathcal{D}$ due to the property of data similarity. Besides, in practical scenarios, a server may not be able to share the entire dataset, providing only a sample of size \( N \). Depending on the size of this sample, \( f_1 \) may differ from \( \hat{f} \). Nevertheless, the larger the volume of this delayed sample, the closer $\hat{f}$ approximates the function $f_1$ (discussed in Lemma \ref{lemmaShalevShwartz} in Appendix).
A small public or synthetic trial dataset is a practical assumption in Byzantine-robust federated learning. Methods such as FLTrust \citep{cao2020fltrust} and Zeno \citep{xie2019zeno} utilize such datasets.
Here we outline the assumptions under which we establish the convergence rates.

\begin{assumption}\label{as1}
    The function $\hat{f}$ is $L$-smooth, i.e., $\|\nabla \hat{f}(x) - \nabla \hat{f}(y)\| \leqslant L\|x - y\|$ ~for any~ $x, y \in \mathbb{R}^d$.
\end{assumption}
\begin{assumption}\label{as2}
    The function $\hat{f}$ is:
    \begin{assumlist}
        \item \label{as2stronglyconvex}
        \textbf{$\mu$-strongly convex} if for all $x, y \in \mathbb{R}^d$, it satisfies:
        $
        \textstyle{\hat{f}(y) \geqslant \hat{f}(x) + \langle \nabla \hat{f}(x), y - x \rangle + \frac{\mu}{2}\|y - x\|^2.}$
        \item \label{as2convex}
        \textbf{convex} if for all $x, y \in \mathbb{R}^d$, it satisfies:
        $\textstyle{\hat{f}(y) \geqslant \hat{f}(x) + \langle \nabla \hat{f}(x), y - x \rangle.}$
        \item \label{as2nonconvex}
        \textbf{non-convex} if it has at least one (not necessarily unique) minimum, i.e., $\textstyle{
        \hat{f}(\hat{x}^*) = \inf\limits_{x \in \mathbb{R}^d} \hat{f}(x) > -\infty.}$
    \end{assumlist}
\end{assumption}
\begin{assumption}\label{as3}
Each worker $i\in \mathcal{G}(t)$ has access to an independent and unbiased stochastic gradient with 
$\mathbb{E}[g_i(x, \xi_i)] = \nabla f_i(x)$ and its variance is bounded by $\sigma^2$: 
$$
    \mathbb{E}\|g_i(x, \xi_i) - \nabla f_i(x)\|^2 \leqslant \sigma^2, \quad \text{for all } x \in \mathbb{R}^d.
$$
\end{assumption}
\noindent We also assume data similarity -- a common premise for ensuring convergence in Byzantine literature \citep{karimireddy2021byzantine, gorbunov2023variance, yan2024recess}.
\begin{assumption}\label{as4}
We assume data similarity in the following way: good clients possess $(\delta_1, \delta_2)$-heterogeneous local loss functions for some $\delta_1 \geqslant 0$ and $\delta_2 \geqslant 0$, such that for all $x \in \mathbb{R}^d$, the following holds:
\begin{align*}
    \|\nabla f_i(x) - \nabla f(x)\|^2 \leqslant \delta_1 + \delta_2 \|\nabla f(x)\|^2 ~~~~\forall i \in \mathcal{G}(t).
        \end{align*}
\end{assumption}
\noindent In over-parameterized models, introducing a positive $\delta_2$ can sometimes reduce the value of $\delta_1$.
Several studies have explored heterogeneous scenarios in which honest workers handle distinct local functions \citep{wu2020federated, karimireddy2021byzantine}. Note that achieving a predefined accuracy becomes feasible in the presence of Byzantines only when heterogeneity is limited to \( \delta_2 \)-bounded settings $(\delta_1 = 0) $\citep{karimireddy2021byzantine}. Under a more general assumption, which we utilize, the term with $\delta_1$ inadvertently appears in the estimate.

\begin{assumption}\label{as5}  
Byzantine workers are assumed to be omniscient, i.e., they have access to the computations performed by the other workers. 
\end{assumption}


\section{Algorithms and Convergence Analysis} \label{sec:main}


\subsection{First method: \texttt{Bant}}\label{sec:main1}


In this section, we introduce our method, termed Byzantine ANTidote (\texttt{Bant}) -- \text{Algorithm \ref{alg2}}. Our method relies on the core idea of assigning trust scores to devices.
We integrate this with the concept of a trial function by aggregating the stochastic gradients $g_i^t$ of devices with their respective weights $w_i^t$. To find the latter, we first calculate the contribution coefficients for each worker $i$ at each step: $~\theta_i^t = \hat{f}(x^t) -\hat{f} (x^t -\gamma g_i^t)~$. These coefficients demonstrate how the $i$-th device affects convergence. If $\theta_i^t > 0$, the stochastic gradient minimizes trial loss and is assigned a weight. Otherwise, it is assigned a weight of zero (Line \ref{itr:line10} in Algorithm \ref{alg2}). We ensure non-negativity with $\left[\theta_i^t\right]_0 = \max\{\theta_i^t, 0\}$ and normalize to provide a total weight of 1.

\noindent To address stochastic gradient instability, which can increase the loss function, we introduce a momentum parameter for the weights (Line \ref{itr:line10}). If all gradients increase the loss, they are assigned zero weights, thereby stopping the minimization process even in the absence of Byzantine devices in the network. 
\begin{algorithm}[t]
\caption{\texttt{Bant}}
\label{alg2}
\begin{algorithmic}[1]
\STATE \textbf{Input:} Starting point $x^0 \in \mathbb{R}^d$, $\omega_i^0 = \nicefrac{1}{n} ~ \forall i$
\STATE \textbf{Parameters:} Stepsize $\gamma > 0$, momentum $\beta \in (0, 1]$
\FOR{$t = 0, 1, 2, \ldots, T-1$}
    \STATE Server sends $x^t$ to each worker
    \FORALL{workers $i = 1, 2, \ldots, n$ in parallel}
        \STATE Generate $\xi_i^t$ independently
        \STATE Compute stochastic gradient $g_i^t = g_i(x^t, \xi_i^t)$
        \STATE Send $g_i^t$ to server
    \ENDFOR
    \STATE $\omega_i^{t} = (1 - \beta)\omega_i^{t-1} + \beta\frac{[\hat{f} (x^t) - \hat{f} (x^t - \gamma  g_i^t)]_0}{\sum_{j = 1}^n [\hat{f} (x^t) - \hat{f} (x^t - \gamma  g_j^t)]_0}$ \label{itr:line10} 
    \IF {each $[\hat{f} (x^t) - \hat{f} (x^t - \gamma  g_i^t)]_0 = 0$}
        \STATE $\omega_i^t = (1-\beta)\omega_i^{t-1} + \beta\frac{1}{n}$
    \ENDIF
    \STATE  $x^{t+1} = x^t - \gamma\sum_{i=1}^n \mathbb{I}_{[\hat{f} (x^t) - \hat{f} (x^t - \gamma  g_i^t) > 0]} \omega_i^{t} g_i^t$ \label{itr:line11}
\ENDFOR
\STATE \textbf{Output:} $\frac{1}{T}\sum\nolimits_{t = 0}^{T-1} x^t$
\end{algorithmic}
\end{algorithm}

\noindent
By adding momentum, we achieve a more stable convergence in practice. This allows previous favorable gradients to influence current weights, even if a device receives a small or zero weight in the current iteration.
An indicator in the step (Line \ref{itr:line11}) ensures that gradients maximizing trial loss are ignored, thereby guaranteeing minimization at each step. We also define $G={\min}_t G(t)$ as the minimum, taken over all iterations, of the number of honest workers at each iteration. Now we are ready to proceed with the theoretical results.
\begin{theorem}\label{ITRSGDa}
     Under Assumptions \ref{as1}, \ref{as2convex}, \ref{as3},
     \ref{as4} with $\delta_2 \leqslant \frac{1}{12}$, \ref{as5}, for solving the problem \eqref{eq:setting_dist}, after $T$ iterations of \text{Algorithm \ref{alg2}} with $\gamma \leqslant \frac{1}{13L}$, the following holds:
\begin{align*}
    \textstyle{\frac{1}{T}\sum\limits_{t = 0}^{T-1} \mathbb E \|\nabla f(x^t)\|^2 }\textstyle{\leqslant}&\textstyle{\frac{\mathbb E\left[\hat{f}(x^0) - \hat{f}(\hat{x}^*)\right]}{\gamma T}\cdot\frac{4 n}{\beta G }+ 6L \gamma\sigma^2} \\
    &\textstyle{+ 3\delta_1 + 4\zeta(N).}
\end{align*}
\end{theorem}
\noindent The first two terms in the result of Theorem \ref{ITRSGDa} replicate the findings from the standard SGD analysis, up to constant factors. The last term, which depends on $\zeta(N)$, is of special interest. The function $\zeta(N)$ reflects the relationship between \( f_1 \) and the trial loss \( \hat{f} \) and represents an approximation error. The dependence of $\zeta(N)$ on $N$ is natural: the larger $N$, the smaller the error becomes. More precisely, for our function \( f \), this error is $\zeta(N) = \mathcal{\widetilde{O}}\left(\frac{1}{N}\right) $ (see Lemma \ref{lemmaShalevShwartz} in Appendix). Although this error degrades convergence, it is common in machine learning tasks. In particular, the original learning problem, like \eqref{eq:setting_dist}, is often replaced by its Monte Carlo approximation \citep{NIPS2013_ac1dd209,defazio2014saga,allen2018katyusha}, and the resulting problem is commonly referred to as empirical risk minimization \citep{shalev2014understanding}. This replacement also leads to an error. Finally, $\delta_1$ is a typical term that represents data similarity (Assumption \ref{as4}) and is unavoidable in the presence of Byzantines \citep{wu2020federated, karimireddy2021byzantine, gorbunov2022secure}.

\noindent Since our approach resembles \textsc{Zeno} \citep{xie2019zeno} in its use of the trial function, we should mention that we found some issues in their proofs. In Theorem 1 of \citep{xie2019zeno}, the authors incorrectly apply the expectation operator when deriving their recursion. They sample a trial function from the full dataset and use $\mathbb{E}[\hat{f}^t(x)] = f(x)$, which is valid for a random point. However, in the case of point $x^{t+1}$, a mistake was made. Since they sample $\hat{f}^t$ in every iteration, the point $x^{t+1}$ depends on the sample $\hat{f}^{t}$, leading to $\mathbb{E}[\hat{f}^{t}(x^{t+1}) \mid x^t] \neq f(x^{t+1})$. By carrying the conditional expectation without considering the full expectation, it becomes impossible to enter the recursion and achieve convergence regarding the function $f$ itself. In turn, we explicitly bound the gradient discrepancy $|\nabla f_1(x^t) - \nabla \hat{f}(x^t)|$, leveraging trial function sampling to ensure convergence as the sample size increases. This difference in Theorem \ref{ITRSGDa} is represented by the discussed $\zeta(N)$.

\begin{corollary}\label{ITRSGDCorollary}
    Under the assumptions of \text{Theorem} \ref{ITRSGDa}, for solving the problem \eqref{eq:setting_dist}, after $T$ iterations of Algorithm \ref{alg2} with $\gamma \leqslant \min\biggl\{\frac{1}{13L}, \frac{\sqrt{2\mathbb E\left[\hat{f}(x^0) - \hat{f}(\hat{x}^*)\right] n}}{\sigma\sqrt{3L  G  \beta T}}\biggr\}$, the following holds:
        \begin{align*}
             \textstyle{ \frac{1}{T}\sum\limits_{t=1}^{T-1}}&\textstyle{\mathbb{E}\|\nabla f(x^t)\|^2 =  \mathcal{O}\biggl(\frac{\mathbb E\left[\hat{f}(x^0) - \hat{f}(\hat{x}^*)\right] Ln}{\beta  G T}}\\
             &\textstyle{+ \frac{\sigma\cdot\sqrt{\mathbb E\left[\hat{f}(x^0) - \hat{f}(\hat{x}^*)\right] \cdot L n}}{\sqrt{\beta G T}} 
            + \delta_1 
            + \zeta(N) \biggr).}
        \end{align*}
\end{corollary}
\noindent If we consider the first two terms in the convergence estimates from Corollary \ref{ITRSGDCorollary}, the only difference from the classical SGD convergence results \cite{NIPS2011_40008b9a, stich2019unified} is the additional factor $\frac{n}{  G }$, however, the rate is asymptotically optimal. The proofs and results for the strongly-convex objective can be found in Appendix \ref{C}.

\subsection{Improving theoretical estimates: \texttt{AutoBant}}\label{sec:main2}
Despite practical advantages of \texttt{Bant}, it has some theoretical imperfections related to the mechanism of assigning trust scores. While parameter $\beta$ helps honest clients maintain trust scores despite occasional bad gradients, it also enables Byzantine devices to retain their weights during attacks. To combat this, we add an indicator for the trial function reduction (the indication of the device being Byzantine at the considered iteration). 
However, this limits the theoretical applicability of the method to non-convex problems prevalent in modern machine learning.
To resolve these limitations, we present our second method, called AUxiliary Trial Optimization for Byzantines ANTidote (\texttt{AutoBant}), Algorithm \ref{alg1}.

\noindent The idea of assigning weights to devices as part of the optimization process has gained popularity in federated learning. For instance, in many works, it leads to improved solution quality \citep{li2023revisiting, tupitsa2024federated}, or it is used in more specific settings, such as personalized learning \citep{mishchenko2023partially}. We propose adapting this to Byzantine-robust learning by optimizing $\hat{f}$ regarding weights calculated after each algorithmic step (Line \ref{tfm:line9}).

\begin{algorithm}[t]
\caption{\texttt{AutoBant}}\label{alg1}
\begin{algorithmic}[1]
\STATE \textbf{Input:} Starting point $x^0 \in \mathbb{R}^d$
\STATE \textbf{Parameters:} Stepsize $\gamma > 0$, error accuracy $\delta$
\FOR{$t = 0, 1, 2, \ldots, T-1$}
    \STATE Server sends $x^t$ to each worker
    \FORALL{workers $i = 0, 1, 2, \ldots, n$ in parallel}
        \STATE Generate $\xi_i^t$ independently
        \STATE Compute stochastic gradient $g_i^t = g_i(x^t, \xi_i^t)$
        \STATE Send $g_i^t$ to server
    \ENDFOR
    \STATE \label{tfm:line9} $\omega^{t} \approx \arg\underset{\omega\in\Delta_1^n}{\min} \hat{f}\left(x^t - \gamma\sum\nolimits_{i=1}^n \omega_i g_i^t\right)$ 
    \STATE \label{tfm:line10} $x^{t+1} = x^t - \gamma\sum\nolimits_{i=1}^n \omega_i^{t} g_i^t$ 
\ENDFOR
\STATE \textbf{Output:} $\frac{1}{T}\sum\nolimits_{t = 0}^{T-1} x^t$
\end{algorithmic}
\end{algorithm}

\noindent To solve the minimization problem, we can use various methods, e.g., Mirror Descent \citep{beck2003mirror, allen2014linear}:
\begin{align*}
\textstyle{\omega^{k+1}~ = \arg\underset{\omega \in \Delta_1^n}{\min}}&\textstyle{ \Bigl\{\eta \left\langle \nabla_\omega\hat{f}\left(x^t - \gamma \sum\limits_{i=1}^n \omega_i^k g_i^t\right), \omega \right\rangle} \\
&\textstyle{+ \mathcal{KL}(\omega \Vert \omega^k) \Bigr\}},
\end{align*}
where $\mathcal{KL}(\cdot\Vert\cdot)$ denotes the Kullback-Leibler divergence. The error in solving this is bounded by $\delta$:
\begin{align*}
\textstyle{\left|\underset{\omega\in\Delta_1^n}{\min} \hat{f}\left(x^t - \gamma\sum\limits_{i=1}^n \omega_i g_i^t\right) - \hat{f}\left(x^t - \gamma\sum\limits_{i=1}^n \omega_i^t g_i^t\right)\right| \leqslant \delta.}
\end{align*}
After solving this auxiliary problem, we produce an actual model update using the optimized weights (Line \ref{tfm:line10}). In light of the proposed method, the question of the cost of implementing such an optimal scheme comes to the forefront. Note that the computational complexity of solving this subproblem at each iteration is only $\mathcal{O}(\nicefrac{\log n}{\delta^2})$ \citep{beck2003mirror}, which is not critical.
For the following theoretical analysis, we assume the minimization subproblem can be solved to arbitrary precision, neglecting errors in our estimates. In practice, however, the convergence rate is directly influenced by the accuracy of this solution. A numerical study of this effect is presented in Table \ref{tab:delta_error_time} in Appendix.
Now we are ready to present the main theoretical result of this section.

\begin{theorem}\label{TFMSGD}
    Under Assumptions \ref{as1}, \ref{as2nonconvex}, \ref{as3}, \ref{as4} with $\delta_2 < \frac{1}{12}$, \ref{as5}, for solving the problem \ref{eq:setting_dist}, after $T$ iterations of \text{Algorithm \ref{alg1}} with $\gamma \leqslant \frac{1}{13L}$, the following holds: 
\begin{align*}
    \textstyle{\frac{1}{T}\sum\limits_{t = 0}^{T-1} \mathbb E \|\nabla f(x^t)\|^2 \leqslant} &\textstyle{\frac{4\mathbb E\left[\hat{f}(x^0) - \hat{f}(\hat{x}^*)\right]}{\gamma T} + 3\delta_1 }\\
    &\textstyle{+ \frac{6L\gamma}{ G }\sigma^2+ 2\zeta(N) 
    . }
\end{align*}
\end{theorem}
 

\begin{corollary}\label{TFMSGDCorollary}
    Under assumptions of \text{Theorem \ref{TFMSGD}}, for solving the problem \eqref{eq:setting_dist}, after $T$ iterations of Algorithm \ref{alg1} with $\textstyle{\gamma \leqslant \min\left\{\frac{1}{13L}, \frac{\sqrt{2\mathbb E\left[\hat{f}(x^0) - \hat{f}(\hat{x}^*)\right] G}}{\sigma\sqrt{3L T}}\right\}}$, the following holds:
    \begin{align*}
    \textstyle{
    \frac{1}{T}\sum\limits_{t=0}^{T-1}\mathbb{E}\left\|\nabla f(x^t)\right\|^2=
    \mathcal{O}\Biggl( }&\textstyle{ \frac{\mathbb E\left[\hat{f}(x^0) - \hat{f}(\hat{x}^*)\right] L}{T} + \delta_1 +\zeta(N)}\\
    &
    \hspace{2mm} 
    \textstyle{+~\frac{\sigma\sqrt{\mathbb E\left[\hat{f}(x^0) - \hat{f}(\hat{x}^*)\right] L}}{\sqrt{T G }}
    \Biggr).}
    \end{align*}
\end{corollary}
\noindent Detailed proofs are presented in Appendix \ref{D}. In the first and second terms, we observe that the method converges in the same pace as the standard SGD only with honest workers \citep{ghadimi2013stochastic, ghadimi2016mini}. It turns out that we throw out all Byzantines, and this result is nearly optimal and unimprovable. As in Algorithm \ref{alg2}, a term responsible for the approximation error $\zeta(N)$ appears. Compared with the result of Corollary \ref{ITRSGDCorollary}, we improve the rate through a more advanced aggregation mechanism. We remove the factor $\frac{n}{\beta G}$ from the main term and achieve a decrease in variance by a factor of $G$. However, an additional error $\delta$ is incurred, which can be viewed as a trade-off for solving the subproblem. Furthermore, our approach applies to a broader class of non-convex functions. Addressing the dependence of the stepsize on the number of Byzantines, this choice is based on the theoretical analysis of the worst-case scenario, considering the number of Byzantines. If this number is unknown, setting the minimum possible value of 1 eliminates this dependency.

\section{Extensions} \label{sec:extensions}
Byzantine robust optimization, as discussed earlier, lacks a solid theoretical foundation in several real-world settings. We address this gap. This section provides a brief overview of the scenarios to which we extend our analysis.
\subsection{Local methods}\label{sec:mainlocal}
The main idea is that each device performs a predefined number of local steps. Then the aggregation of gradients and the mutual updates of model parameters, initialized by the server, take place. This reduces the number of communication rounds. However, it affects the convergence proportionally to the length of the communication round.
Complete updates are performed only at specific iterations: $t = t_{k \cdot l} \text{ for some } k = \overline{0, \left\lfloor \nicefrac{T}{l} \right\rfloor}$. During the remaining iterations, we perform local updates using the rule \(x_i^{t+1} = x_i^{t} - \gamma g_i^t\).
This approach ensures that communication overhead is minimized while maintaining efficient convergence.
\subsection{Partial participation}\label{sec:mainpartial}
It occurs when only a subset of clients actively participates in the training process during each communication round \citep{yang2021achieving}, allowing clients to join or leave the system. This approach is beneficial in scenarios such as mobile edge computing. However, it poses challenges such as incomplete model updates and potential degradation in model performance due to missed contributions from inactive clients \citep{wang2022unified, li2022federated}. Our methods adapt to the scenario of partial participation by assigning trust scores to devices explicitly participating in training during the considered iteration. Furthermore, it is crucial to account for the minimum number of nodes participating in training across all iterations. Specifically, we analyze $\widetilde{G}(t) = \min_{t \leqslant T} G(t)$, where $G(t)$ denotes the set of active honest workers at iteration $t$.
\subsection{Scaled methods}\label{sec:mainscaled}
\noindent Adaptive methods such as \textsc{Adam} \citep{kingma2014adam} and \textsc{RMSProp} \citep{tieleman2012lecture} have become widely popular due to their superior performance compared to standard \textsc{SGD}-like methods. We propose corresponding methods that utilize a diagonal preconditioner \((\hat{P}^t)^{-1}\), which scales a gradient to \((\hat{P}^t)^{-1} g_i^t\), and the step is executed using this scaled gradient.
We present the part of \texttt{Scaled AutoBant} based on Algorithm \ref{alg1}.
The estimates obtained are identical to those derived from the scaled methods in the non-Byzantine regime.
All details can be found in Appendix \ref{E}.
\begin{algorithm}[t]
\caption{\texttt{Scaled AutoBant} (part)}
\begin{algorithmic}[1]
\setcounter{ALC@line}{9} 
\STATE \small $\omega^{t} \approx \arg\underset{\omega\in\Delta_1^n}{\min} \hat{f}\left(x^t - \gamma\left(\hat{P}^t\right)^{-1}\sum\nolimits_{i=1}^n \omega_i g_i^t\right)$ 
\setcounter{ALC@line}{10} 
\STATE $x^{t+1} = x^t - \gamma\left(\hat{P}^t\right)^{-1}\sum\nolimits_{i=1}^n \omega_i^{t} g_i^t$ 
\end{algorithmic}
\end{algorithm}



\subsection{Finding scores from validation} \label{sec:mainfinetuned}
\noindent Another interesting direction to obtain trust scores $w_i$ is to calculate \textit{similarity between the logits} obtained on the server and on the device. The trust score for the $i$-th device is the function $\alpha_i\rightarrow sim(m(x^t - \gamma g_i^t, \mathcal{\hat{D}}),m(x^t - \gamma g_1^t, \mathcal{\hat{D}}))$. Based on Algorithm \ref{alg1}, we present a part of the \texttt{SimBant} algorithm (see details in Appendix \ref{sec:fine_tuned}).
\begin{algorithm}[t]
\caption{\texttt{SimBant} (part)}
\begin{algorithmic}[1]
\setcounter{ALC@line}{9}
\STATE \small $\omega_i^{t} = (1-\beta)\omega_i^{t-1} + \beta\frac{sim(m(x^t - \gamma g_i^t, \mathcal{\hat{D}}),m(x^t - \gamma g_1^t, \mathcal{\hat{D}}))}{\sum\nolimits_{j = 1}^n sim(m(x^t - \gamma g_j^t, \mathcal{\hat{D}}),m(x^t - \gamma g_1^t, \mathcal{\hat{D}}))}$ 
\STATE $x^{t+1} = x^t - \gamma\sum\nolimits_{i=1}^n \omega_i^{t} g_i^t$ 
\end{algorithmic}
\end{algorithm}

\section{Experiments}\label{sec:experiments}
To evaluate the performance of the proposed methods, we conduct experiments on several benchmarks.
\begin{itemize}[leftmargin=*]
   \item \textbf{Classification Task:}  We first validate our approach on the public dataset. We use \textsc{ResNet-18} \citep{he2016deepresnset} for CIFAR‑10 \citep{krizhevsky2009cifar} classification.
  \item \textbf{ECG Abnormality Detection:} In multi‐hospital collaborations, labels are derived from expert annotations and automated pipelines, making attacks and subtle manipulations practical threats that can compromise patient safety. We obtain a proprietary dataset of 12-lead digital electrocardiograms (ECG) from five hospitals and train \textsc{ResNet1d18} model for ECG abnormality detection.
 \item \textbf{Learning‑to‑Rank (Recommender Systems):} We conducted a series of experiments applied to the Learning-to-Rank (LTR) task, common in information retrieval and recommendation systems. We adopt the Transformer architecture \citep{vaswani2017attention}, evaluating its performance on the dataset \textsc{Web30k} \citep{DBLP:journals/corr/QinL13} under attacks.
\end{itemize}

\noindent We consider various Byzantine attacks to test our methods.\\
    $\bullet$ \textbf{Label Flipping.} Attackers send gradients based on the loss calculated with randomly flipped labels.
    \\
    $\bullet$ \textbf{Sign Flipping.} Attackers send the opposite gradient.
    \\
    $\bullet$ \textbf{Random Gradients.} Attackers send random gradients.
    \\
    $\bullet$ \textbf{IPM (Inner Product Manipulation).} Attackers send the average gradient of all honest clients multiplied by a factor of -$\kappa$ (we set $\kappa$ to 0.5) \citep{xie2020fall}.
    \\
    $\bullet$ \textbf{ALIE (A Little Is Enough).} Attackers average their gradients and scale the standard deviation to mimic the majority \citep{baruch2019little}.
    
\noindent We define the number of Byzantine clients as a percentage of the total number of clients, and specify it in the attack name.
We train \texttt{Bant} and \texttt{AutoBant} in the scaled version (see Section \ref{sec:mainscaled}) with the \textsc{Adam} preconditioner. We include \texttt{SimBant}, \textsc{Adam}, and the existing methods: \textsc{Zeno} \citep{xie2019zeno}, \textsc{Recess} \citep{yan2024recess}, \textsc{Centered Clip} \citep{karimireddy2021learning}, \textsc{Safeguard} \citep{allen2020byzantine}, \textsc{VR Marina} \citep{gorbunov2023variance} and \textsc{FLTrust} \citep{cao2020fltrust}. For \textsc{Centered Clip}, we added techniques \textsc{Fixing by Mixing} \citep{allouah2023fixing} and \textsc{Bucketing} \citep{karimireddy2021byzantine}. The methods were trained on the CIFAR-10 and ECG datasets for 200 and 150 rounds, respectively.

\begin{table*}[t]
\centering
\fontsize{9pt}{10pt}\selectfont
\begin{adjustbox}{max width=\textwidth}
\begin{tabular}{|>{\arraybackslash}p{2cm}|c|c|c|c|c|c|c|c|c|c|}
\toprule
\multirow{2}{*}{Algorithm}
 & \multicolumn{2}{c|}{Without Attack}
 & \multicolumn{2}{c|}{Label Flipping (60\%)}
 & \multicolumn{2}{c|}{Random Gradients (60\%)}
 & \multicolumn{2}{c|}{IPM (80\%)}
 & \multicolumn{2}{c|}{ALIE (40\%)} \\
\cline{2-11}
 & G-mean & f1-score & G-mean & f1-score & G-mean & f1-score & G-mean & f1-score & G-mean & f1-score \\
\midrule
\textsc{Adam}     & 0.956$\pm$0.017 & 0.811$\pm$0.016 & 0.262$\pm$0.023 & 0.041$\pm$0.019 & 0.348$\pm$0.011 & 0.126$\pm$0.016 & 0.197$\pm$0.027 & 0.036$\pm$0.015 & 0.125$\pm$0.011 & 0.123$\pm$0.020 \\
\textsc{FLTrust}  & 0.952$\pm$0.020 & 0.800$\pm$0.019 & 0.952$\pm$0.016 & 0.753$\pm$0.011 & 0.617$\pm$0.020 & 0.174$\pm$0.019 & 0.061$\pm$0.017 & 0.125$\pm$0.015 & 0.017$\pm$0.013 & 0.123$\pm$0.018 \\
\textsc{Recess}   & 0.949$\pm$0.016 & 0.783$\pm$0.019 & 0.366$\pm$0.019 & 0.128$\pm$0.020 & 0.593$\pm$0.020 & 0.163$\pm$0.020 & 0.493$\pm$0.019 & 0.112$\pm$0.015 & 0.450$\pm$0.014 & 0.127$\pm$0.018 \\
\textsc{Zeno} & 0.921$\pm$0.012 & 0.787$\pm$0.014 & 0.014$\pm$0.017 & 0.110$\pm$0.015 & 0.163$\pm$0.010 & 0.089$\pm$0.014 & 0.102$\pm$0.012 & 0.066$\pm$0.018 & 0.010$\pm$0.009 & 0.091$\pm$0.011 \\
\textsc{CC}       & 0.949$\pm$0.020 & 0.772$\pm$0.019 & 0.285$\pm$0.018 & 0.114$\pm$0.020 & 0.580$\pm$0.019 & 0.155$\pm$0.020 & 0.084$\pm$0.019 & 0.014$\pm$0.020 & 0.530$\pm$0.018 & 0.154$\pm$0.020 \\
\textsc{CC+fbm}   & 0.954$\pm$0.016 & 0.808$\pm$0.020 & 0.840$\pm$0.019 & 0.716$\pm$0.014 & 0.562$\pm$0.011 & 0.151$\pm$0.020 & 0.027$\pm$0.018 & 0.123$\pm$0.015 & 0.876$\pm$0.017 & 0.594$\pm$0.013 \\
\footnotesize\textsc{CC+bucketing}  & 0.947$\pm$0.013 & 0.790$\pm$0.018 & 0.829$\pm$0.011 & 0.708$\pm$0.020 & 0.570$\pm$0.012 & 0.164$\pm$0.018 & 0.035$\pm$0.020 & 0.118$\pm$0.012 & 0.870$\pm$0.019 & 0.587$\pm$0.014 \\
\textsc{Safeguard} & \textbf{0.957$\pm$0.020} & 0.821$\pm$0.019 & 0.107$\pm$0.012 & 0.123$\pm$0.020 & 0.258$\pm$0.011 & 0.124$\pm$0.019 & 0.951$\pm$0.018 & 0.082$\pm$0.020 & 0.010$\pm$0.009 & 0.123$\pm$0.012 \\
\textsc{VR Marina} & 0.010$\pm$0.014 & 0.120$\pm$0.010 & 0.027$\pm$0.018 & 0.123$\pm$0.020 & 0.176$\pm$0.012 & 0.103$\pm$0.013 & 0.127$\pm$0.013 & 0.079$\pm$0.019 & 0.012$\pm$0.010 & 0.108$\pm$0.013 \\
\midrule
\texttt{Bant}     & 0.953$\pm$0.017 & \textbf{0.830$\pm$0.020} & \textbf{0.956$\pm$0.016} & \textbf{0.777$\pm$0.020} & \textbf{0.948$\pm$0.018} & \textbf{0.809$\pm$0.020} & 0.946$\pm$0.020 & 0.676$\pm$0.015 & \textbf{0.947$\pm$0.018} & \textbf{0.770$\pm$0.020} \\
\texttt{AutoBant} & 0.953$\pm$0.019 & 0.781$\pm$0.020 & 0.790$\pm$0.020 & 0.276$\pm$0.020 & 0.946$\pm$0.019 & 0.748$\pm$0.018 & 0.942$\pm$0.020 & 0.690$\pm$0.020 & 0.892$\pm$0.016 & 0.585$\pm$0.020 \\
\texttt{SimBant}  & 0.956$\pm$0.020 & 0.790$\pm$0.018 & 0.949$\pm$0.020 & 0.774$\pm$0.020 & 0.945$\pm$0.020 & 0.712$\pm$0.018 & \textbf{0.955$\pm$0.020} & \textbf{0.783$\pm$0.018} & 0.946$\pm$0.019 & 0.705$\pm$0.020 \\
\bottomrule
\end{tabular}
\end{adjustbox}
\caption{\textsc{ResNet1d18} on ECG (AFIB) for Byzantine-tolerance techniques under various attacks.}
\label{tab:ecggmean}
\end{table*}

\paragraph{CIFAR‑10 and ECG setups.}
For the CIFAR-10 dataset, we divide the data among 10 and 100 (see Appendix \ref{sec:appendix_experiments_cifar_stress}) clients. We consider a homogeneous split with 5,000 images per client, as well as a Dirichlet split with $\alpha=0.5$ and $\alpha=1$. We use 500 separate samples to form $\hat{f}$. For the ECG dataset, we consider five clients, each representing a hospital with 10,000 and 20,000 records. To form $\hat{f}$ on ECG, we use 100 samples from the publicly-available external PTB-XL dataset \citep{wagner2020ptb}. We solve the task of multiclass classification for CIFAR-10 and binary classification of 4 heart abnormalities for ECG: Atrial FIBrillation (AFIB), First-degree AV Block (1AVB), Premature Ventricular Complex (PVC), and Complete Left Bundle Branch Block (CLBBB).

\noindent The accuracy for all considered attacks on the CIFAR-10 test dataset are illustrated in Figure \ref{fig:cifaraccuracy}. To further stress test the proposed methods, we consider the most strong Byzantine attacks under heterogeneous setups, as well as homogeneous split under 100 clients. Figure \ref{fig:cifaraccuracy_dirichlet} shows the accuracy plots of the proposed methods with Dirichlet $\alpha=0.5$ for the ALIE and Random Gradients attacks. More details are presented in the Appendix. 

\includegraphics[width=0.95\linewidth]{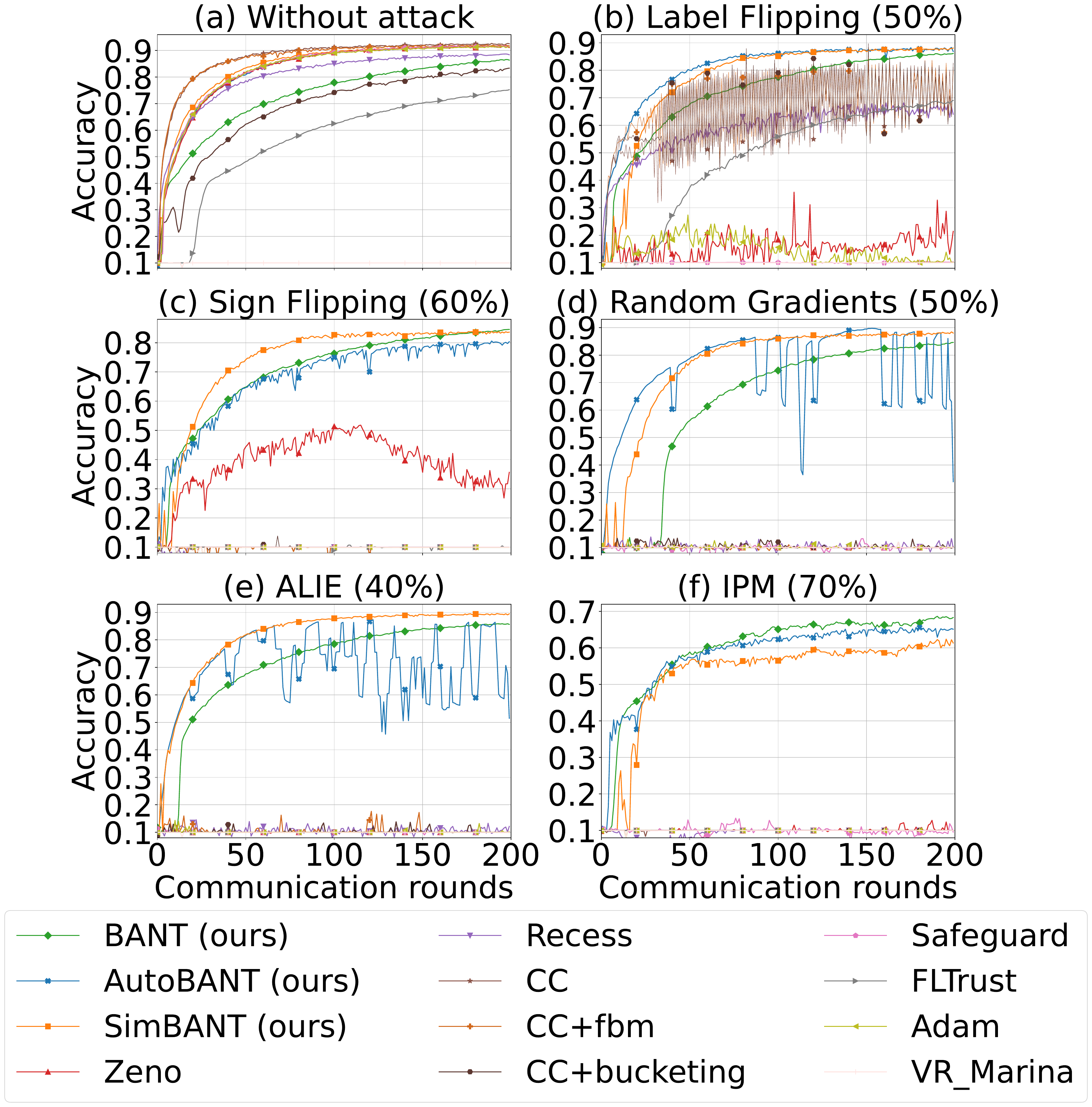}
\captionof{figure}{Test accuracy, ResNet18 on CIFAR-10.}
\label{fig:cifaraccuracy}

\includegraphics[width=\linewidth]{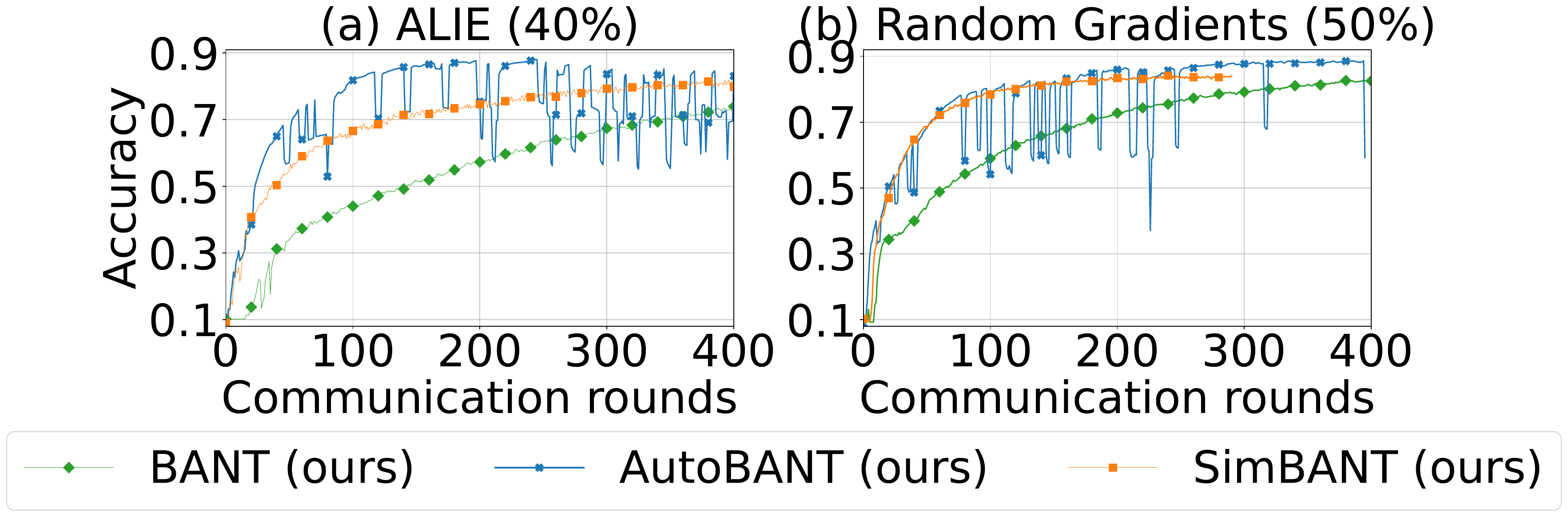}
\captionof{figure}{Test accuracy, ResNet18 on Dirichlet.}
\label{fig:cifaraccuracy_dirichlet}

\includegraphics[width=\linewidth]{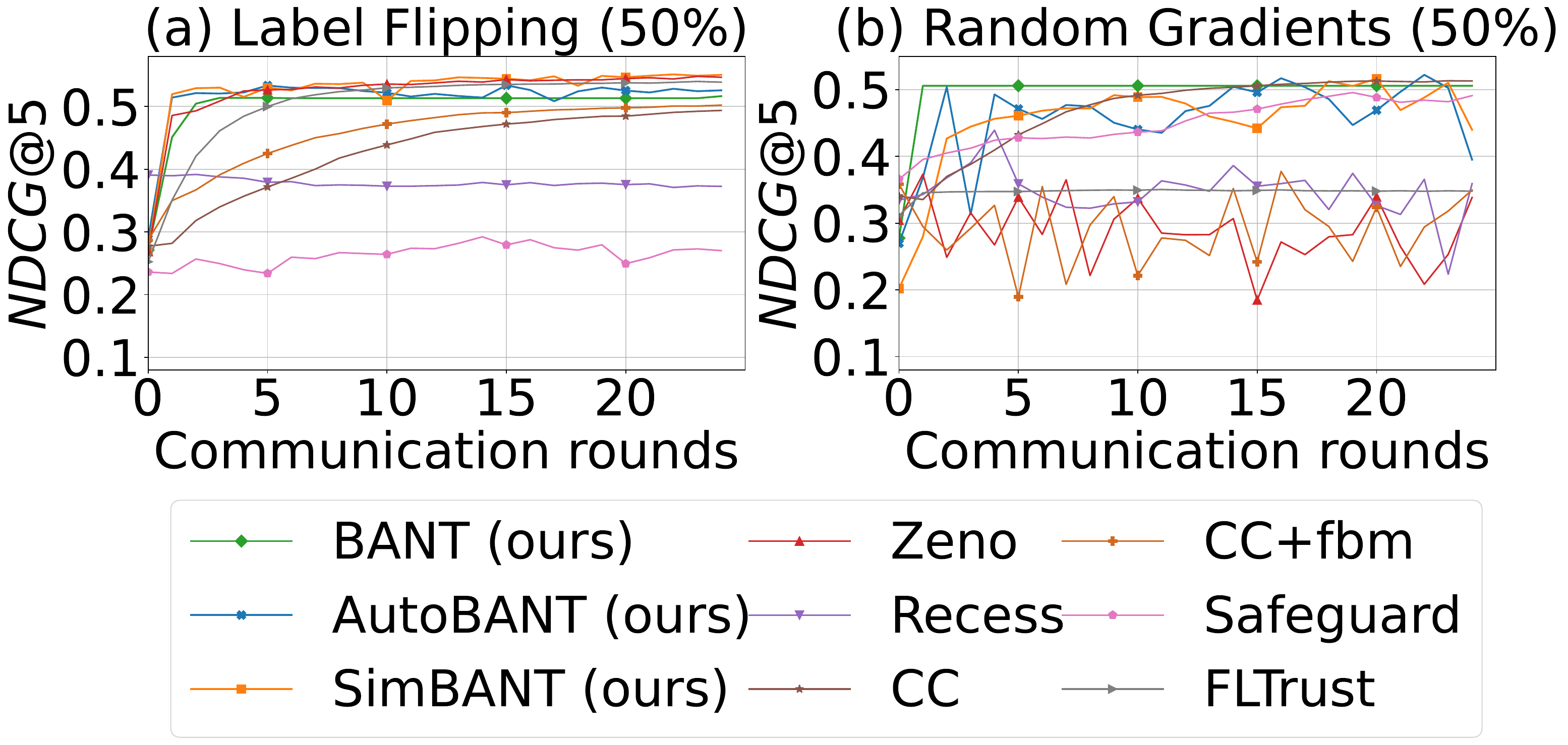}
    \caption{Test NDCG@5, Transformer on LTR task.}
    \label{fig:ranking_main}

\noindent To assess model performance on the ECG data, we use the G-mean (the square root of sensitivity multiplied by specificity) and the f1-score metrics. Table \ref{tab:ecggmean} summarizes the results of the methods for the AFIB disease classification. Detailed results for all considered abnormalities across multiple metrics are presented in Tables \ref{tab:ecgafib}-\ref{tab:ecgclbbb} in Appendix.

\noindent Unlike previously established techniques, our methods exhibit robustness against all Byzantine attacks across different benchmarks. We note that \texttt{AutoBant} performs slightly worse than \texttt{Bant} and \texttt{SimBant} under Random Gradients and ALIE attacks. This occurs due to solving an auxiliary subproblem (Line \ref{tfm:line9} in Algorithm \ref{alg1}) using \textsc{Mirror Descent} with KL-divergence. According to its properties, the algorithm assigns small but non-zero weights to Byzantines, contributing to unstable convergence, while \texttt{Bant} and \texttt{SimBant} lack this drawback. We analyze the required number of such iterations and the size of $\hat{f}$ in Appendix \ref{sec:appendix_experiments_ecg}. \textsc{Recess} and \textsc{FLTrust} leverage the concept of trust scores but rely on a majority of honest devices. As a result, this leads to a significant decrease in the final quality under the majority of Byzantines in Random Gradients, IPM, and ALIE attacks that simulate a malicious majority. Similar behavior is observed for the \textsc{CC} and \textsc{Safeguard} methods, which suffer from sensitivity to parameter tuning. \textsc{Fixing by Mixing} and \textsc{Bucketing} increase Label Flipping and ALIE metrics for ECG setup, but do not provide reliable convergence for all cases. \textsc{Zeno} exploits the trial function approach, however, it relies on the number of Byzantine clients. 
We address this in Appendix, Table \ref{tab:ecggmean_zeno}.

\paragraph{Learning-to-Rank.} In the LTR task, the goal is to learn a ranking function over query-document pairs. Each pair is represented using standard frequency-based feature vectors. The target labels correspond to human-assigned relevance scores, reflecting how well a document matches a given query.
This setting provides a natural context for exploring Byzantine robustness. For example, label flipping attacks arise from this scenario. Annotators may provide inconsistent or biased relevance assessments. Such inconsistencies reflect real-world challenges in supervised learning from human-generated data. The quality of labeling is influenced by personal biases or mistakes. We compare our methods against the baselines from prior experiments -- under the most severe attacks (Label Flipping 50\%, Random Gradients 50\%), see Figure~\ref{fig:ranking_main}.



\bibliography{refs}
\bibstyle{plainnat}
\addcontentsline{toc}{section}{References} \label{references}

\newpage

\onecolumn
\appendix
\part*{Appendix}

\allowdisplaybreaks

\section{Additional Experiments}\label{A}

\paragraph{Overview.} 
This section provides an extensive overview of additional experiments. We begin by outlining the key technical details relevant to our experimental setup, including a description of the computational resources and a complete list of hyperparameters for all compared methods. We then present a series of extended experiments across various settings.

In Section~\ref{sec:appendix_experiments_cifar}, we compare the final test accuracies of all methods under different attack scenarios, accompanied by training loss curves over epochs. We also provide a comprehensive table reporting the runtime of all methods both in the absence of attacks and under the ALIE attack.

We then address another important factor: the number of samples used in the trial function. A key question in our methodology is how the size of the local dataset on the honest device influences convergence. We demonstrate that while such an effect exists, it remains minor, and even with a small number of samples, our methods maintain strong convergence properties.

Next, we analyze how the solution quality of the inner minimization problem in \texttt{AutoBant} affects the final performance. By varying the depth of the mirror descent procedure, we obtain different levels of accuracy in solving the inner problem and assess their influence on downstream metrics.

In Section~\ref{sec:appendix_experiments_cifar_stress}, we investigate the impact of scaling the number of clients and data heterogeneity. As claimed in the main part of the paper, our methods remain effective in these stressful conditions. For the first one, we split the \textsc{CIFAR-10} training dataset between 100 clients. For the second, we simulate heterogeneity using a Dirichlet$(\alpha)$ distribution with varying values of $\alpha$.

Section \ref{sec:appendix_experiments_ecg} focuses on experiments in the ECG domain. We start by analyzing the sensitivity of the \textsc{Zeno} algorithm to its hyperparameters. Although \textsc{Zeno} performs well in the main part of the paper, we show that this is largely due to a correctly chosen estimate of the number of Byzantine clients. In scenarios where this proportion is unknown, the method fails to converge. We also provide extended results on the detection of various cardiac conditions, including AFIB, 1AVB, PVC, and CLBBB.

Finally, we present additional experiments for the Learning-to-Rank task in Section \ref{sec:appendix_experiments_ltr}. While the main part of the paper includes a comparison of our proposed methods with only the most robust baselines under strong attacks, this section offers a broader experimental validation to further support the superiority of our approach.

\subsection{Technical details} \label{app:details}

\paragraph{Compute resources.} Our implementation is developed in Python 3.10. We simulate a distributed system on a server. The server is equipped with an AMD EPYC 7513 32-Core Processor running at 2.6 GHz and $4\times$Nvidia A100 SXM4 40GB. This configuration is used for the experiments described in Section \ref{sec:experiments}.

\begin{minipage}{\textwidth}
\centering
\captionof{table}{General hyperparameter setup}
\label{tab:gen_hyp_setup}
\begin{adjustbox}{width=0.5\textwidth}

\begin{tabular}{|c|c|c|}
\toprule
\small \textbf{Hyperparameters} & \textsc{CIFAR-10}, LTR & ECG \\
\midrule
Batch Size & 32 & 64 \\
Client lr & 0.003 & 0.003 \\
Loss & Cross-Entropy (CE) & Binary CE \\
\bottomrule
\end{tabular}
\end{adjustbox}
\end{minipage}

\paragraph{Hyperparameters and strategies.} To ensure a fair comparison, we maintain consistent hyperparameters across all methods. Table \ref{tab:gen_hyp_setup} summarizes them.  The batch size is set to 32 for CIFAR-10 and 64 for the ECG dataset, with a local client learning rate of 0.003 and \textsc{adam} preconditioner. All clients perform local computations for 1 epoch. For CIFAR-10, cross-entropy was used as the local client loss function, while for ECG is was its binary version. To address the problem of imbalance of positive and negative examples in the ECG, the minority class was reweighted to the majority in the corresponding loss function. For the ECG classification task, we train the models on 10-second 12-lead ECG records, with all records resampled to a frequency of 500 Hz. Additionally, we train the models exclusively on patients older than 18 years of age.  

We also select specific parameters for the implemented methods. Table \ref{tab:method_hyp_setup} summarizes them. For the \textsc{Safeguard} method, we use window sizes of 1 and 6 for two different accumulation settings, with the threshold chosen automatically, as described in the original paper. We adapt the \textsc{CC} method to a local computation case and set the clipping coefficient $\tau = 0.1$ and  the \textsc{SGD} momentum $\beta = 0.9$ for CIFAR-10, as well as $\tau = 1$ and $\beta = 0.5$ for the ECG case. For both setups we fixed the number of clipping iterations to $l = 1$. For the \textsc{Fixing-by-Mixing} technique, we set the number of Byzantine clients $f$ to be less than half of all clients as suggested in the article: for the CIFAR-10 and LTR $f = 4$, while for the ECG $f = 2$. For the \textsc{Bucketing} technique we apply \textit{2-bucketing} strategy with global learning rate $\eta=0.9$. For the \textsc{Recess} method, we set the decrease score equal to $0.1$. For the \texttt{Bant} method, we set the momentum parameter $\beta=0.5$. The \texttt{AutoBant} method uses the number of optimization epochs equal to 5 and $\gamma = 1$. For \texttt{SimBant}, we set the softmax temperature parameter for the model logits to 0.05.

Specifically, we want to highlight the choice of hyperparameters for the \textsc{Zeno} method. We set the regularization weight $\rho = 0.0005$ as a default value in the paper. As for the threshold for defining Byzantines -- trim parameter -- we set $b = 2$ for the ECG setup, $b = 3$ for CIFAR-10 and $b = 5$ for Learning-to-Rank. We address the choice of hyperparameter $b$ in Table \ref{tab:ecggmean_zeno} in Appendix \ref{sec:appendix_experiments_ecg}, as it is critical in real-world scenarios. 

Additionally, we define distinct functions for \texttt{SimBant} based on the dataset, as ECG classification is binary, whereas CIFAR-10 is a multi-class classification problem. Specifically, for the ECG dataset, we use:
\begin{equation}\label{eq:similarity_ecg} sim_{\text{ECG}}(x, y) = 1 - |x - y|, \end{equation}
where $x$ is the output of the client model, and $y$ is the output of the model fine-tuned on the server. For the CIFAR-10 dataset, we apply cosine similarity:
\begin{equation}\label{eq:similarity_cifar} sim_{\text{CIFAR}}(x, y) = \frac{x \cdot y}{\|x\| \cdot \|y\|},
\end{equation}
where $y$ is one-hot encode targets.

\begin{table}[]
\centering
\caption{Specific hyperparameter setup}
\label{tab:method_hyp_setup}
\begin{adjustbox}{width=\columnwidth}
\begin{tabular}{|c|c|}
\hline
\textbf{Method} & \textbf{Hyperparameters} \\ 
\hline
\textsc{Safeguard} & 
\begin{tabular}[c]{lccc}
& \hspace{30mm} & Window Sizes ($T_0, T_1$) & \hspace{20mm}\\ 
• CIFAR-10, LTR: & \hspace{30mm} & $T_0 = 1$, $T_1 = 6$ & \hspace{20mm} \\
• ECG: &  & $T_0 = 1$, $T_1 = 6$ &
\end{tabular} \\ 
\hline
\textsc{Central Clip} & 
\begin{tabular}[c]{lccc}
& \hspace{8mm} Clip coefficient ($\tau$) \hspace{6mm} & \hspace{1mm} Momentum ($\beta$) & Clip iterations ($l$) \\
• CIFAR-10: &  \hspace{8mm} $\tau = 0.1$ \hspace{6mm} & \hspace{1mm} $\beta = 0.9$ & $l = 1$ \\
• ECG: & \hspace{8mm} $\tau = 1$ \hspace{6mm} & \hspace{1mm} $\beta = 0.5$ & $l = 1$ \\
\end{tabular} \\ 
\hline
\textsc{Fixing-By-Mixing} & 
\begin{tabular}[c]{lccc}
& \hspace{28mm} & Number of Byzantines ($f$) & \hspace{17mm} \\
• CIFAR-10, LTR: & \hspace{28mm} & $f = 4$ &  \hspace{17mm} \\
• ECG: & & $f = 2$ & 
\end{tabular} \\ 
\hline
\textsc{Bucketing} & 
\begin{tabular}[c]{lccc}
& \hspace{38mm} & Global Learning Rate ($\eta$) & \hspace{17mm} \\ 
• CIFAR-10: & \hspace{38mm} & $\eta = 0.9$ & \hspace{17mm} \\
• ECG: & & $\eta = 0.9$ &
\end{tabular} \\ 
\hline
\textsc{Recess} & 
\begin{tabular}[c]{lccc}
& \hspace{35mm} & Decrease Score ($d$) & \hspace{20mm} \\ 
• CIFAR-10, LTR: & \hspace{35mm} & $d = 0.1$ &  \hspace{20mm} \\
• ECG: & & $d = 0.1$ &
\end{tabular} \\ 
\hline
\textsc{Zeno} & 
\begin{tabular}[c]{lccc}
& Regularization weight ($\rho$) & \hspace{14mm} &  \hspace{6mm} Trim parameter ($b$) \\
• CIFAR-10, LTR: & \hspace{-6mm} $\rho = 0.005$ & \hspace{14mm} & \hspace{6mm} $b = 3, 5$\\
• ECG: & \hspace{-6mm} $\rho = 0.005$ & \hspace{14mm} & \hspace{6mm} $b = 2$
\end{tabular} \\ 
\hline
\hline
\texttt{Bant} & 
\begin{tabular}[c]{lccc}
& \hspace{4mm} Momentum ($\beta$) \hspace{10mm} &  \hspace{26mm} & Trial size ($ts$)  \\
• CIFAR-10, LTR: & \hspace{4mm} $\beta= 0.5$ \hspace{10mm} & \hspace{26mm} &  $ts = 500$ \\
• ECG: & \hspace{-6mm} $\beta= 0.5$ \hspace{10mm}& \hspace{26mm} & $ts = 100$
\end{tabular} \\ 
\hline
\texttt{AutoBant} & 
\begin{tabular}[c]{lccc}
& \hspace{4mm} Mirror epochs ($e$) \hspace{12mm} & Mirror $\gamma$ \hspace{8mm} & Trial size ($ts$) \\
• CIFAR-10, LTR: & \hspace{4mm} $e = 5$ \hspace{12mm} & $\gamma = 1$ \hspace{8mm} & $ts = 500$ \\
• ECG: & \hspace{4mm} $e = 5$ \hspace{12mm} & $\gamma = 1$ \hspace{8mm} & $ts = 100$
\end{tabular} \\  
\hline
\texttt{SimBant} & 
\begin{tabular}[c]{lccc}
& Softmax temperature ($T$) & Similarity function $\gamma$ & Trial size ($ts$) \\
• CIFAR-10, LTR: & $T = 0.05$ & see eq. \eqref{eq:similarity_cifar} & $ts = 500$ \\
• ECG: & $T = 0.05$ & see eq. \eqref{eq:similarity_ecg} & $ts = 100$
\end{tabular} \\ 
\hline
\end{tabular}
\end{adjustbox}
\end{table}

\subsection{CIFAR-10 Experiments} \label{sec:appendix_experiments_cifar}

\paragraph{Final accuracy.} In this section, we present supplementary data regarding the experiments conducted. As mentioned in the main part of the paper, we utilized \textsc{ResNet-18} models on the \textsc{CIFAR-10} dataset. We begin with a comparative Table \ref{tab:cifaraccuracy} of the methods applied to the CIFAR-10 dataset. 

\begin{table}[ht]
\vspace{-4mm}
\centering
\caption{\textsc{ResNet18} on CIFAR-10. Accuracy under various attacks.}
\begin{adjustbox}{width=\columnwidth}
\centering
\renewcommand{\arraystretch}{1.0}
\renewcommand{\tabcolsep}{4pt}
\scriptsize
\begin{tabular}{|
        p{0.12\columnwidth} |
        >{\centering\arraybackslash}p{0.08\columnwidth} |
        >{\centering\arraybackslash}p{0.11\columnwidth} |
        >{\centering\arraybackslash}p{0.10\columnwidth} |
        >{\centering\arraybackslash}p{0.12\columnwidth} |
        >{\centering\arraybackslash}p{0.06\columnwidth} |
        >{\centering\arraybackslash}p{0.06\columnwidth} |
        >{\centering\arraybackslash}p{0.10\columnwidth} |
        >{\centering\arraybackslash}p{0.06\columnwidth} |
    }
\toprule 
Algorithm & Without Attack & Label Flipping (50\%) & Sign Flipping (60\%) & Random Gradients (50\%) & IPM (70\%) & ALIE (40\%) & Sign Flipping (40\%) & IPM (50\%) \\ 
\midrule
\textsc{Adam} & 0.902 & 0.207 & 0.100 & 0.100 & 0.100 & 0.100 & 0.624 & 0.832 \\
\textsc{FLTrust} & 0.767 & 0.694 & 0.100 & 0.100 & 0.100 & 0.100 & 0.254 & 0.519 \\
\textsc{Recess} & 0.887 & 0.633 & 0.100 & 0.103 & 0.106 & 0.128 & 0.488 & 0.774 \\
\textsc{Zeno} & 0.910 & 0.156 & 0.410 & 0.100 & 0.100 & 0.100 & 0.838 & 0.100 \\
\textsc{CC} & \textbf{0.917} & 0.603 & 0.102 & 0.100 & 0.100 & 0.100 & 0.511 & \textbf{0.864} \\
\textsc{CC+fbm} & 0.915 & \textbf{0.887} & 0.098 & 0.100 & 0.101 & 0.100 & 0.823 & 0.923 \\
\textsc{CC+bucketing} & 0.845 & 0.818 & 0.089 & 0.101 & 0.100 & 0.101 & 0.815 & 0.100 \\
\textsc{Safeguard} & 0.918 & 0.102 & 0.100 & 0.102 & 0.104 & 0.113 & 0.826 & 0.112 \\
\textsc{VR Marina} & 0.100 & 0.100 & 0.100 & 0.100 & 0.100 & 0.100 & 0.100 & 0.100 \\
\midrule
\texttt{Bant} & 0.864 & 0.861 & \textbf{0.846} & 0.846 & \textbf{0.725} & 0.856 & \textbf{0.856} & 0.751 \\
\texttt{AutoBant} & 0.906 & 0.884 & 0.783 & \textbf{0.898} & 0.666 & \textbf{0.882} & 0.839 & 0.847 \\
\texttt{SimBant} & 0.909 & 0.855 & 0.827 & 0.865 & 0.623 & 0.878 & 0.852 & 0.827 \\
\bottomrule
\end{tabular}
\end{adjustbox}
\label{tab:cifaraccuracy}
\end{table}

This table illustrates the performance of different algorithms under various attacks, highlighting their effectiveness. We've supplemented it with scenarios of Sign Flipping ($40\%$) and IPM ($50\%$) attacks compared to the main part to demonstrate the results of baselines in less stressfull conditions. It is noteworthy that while existing methods provide some level of protection against certain attacks, none are effective when faced with a majority of malicious clients in gradient attacks, as well as in IPM ($70\%$) and ALIE ($40\%$), which simulates the majority. In contrast, all three of our methods demonstrate impressive performance in such attack scenarios. 
Furthermore, it is important to compare these results with those in the first column, which represents the metric in the absence of attacks. As we can see, our methods achieve only a slight reduction in the metric, yet they maintain relatively strong performance even under the most severe attacks. This resilience underscores the effectiveness of our approaches and their potential for real-world applications where robust defense mechanisms are crucial.

\paragraph{Decrease of loss functions.} Now we examine the graphs depicting the reduction of loss over the course of training (Figure \ref{fig:cifarloss}). The result is similar: as the number of attackers increases, the existing methods exhibit divergence, while our methods continue to decrease the loss effectively.
As mentioned in the main part of the paper, \texttt{AutoBant} may behave inconsistently under Random Gradients and ALIE attacks. This instability can be attributed to the solution of an additional minimization problem and the absence of an indicator that prevents theoretical advancement in non-convex scenarios. Nevertheless, even under these circumstances, \texttt{AutoBant} demonstrates significantly better results compared to its counterparts.

\begin{figure}[ht]
    \centering
    \includegraphics[width=0.96\textwidth]{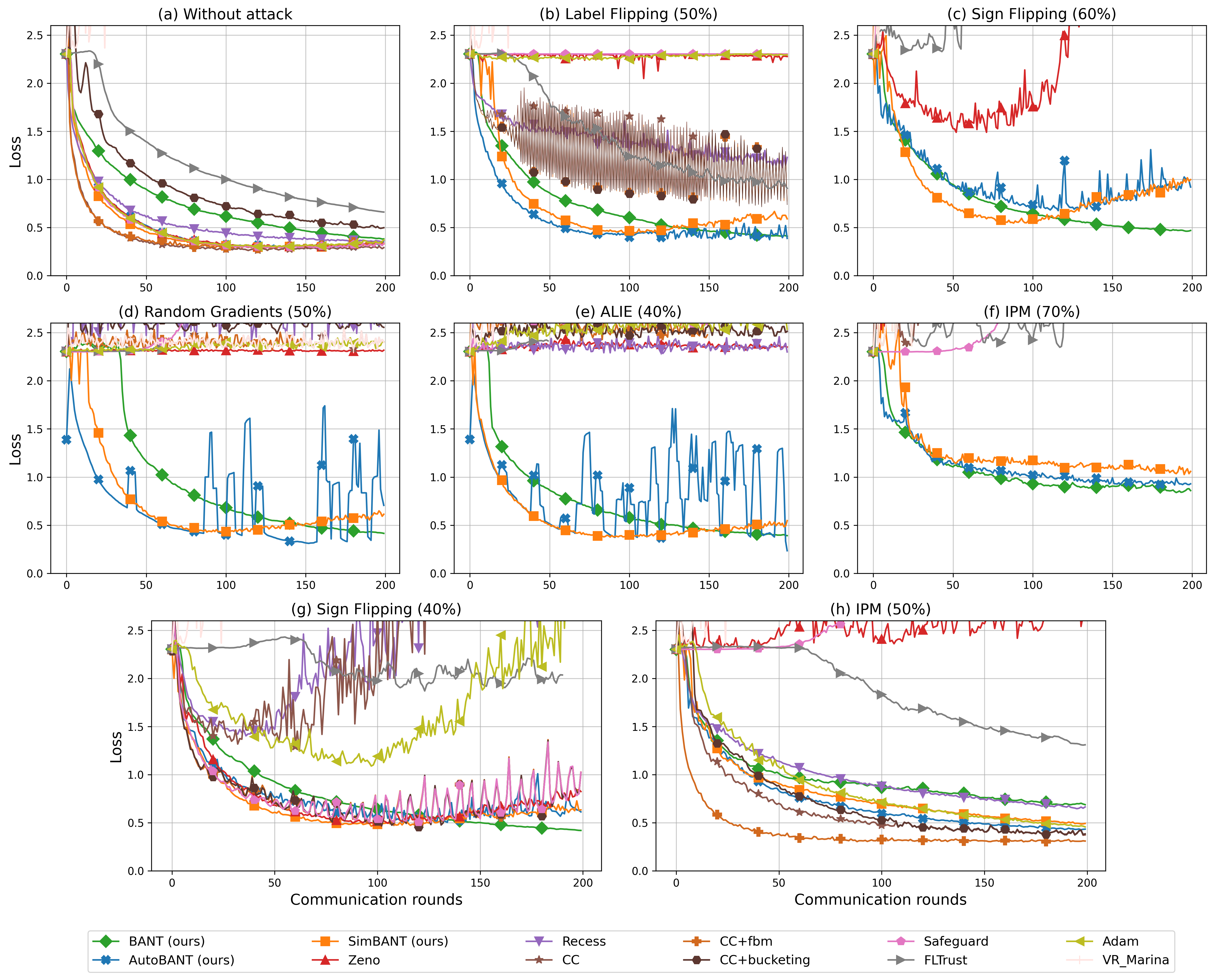}
    \caption{\textsc{ResNet18} on CIFAR-10. Loss on test for Byzantine-tolerance techniques under various attacks.}
    \label{fig:cifarloss}
\end{figure}

\begin{table}[htbp]
\centering
\captionsetup{justification=centering}
\caption{Time per communication round under ALIE attack (40\%). \textsc{ResNet18} on \textsc{CIFAR-10} and \textsc{ResNet1d18} on ECG (AFIB).}
\label{tab:time_alie_sidebyside}

\begin{minipage}{0.49\textwidth}
\centering
\captionof*{table}{\textsc{CIFAR-10} (ResNet18)}
\label{tab:cifartime}
\begin{tabular}{|l|c|}
\toprule
Algorithm & ALIE (40\%) \\
\midrule
\textsc{Adam} & 15.38 ± 0.94 \\
\textsc{FLTrust} & 36.38 ± 0.92 \\
\textsc{Recess} & 38.50 ± 3.14 \\
\textsc{Zeno} & 27.44 ± 2.36 \\
\textsc{CC} & 21.67 ± 3.14 \\
\textsc{CC+fbm} & 23.59 ± 3.98 \\
\textsc{CC+bucketing} & 21.73 ± 3.10 \\
\textsc{Safeguard} & 29.15 ± 1.16 \\
\textsc{VR Marina} & 36.96 ± 3.45 \\
\midrule
\texttt{Bant} & 27.94 ± 3.14 \\
\texttt{AutoBant} & 33.63 ± 1.74 \\
\texttt{SimBant} & 30.05 ± 4.35 \\
\bottomrule
\end{tabular}
\end{minipage}
\hfill
\begin{minipage}{0.49\textwidth}
\centering
\captionof*{table}{ECG (AFIB) (ResNet1d18)}
\label{tab:ecgtime_vertical}
\begin{tabular}{|l|c|}
\toprule
Method & ALIE (40\%) \\
\midrule
\textsc{Adam} & 32.14 ± 2.83 \\
\textsc{FLTrust} & 48.57 ± 8.17 \\
\textsc{Recess} & 49.55 ± 2.24 \\
\textsc{Zeno} & 34.23 ± 4.22 \\
\textsc{CC} & 43.36 ± 3.47 \\
\textsc{CC + fbm} & 45.84 ± 4.41 \\
\textsc{CC + bucketing} & 47.13 ± 3.56 \\
\textsc{Safeguard} & 35.67 ± 2.34 \\
\textsc{VR Marina} & 45.13 ± 6.41 \\
\midrule
\texttt{Bant} (ours) & 36.10 ± 2.63 \\
\texttt{AutoBant} (ours) & 41.34 ± 1.03 \\
\texttt{SimBant} (ours) & 36.24 ± 3.82 \\
\bottomrule
\end{tabular}
\end{minipage}
\end{table}

\paragraph{Time measurement.} Tables above report the average time per communication round, including standard deviation for various federated learning algorithms using \textsc{ResNet18} on \textsc{CIFAR-10} and  \textsc{ResNet1d18} on ECG (AFIB) in two settings: Without attack and ALIE. We see that \textsc{Recess} took the longest time, since it requires double local calculations of the client in one round of communication. \texttt{AutoBant} requires a little more time compared to baselines, except for the \textsc{Recess} and \textsc{FLTrust} methods. We attribute this to the solution of the auxiliary task Algorithm \ref{sec:mainscaled}, Line 10), which is slightly longer than in the ECG setup and has a greater impact on the execution time. However, it is not expected that the size of $\hat{f}$ increases significantly as the system scales. Thus, this contribution is negligible in real-world scenarios, as demonstrated by the ECG experiment in Table \ref{tab:ecgtime_vertical}. In addition, these computations are performed on the central node, which has more computational resources in the federated learning paradigm. All of this reflects the applicability of the proposed methods in real-world setups.

\paragraph{Sensitivity to Trial Set Size.}

The trial dataset plays a central role in all proposed methods, serving as a reference for evaluating client gradients via the surrogate loss $\hat f$. While our theoretical analysis suggests that the impact of finite sampling is mild (via $\zeta(N)$), it remains essential to validate this empirically -- particularly for small $N$, which is desirable in privacy-sensitive or resource-constrained settings. A robust method should maintain performance even when the server has access to only a limited trial set. To this end, we investigate the impact of the trial set size $N$ on convergence and stability for \texttt{Bant}, \texttt{SimBant}, and \texttt{AutoBant}. In the main experiments, we use $N = 500$. Here, we vary $N \in \{100,150,200,250,500,1000\}$.

\texttt{Bant} and \texttt{SimBant} demonstrate stable convergence across all values of $N$, even as low as 100, confirming their robustness to trial set sampling. \texttt{AutoBant}, however, exhibits higher sensitivity. At $N = 100$, convergence breaks down entirely, and even for $N \leq 200$, we observe increased variance and less stable updates. Nonetheless, performance remains strong in the range $N = 250$–$1000$, with fluctuations that do not degrade the overall results.
This behavior aligns with the design of \texttt{AutoBant}, which solves an optimization problem over client weights using noisy evaluations of the surrogate objective $\hat f$. When $N$ is too small, noise dominates, leading to unstable direction selection. In contrast, the trust-averaged updates in \texttt{Bant} and \texttt{SimBant} mitigate this effect and remain effective even under highly reduced supervision.

Nevertheless, we address the issue related to unstable minimization problem solving in the \texttt{AutoBant} algorithm below.

\begin{figure}[h]
    \centering
    \includegraphics[width=\textwidth]{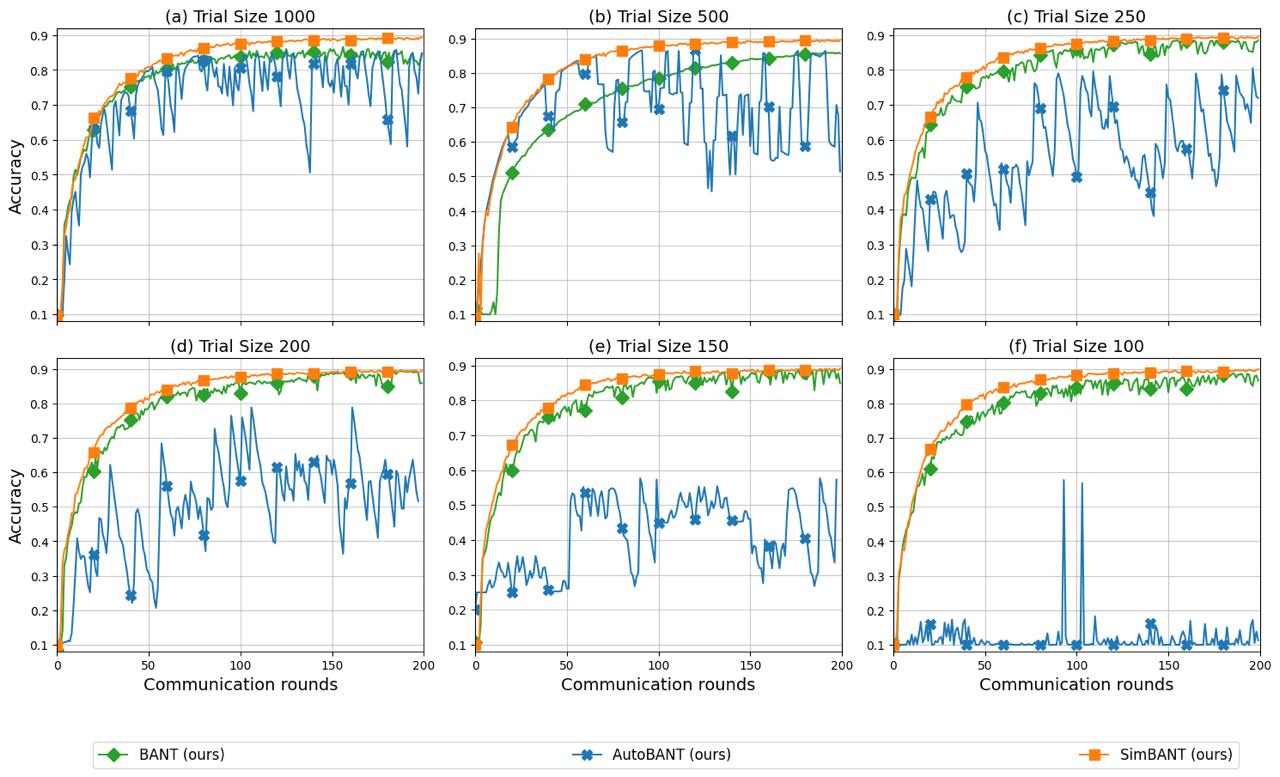}
    \caption{Test accuracy for \textsc{ResNet18} on \textsc{CIFAR-10} with different number of samples to obtain trial data.}
    \label{fig:trial_sample_size}
    \vspace{-4mm}
\end{figure}

\begin{minipage}{\columnwidth}
\centering
\includegraphics[width=0.5\linewidth]{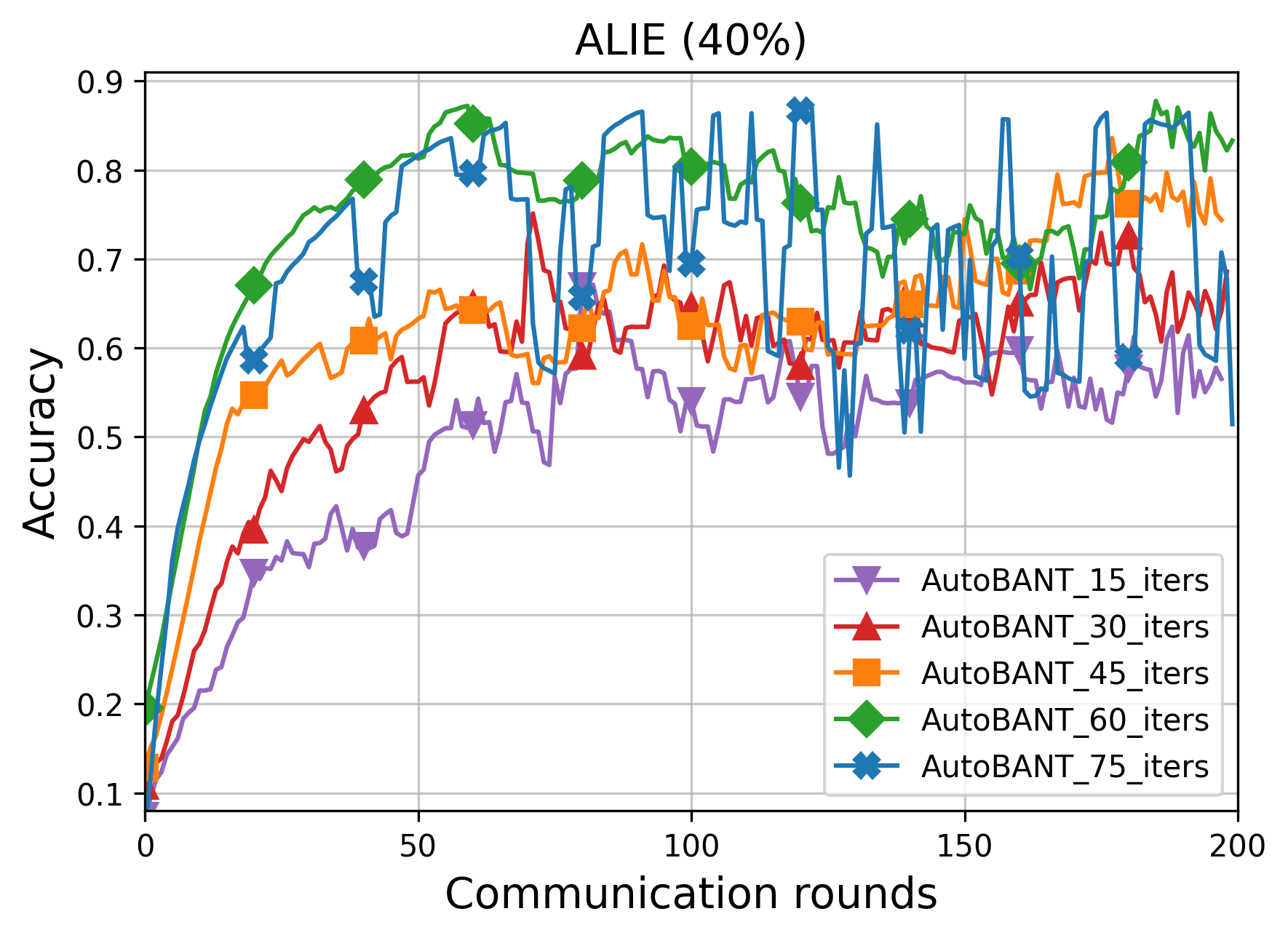}
\captionof{figure}{Number of Mirror Descent steps impact on global convergence}
\label{fig:delta_error}
\end{minipage}

\newpage
\paragraph{Study on precision of the auxiliary minimization task  in \texttt{AutoBant}.} 
\texttt{AutoBant} introduces an adaptive weighting mechanism via an optimization step over the client weight simplex. The quality of this step is controlled by the number of iterations in the mirror descent routine, corresponding to a target precision $\delta$. From a practical perspective, this parameter governs a trade-off between computational cost and stability. On the one hand, too coarse a solution may lead to noisy updates. On the other hand, the method may overfit transient fluctuations in the trial loss. Therefore, it is important to assess how sensitive the method is to this optimization accuracy and whether reliable convergence is retained under realistic constraints. In our main experiments, we set the number of mirror descent steps to $T=75$. Here, we vary $T \in \{15, 30, 45, 60, 75\}$.

\begin{minipage}{\textwidth}
\centering
\captionof{table}{Accuracy and training time from the Number of Mirror Descent steps.}
\label{tab:delta_error_time}
\begin{adjustbox}{width=0.4\linewidth}
\begin{tabular}{|c|c|c|c|}
\toprule
\textbf{T} & \textbf{Accuracy} & \textbf{Mirror Time} \\
\midrule
15 & 0.565 & \textbf{18.35} ± 4.19 \\
30 & 0.686 & 27.62 ± 5.11 \\
45 & 0.744 & 37.58 ± 3.16 \\
60 & \textbf{0.833} & 50.43 ± 4.25 \\
75 & 0.715 & 67.27 ± 3.49 \\
\bottomrule
\end{tabular}
\end{adjustbox}
\end{minipage}

We mention that due to the inherent stochasticity arisen from solving a convex minimization problem over the client weight simplex, the method is potentially unstable. To mitigate this, we adopt a natural stabilization strategy in our experimental setup. Thus, we leverage the history of client weights over the last round via momentum mechanism, thereby smoothing abrupt shifts that may result from sample-level variability. This prevents the algorithm from over-committing to transiently favorable clients and promotes consistent progress.

We observe that:
\begin{itemize}[leftmargin=*]
    \item At $T=15$ and $T=30$, a low deterioration in convergence occurs.
    \item At $T=45$ and $T=60$, convergence improves significantly, yielding smooth and high-quality updates.
    \item At $T=75$, performance becomes less stable due to overfitting to a single dominant client, a known issue when mirror descent pushes weights too aggressively toward one vertex of the simplex.
\end{itemize}

Despite this instability, convergence is preserved, and the method still outperforms baselines. Moreover, according to our convergence analysis (see Theorem \ref{TFMSGD}), the error introduced by inexact minimization--quantified by the optimization precision parameter $\delta$--only enters as a second-order additive term and does not fundamentally affect the convergence rate. As evident from Table \ref{tab:delta_error_time}, reducing the number of mirror descent iterations directly corresponds to a lower computational overhead, with $T=45$ requiring approximately $45\%$ less computation time than $T=75$ while still maintaining reasonable accuracy. This trade-off is particularly valuable for resource-constrained server node where computation time can bottleneck synchronization, affecting overall system efficiency. These findings highlight that \texttt{AutoBant} benefits from moderate optimization depth and additional smoothing over rounds, confirming that its performance can be controlled through simple and intuitive mechanisms without excessive parameter tuning.


\subsection{Stress Testing conditions} \label{sec:appendix_experiments_cifar_stress}

To further test the proposed methods, we examine the stress conditions of the experiments. For this purpose, the strongest Byzantine attacks (Random Gradients, ALIE, IPM) were considered, in which only \texttt{Bant}-like methods show resistance. The first experiment splits \textsc{CIFAR-10} homogeneously among 100 clients, while the second splits heterogeneously among 10 clients using the Dirichlet distribution.

\paragraph{\textsc{CIFAR-10} on 100 clients.} To test the robustness of the methods in a scalable setup, we split CIFAR-10 homogeneously among 100 clients. In this case, each client has less local data and its results are less representative compared to the Byzantine, which affects our methods when measuring $\hat{f}$. Due to computational challenges, we test only the ALIE (40\%) attack. Figure \ref{fig:cifar_100client} illustrates the convergence results on the test part of the dataset. 

\begin{minipage}{\columnwidth}
\centering
\includegraphics[width=0.5\linewidth]{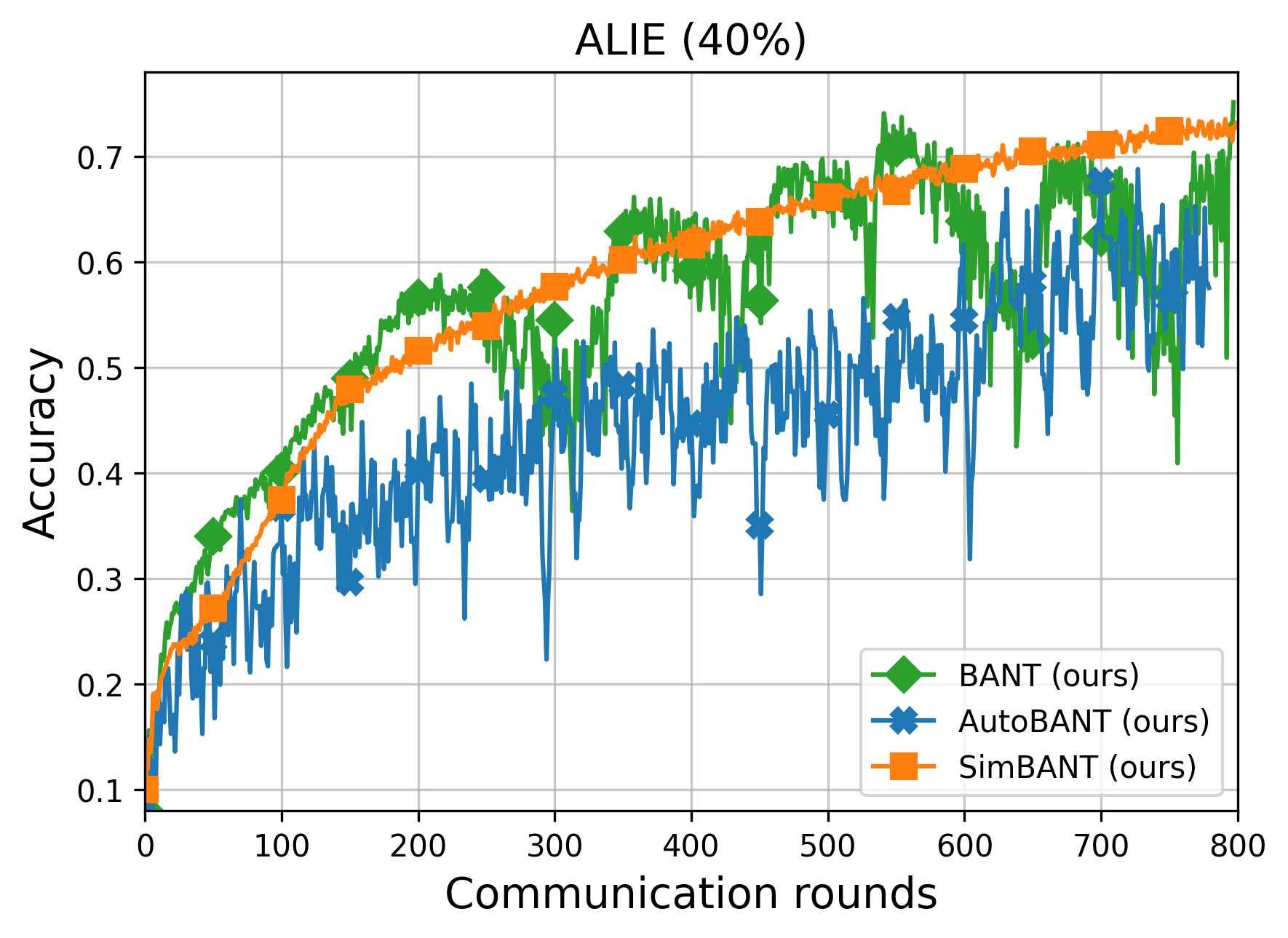}
\captionof{figure}{Test Accuracy for \textsc{ResNet18} on \textsc{CIFAR-10} with 100 clients.}
\label{fig:cifar_100client}
\end{minipage}

Comparing with the corresponding plot in Figure \ref{fig:cifaraccuracy}, we observe visible changes for the \texttt{Bant} and \texttt{AutoBant} methods. However, the effects described below are due to the complication of an already extreme setup. All methods demonstrate fundamental robustness, which was the goal of this experiment. More noisy results for \texttt{AutoBant} are related to solving an additional subproblem on a higher-dimensional simplex. As we mentioned in the main part of the paper, the method produces small non-zero trust scores for Byzantine clients, which is further highlighted when the number of clients increases. The \texttt{Bant} method directly compares the Byzantines to the global state of the model, so we observe periodic dips during federated training. All methods are trained over 800 rounds and exhibit slower convergence.

\paragraph{Dirichlet Scenario.} To examine the robustness of the proposed methods under a heterogeneous setup, we perform Byzantine attacks on \textsc{CIFAR-10} with Dirichlet distribution of client data \citep{yurochkin2019bayesian, wang2020federated, fedmix}. We partition each global dataset among \(n\) clients according to Dirichlet distribution with concentration parameter \(\alpha\). As \(\alpha\) decreases, the distributions on different clients become more skewed, which empirically increases both the inter‐client gradient variance and the bias of each client’s expected gradient from the true global gradient.  Concretely, under a Dirichlet(\(\alpha\)) partition, one observes that $\delta_1$ and $\delta_2$
grow monotonically as \(\alpha \to 0\) (stronger non‑IID) and shrink as \(\alpha \to \infty\) (nearly IID).



\begin{figure}[h]
    \centering
    \includegraphics[width=0.8\textwidth]{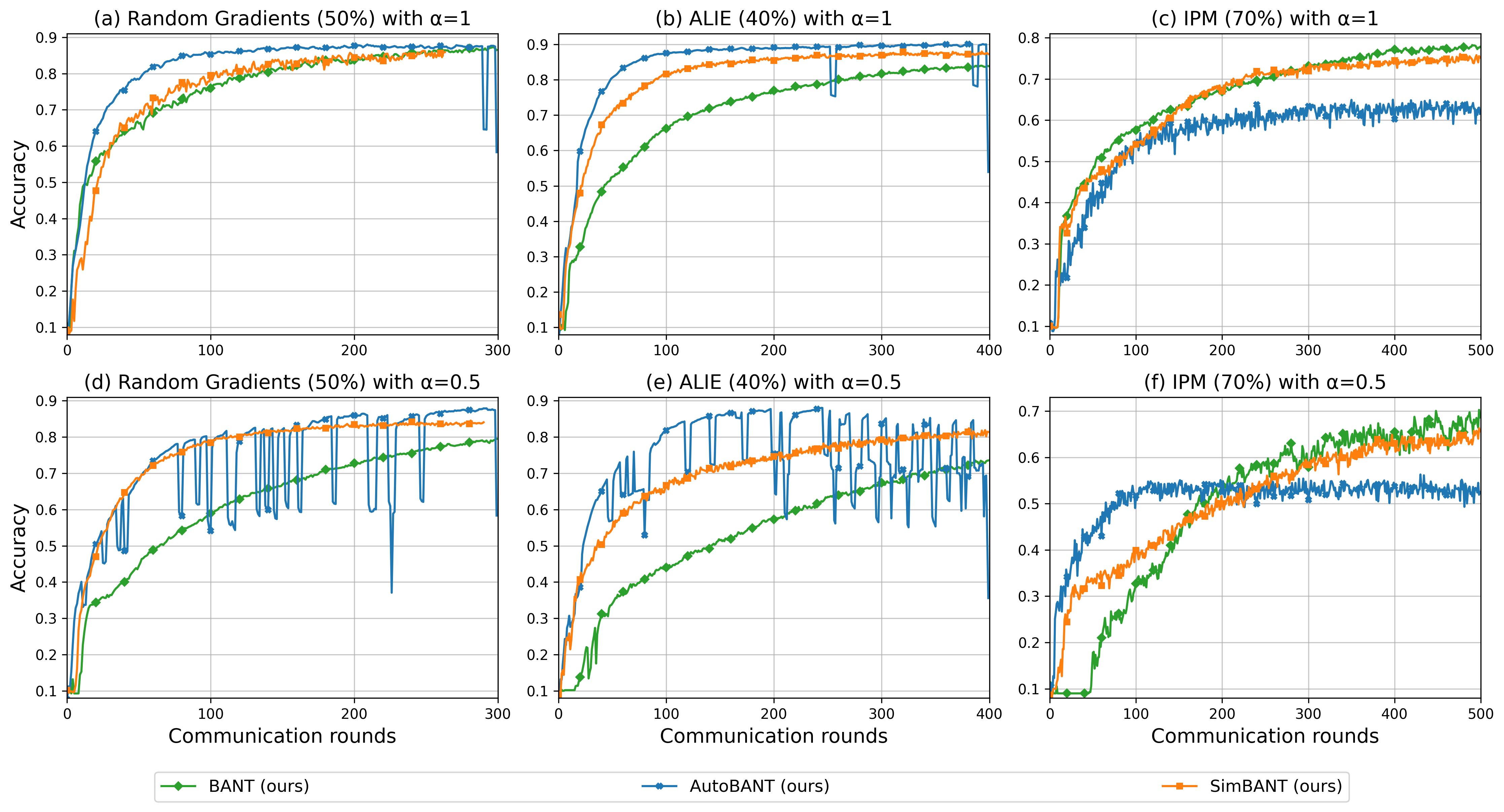}
    \caption{Test accuracy for \textsc{ResNet18} on \textsc{CIFAR-10} with Dirichlet heterogeneity.}
    \label{fig:cifardirichlet_acc_app}
\end{figure}

We note that $\mathbf{p}_k \sim \text{Dir}(\alpha \cdot \mathbf{1}_n)$ represents distribution of samples with label $k = \overline{1, K}$ over $n$ clients, i.e. $\mathbf{p}_k \in S_n(1)$ where $S_n(1)$ is the standard $n$ unit simplex.
We address two heterogeneity regimes: medium with $\alpha=1$ and strong with $\alpha=0.5$. Since the goal of the experiment is to test the robustness of the proposed methods, we consider the 3 strongest attacks in our setup: Random Gradient (50\%), ALIE (40\%), and IPM (70\%).
Figure \ref{fig:cifardirichlet_acc_app}, \ref{fig:cifardirichlet_loss_app} illustrates the performance of the \texttt{Bant}-based methods.

In the presence of Dirichlet‐induced heterogeneity and high‐fraction Byzantine attacks, the proposed automated defenses exhibit fast convergence and generalization compared to the original Bant protocol. As shown in Figures \ref{fig:cifardirichlet_acc_app} (a - c) and \ref{fig:cifardirichlet_loss_app} (a - c), under mild skew ($\alpha=1$) and Random Gradient, ALIE, or IPM attacks, AutoBant drives test accuracy above 80\% and reduces test loss below 0.7 in roughly 100 communication rounds, approximately half the rounds required by Bant, while SimBant achieve intermediate performance. 

As illustrated in Figures \ref{fig:cifardirichlet_acc_app} (d - f) and \ref{fig:cifardirichlet_loss_app} (d - f), when the heterogeneity is intensified ($\alpha=0.5$), Bant’s accuracy curves develop large oscillations and its loss decay slows substantially, whereas both AutoBant and SimBant maintain smooth, rapid descent to peak accuracies of 85–88\% and asymptotic losses in the 0.5–0.9 range.

\begin{figure}[H]
    \centering
    \includegraphics[width=0.8\textwidth]{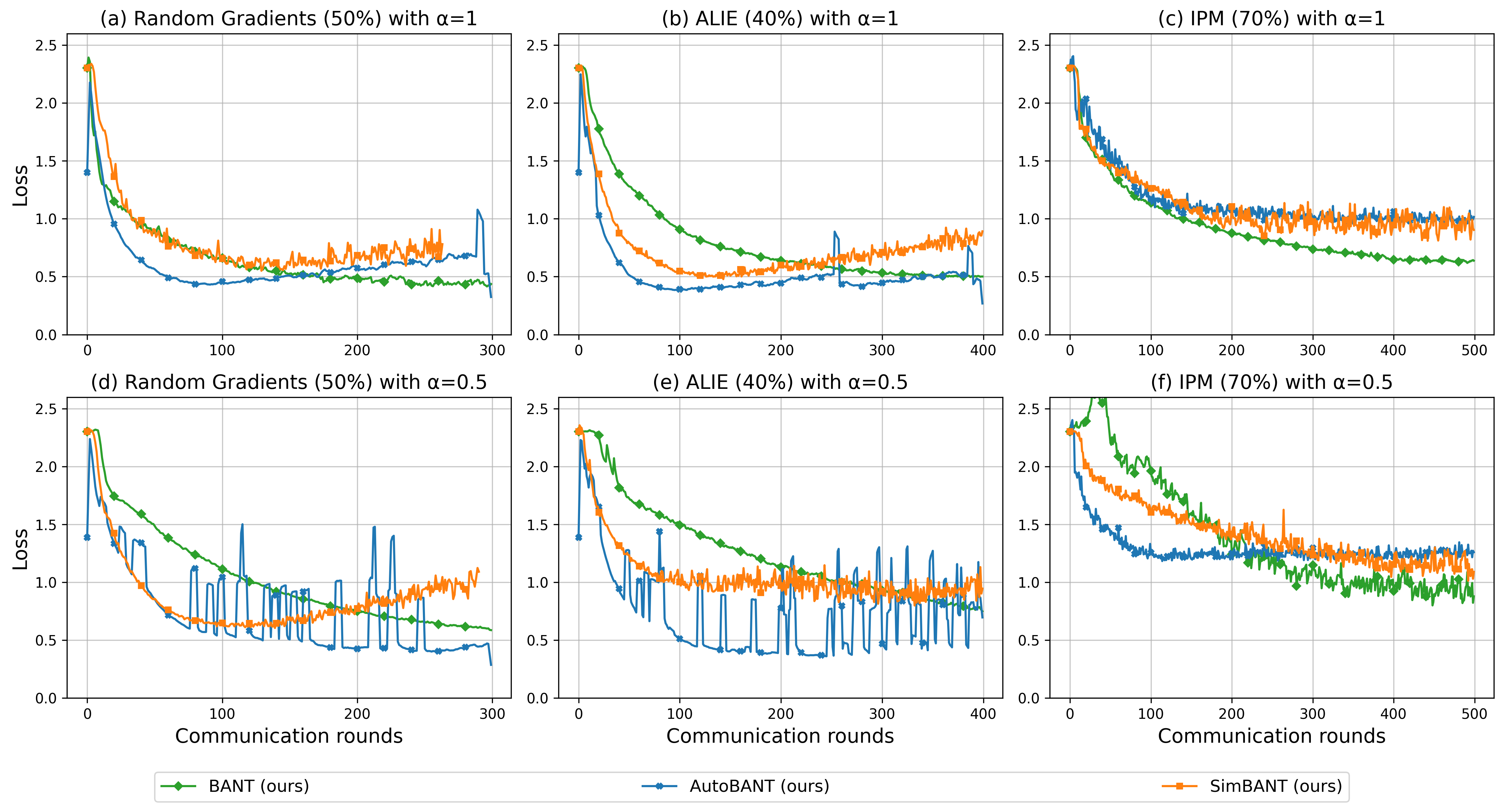}
    \caption{Test loss for \textsc{ResNet18} on \textsc{CIFAR-10} with Dirichlet heterogeneity.}
    \label{fig:cifardirichlet_loss_app}
\end{figure}



\subsection{ECG Experiments} \label{sec:appendix_experiments_ecg}

\paragraph{On \textsc{Zeno}'s sensitivity to hyperperameters.}
Now, let us move on to the choice of the trim hyperparameter $b$ in the \textsc{Zeno} method. This parameter cuts off the number of clients that will not be used in the aggregation of the global model in the training round. The ranking of clients is based on $Score_{\gamma, \rho}$, which exploits the trial function approach. In the paper, it is assumed that $b \geqslant q$, where $q$ is the number of Byzantines participating in the training. Thus, the method provides protection against an arbitrary number of malicious participants and can be stable in critical attack setups. However, in real-world scenarios, we do not have a priori information about the number of Byzantines. As a result, \textsc{Zeno} does not show comparable quality metrics in the ECG setup for $b=2$ compared to the \texttt{Bant}-like methods, which is reflected in Table \ref{tab:ecggmean}.
We address the issue of choosing this hyperparameter in the ECG setup in Table \ref{tab:ecggmean_zeno}. As can be seen from the results, a suitably selected trim coefficient value of $b=3$, which cuts off $60\%$ of Byzantine clients, shows comparable results with the \texttt{Bant}-like methods in all considered attack scenarios. Cutting 80\% of Byzantines leads to a slightly lower quality, which also highlights the hyperparameter sensitivity.


\begin{table}[t]
\centering
\caption{\textsc{ResNet1d18} on ECG (AFIB). G-mean and f1-score for Zeno under various attacks. The value in parentheses indicates the assumed number of Byzantines.}
\label{tab:ecggmean_zeno}
\resizebox{\textwidth}{!}{
\begin{tabular}{|p{2cm} | c c |  c c | c c | c c | c c | c c | c c | c c}
\toprule 
Algorithm & 
\multicolumn{2}{c|}{\begin{tabular}{c|c} \multicolumn{2}{c}{Without Attack} \\ \hline G-mean & f1-score \end{tabular}} &
\multicolumn{2}{c|}{\begin{tabular}{c|c} \multicolumn{2}{c}{Label Flipping (60\%)} \\ \hline G-mean & f1-score \end{tabular}} &
\multicolumn{2}{c|}{\begin{tabular}{c|c} \multicolumn{2}{c}{Random Gradients (60\%)} \\ \hline G-mean & f1-score \end{tabular}} &
\multicolumn{2}{c|}{\begin{tabular}{c|c} \multicolumn{2}{c}{IPM (80 \%)} \\ \hline G-mean & f1-score \end{tabular}} &
\multicolumn{2}{c|}{\begin{tabular}{c|c} \multicolumn{2}{c}{ALIE (40 \%)} \\ \hline G-mean & f1-score \end{tabular}} \\
\midrule
\textsc{Zeno (80\%)} & 0.860$\pm$0.016 & 0.707$\pm$0.015 & 0.735$\pm$0.023 & 0.619$\pm$0.010 & 0.824$\pm$0.018 & 0.696$\pm$0.014 & 0.930$\pm$0.011 & 0.715$\pm$0.023 & 0.930$\pm$0.017 & 0.705$\pm$0.013 \\
\textsc{Zeno (60\%)} & \textbf{0.953}$\pm$0.018 & \textbf{0.806}$\pm$0.017 & \textbf{0.950}$\pm$0.020 & \textbf{0.753}$\pm$0.019 & \textbf{0.954}$\pm$0.020 & \textbf{0.770}$\pm$0.017 & \textbf{0.945}$\pm$0.017 & \textbf{0.730}$\pm$0.020 & \textbf{0.946}$\pm$0.015 & \textbf{0.717}$\pm$0.016 \\
\hline
\textsc{Zeno} & 0.921$\pm$0.012 & 0.787$\pm$0.014 & 0.014$\pm$0.017 & 0.110$\pm$0.015 & 0.163$\pm$0.010 & 0.089$\pm$0.014 & 0.102$\pm$0.012 & 0.066$\pm$0.018 & 0.010$\pm$0.009 & 0.091$\pm$0.011 \\
\bottomrule
\end{tabular}
}

\end{table}

\paragraph{Metrics for all abnormalities.} Here, we present all the obtained metric results for the four heart abnormalities we investigated: Atrial FIBrillation (AFIB), First-degree AV block (1AVB), Premature Ventricular Complex (PVC), and Complete Left Bundle Branch Block (CLBBB). In Table \ref{tab:ecg2attacks}, we show results for AFIB pathology under 2 additional attacks scenarios: Sign Flipping $(60\%)$ and IPM $(60\%)$. In addition, we compare all methods for the attacks described in \ref{sec:experiments} section in Tables \ref{tab:ecgafib} - \ref{tab:ecgclbbb}. To obtain confidence intervals in Tables \ref{tab:ecggmean}, \ref{tab:ecggmean_zeno}, \ref{tab:ecg2attacks}, each run was repeated $5$ times. We do not report the variance of metrics for the 1AVB, PVC and CLBBB pathologies in Tables \ref{tab:ecgavb}-\ref{tab:ecgclbbb} due to overabundance and computational factor.

\begin{table}[H]
\centering
\caption{\textsc{ResNet1d18} on ECG (AFIB). G-mean and f1-score for Byzantine-tolerance techniques under 2 attacks.}
\label{tab:ecg2attacks}
\resizebox{0.5\textheight}{!}{
\begin{tabular}{|p{2cm} | c c |  c c | c c | c c | c c | c c | c c | c c}
\toprule 
Algorithm & 
\multicolumn{2}{c|}{\begin{tabular}{c|c} \multicolumn{2}{c}{Sign Flipping (60 \%)} \\ \hline G-mean & f1-score \end{tabular}} &
\multicolumn{2}{c|}{\begin{tabular}{c|c} \multicolumn{2}{c}{IPM (60\%)} \\ \hline G-mean & f1-score \end{tabular}}
 \\
\midrule
\textsc{Adam}     & 0.304$\pm$0.015 & 0.116$\pm$0.018 & 0.952$\pm$0.014 & 0.738$\pm$0.011  \\
\textsc{FLTrust}  & 0.586$\pm$0.018 & 0.179$\pm$0.015 & 0.011$\pm$0.014 & 0.123$\pm$0.013  \\
\textsc{Recess}   & 0.359$\pm$0.018 & 0.115$\pm$0.014& 0.933$\pm$0.017 & 0.611$\pm$0.018  \\
\textsc{Zeno}     & 0.017$\pm$0.016 & 0.130$\pm$0.018& 0.010$\pm$0.018 & 0.140$\pm$0.020  \\
\textsc{CC}       & 0.479$\pm$0.020 & 0.124$\pm$0.017 & 0.945$\pm$0.020 & 0.710$\pm$0.016  \\
\textsc{CC+fbm}   & 0.155$\pm$0.016 & 0.124$\pm$0.017 & 0.948$\pm$0.018 & 0.695$\pm$0.020  \\
\footnotesize\textsc{CC+bucketing}   & 0.137$\pm$0.017& 0.119$\pm$0.010 & 0.944$\pm$0.016 & 0.689$\pm$0.013\\
\textsc{Safeguard} & 0.084$\pm$0.016 & 0.014$\pm$0.018 & 0.109$\pm$0.010 & 0.123$\pm$0.016 \\
\textsc{VR Marina} & 0.096$\pm$0.017& 0.078$\pm$0.019 & 0.098$\pm$0.011& 0.110$\pm$0.014\\
\midrule
\texttt{Bant}     & 0.943$\pm$0.019 & \textbf{0.792$\pm$0.019}& 0.949$\pm$0.020 & 0.704$\pm$0.017 \\
\texttt{AutoBant} & 0.737$\pm$0.020 & 0.243$\pm$0.019 & 0.948$\pm$0.018 & 0.695$\pm$0.015 \\
\texttt{SimBant}  & \textbf{0.951$\pm$0.020} & 0.760$\pm$0.020 & \textbf{0.965$\pm$0.017} & \textbf{0.753$\pm$0.020}\\
\bottomrule
\end{tabular}
}

\end{table}

\newpage
\begin{table}[H]
\begin{minipage}{0.4\textwidth}
    \centering
    \captionof{table}{\textsc{ResNet1d18} on ECG (AFIB).}
\label{tab:ecgafib} 
  \begin{threeparttable}
  \resizebox{\textwidth}{!}{
    \begin{tabular}{l|l|c|c|c|c}
    \toprule
    \multicolumn{1}{c|}{} & Algorithm & Sensitivity & Specificity & G-mean & f1-score
    \\
    \midrule
    \multirow{6}{*}{\rotatebox[origin=c]{90}{\parbox[c]{3.5cm}{\large Without Attack}}} 
    & \textsc{Adam} & \textit{0.940} & 0.974 & \textit{0.956}& 0.811 \\
    & \textsc{FLTrust} & 0.932 & 0.972 & 0.952 & 0.800 \\
    & \textsc{Recess} & 0.929 & 0.969 & 0.949 & 0.783 \\
    & \textsc{Zeno} & 0.890 & 0.952 & 0.921 & 0.787 \\
    & \textsc{CC} & 0.932 & 0.966 & 0.949 & 0.772 \\
    & \textsc{CC+fbm} & 0.930 & 0.978 & 0.954 & 0.808 \\
    & \textsc{CC+bucketing} & 0.910 & 0.985 & 0.947 & 0.790 \\
    & \textsc{Safeguard} & \textit{0.940}& \textit{0.976} & \textbf{0.957} & \textit{0.821} \\
    & \textsc{VR Marina} & 0.150 & 0.001 & 0.010 & 0.120 \\
    & \texttt{Bant} & 0.929 & \textbf{0.978 }& 0.953 & \textbf{0.830} \\
    & \texttt{AutoBant} & \textit{0.940 }& 0.967 & 0.953 & 0.781 \\
    & \texttt{SimBant} & \textbf{0.943} & 0.969 & \textit{0.956} & 0.790 \\
    \midrule
    \multirow{6}{*}{\rotatebox[origin=c]{90}{\parbox[c]{3.75cm}{\large Label Flip (60\%)}}}
    & \textsc{Adam} & 0.089 & 0.774 & 0.262 & 0.041 \\
    & \textsc{FLTrust} & 0.943 & 0.961 & \textit{0.952} & 0.753 \\
    & \textsc{Recess} & 0.146 & 0.920 & 0.366 & 0.128 \\
    & \textsc{Zeno} & 0.160 & 0.001 & 0.014 & 0.110 \\
    & \textsc{CC} & 0.840 & 0.097 & 0.285 & 0.114 \\
    & \textsc{CC+fbm} & 0.880 & 0.801 & 0.840 & 0.716 \\
    & \textsc{CC+bucketing} & 0.870 & 0.789 & 0.829 & 0.708 \\
    & \textsc{Safeguard} & 0.018 & 0.063 & 0.107 & 0.123 \\
    & \textsc{VR Marina} & 0.210 & 0.003 & 0.027 & 0.123 \\
    & \texttt{Bant} & \textit{0.947} & \textit{0.966} & \textbf{0.956} & \textbf{0.777} \\
    & \texttt{AutoBant} & \textbf{0.964} & 0.647 & 0.790 & 0.276 \\
    & \texttt{SimBant} & 0.932 & \textbf{0.967} & 0.949 & \textit{0.774} \\ 
    \midrule
    \multirow{6}{*}{\rotatebox[origin=c]{90}{\parbox[c]{3.7cm}{\large Sign Flip (60\%)}}}
    & \textsc{Adam} & 0.096 & 0.961 & 0.304 & 0.116 \\
    & \textsc{FLTrust} & 0.466 & 0.738 & 0.586 & 0.179 \\
    & \textsc{Recess} & 0.142 & 0.906 & 0.359 & 0.115 \\
    & \textsc{Zeno} & 0.230 & 0.001 & 0.017 & 0.130\\
    & \textsc{CC} & 0.328 & 0.701 & 0.479 & 0.124 \\
    & \textsc{CC+fbm} & 0.270 & 0.089 & 0.155 & 0.124 \\
    & \textsc{CC+bucketing} & 0.250 & 0.075 & 0.137 & 0.119 \\
    & \textsc{Safeguard} & 0.007 & \textbf{0.998} & 0.084 & 0.014 \\
    & \textsc{VR Marina} & 0.150 & 0.061 & 0.096 & 0.078 \\
    & \texttt{Bant} & \textit{0.915} & \textit{0.972} & \textit{0.943} & \textbf{0.792} \\
    & \texttt{AutoBant} & 0.854 & 0.636 & 0.737 & 0.243 \\
    & \texttt{SimBant} & \textbf{0.940} & 0.963 & \textbf{0.951} & \textit{0.760}
    \\
    \midrule
    \multirow{6}{*}{\rotatebox[origin=c]{90}{\parbox[c]{4.1cm}{Random Gradients (60\%)}}}
    & \textsc{Adam} & 0.136 & 0.893 & 0.348 & 0.126 \\
    & \textsc{FLTrust} & 0.623 & 0.611 & 0.617 & 0.174 \\
    & \textsc{Recess} & 0.562 & 0.626 & 0.593 & 0.163 \\
    & \textsc{Zeno} & 0.300 & 0.088 & 0.163 & 0.089 \\
    & \textsc{CC} & 0.459 & 0.733 & 0.580 & 0.155 \\
    & \textsc{CC+fbm} & 0.500 & 0.632 & 0.562 & 0.252 \\
    & \textsc{CC+bucketing} & 0.420 & 0.774 & 0.570 & 0.164  \\
    & \textsc{Safeguard} & 0.929 & 0.071 & 0.258 & 0.124 \\
    & \textsc{VR Marina} & 0.320 & 0.096 & 0.176 & 0.103 \\
    & \texttt{Bant} & 0.922 & \textbf{0.975} & \textbf{0.948} & \textbf{0.809} \\
    & \texttt{AutoBant} & \textit{0.932} & \textit{0.961} & \textit{0.946} & \textit{0.748} \\
    & \texttt{SimBant} & \textbf{0.951} & 0.940 & 0.945 & 0.712 
    \\
    \midrule
    \multirow{6}{*}{\rotatebox[origin=c]{90}{\parbox[c]{3.35cm}{\large IPM (60 \%)}}}
    & \textsc{Adam} & 0.947 & \textit{0.957} & \textit{0.952} & \textit{0.738} \\
    & \textsc{FLTrust} & 0.210 & 0.001 & 0.011 & 0.123 \\
    & \textsc{Recess} & 0.947 & 0.919 & 0.933 & 0.611 \\
    & \textsc{Zeno} & 0.200 & 0.001 & 0.010 & 0.140 \\
    & \textsc{CC} & 0.940 & 0.950 & 0.945 & 0.710 \\
    & \textsc{CC+fbm} & \textit{0.950} & 0.946 & 0.948 & 0.695 \\
    & \textsc{CC+bucketing} & 0.940 & 0.948 & 0.944 & 0.689 \\
    & \textsc{Safeguard} & 0.190 & 0.062 & 0.109 & 0.123 \\
    & \textsc{VR Marina} & 0.170 & 0.056 & 0.098 & 0.110 \\
    & \texttt{Bant} & \textit{0.950} & 0.947 & 0.949 & 0.704 \\
    & \texttt{AutoBant} & \textit{0.950} & 0.945 & 0.948 & 0.695 \\
    & \texttt{SimBant} & \textbf{0.974} & \textbf{0.955} & \textbf{0.965} & \textbf{0.753}
    \\
    \midrule
    \multirow{6}{*}{\rotatebox[origin=c]{90}{\parbox[c]{3.35cm}{\large IPM (80 \%)}}}
    & \textsc{Adam} & 0.043 & 0.904 & 0.197 & 0.036 \\
    & \textsc{FLTrust} & 0.220 & 0.016 & 0.061 & 0.124 \\
    & \textsc{Recess} & 0.456 & 0.533 & 0.493 & 0.112 \\
    & \textsc{Zeno} & 0.131 & 0.080 & 0.102 & 0.066 \\
    & \textsc{CC} & 0.007 & \textbf{0.997} & 0.084 & 0.014 \\
    & \textsc{CC+fbm} & 0.210 & 0.003 & 0.027 & 0.123 \\
    & \textsc{CC+bucketing} & 0.220 & 0.005 & 0.035 & 0.118 \\
    & \textsc{Safeguard} & 0.200 & 0.060 & 0.110 & 0.082 \\
    & \textsc{VR Marina} & 0.240 & 0.067 & 0.127 & 0.079 \\
    & \texttt{Bant} & \textbf{0.954} & 0.939 & \textit{0.946} & 0.676 \\
    & \texttt{AutoBant} & 0.940 & 0.945 & 0.942 & \textit{0.690} \\
    & \texttt{SimBant} & \textit{0.950} & \textit{0.960} & \textbf{0.955} & \textbf{0.783} 
    \\
    \midrule
    \multirow{6}{*}{\rotatebox[origin=c]{90}{\parbox[c]{3.2cm}{\large ALIE (40\%)}}}
    & \textsc{Adam} & 0.265 & 0.058 & 0.125 & 0.123 \\
    & \textsc{FLTrust} & 0.220 & 0.001 & 0.017 & 0.123 \\
    & \textsc{Recess} & 0.249 & \textit{0.811} & 0.450 & 0.127 \\
    & \textsc{Zeno} & 0.180 & 0.001 & 0.010 & 0.091 \\
    & \textsc{CC} & 0.370 & 0.758 & 0.530 & 0.154 \\
    & \textsc{CC+fbm} & 0.870 & 0.882 & 0.876 & 0.594 \\
    & \textsc{CC+bucketing} & 0.890 & 0.850 & 0.870 & 0.587 \\
    & \textsc{Safeguard} & 0.150 & 0.000 & 0.010 & 0.123 \\
    & \textsc{VR Marina} & 0.200 & 0.001 & 0.012 & 0.108 \\
    & \texttt{Bant} & \textit{0.929} & \textbf{0.966} & \textbf{0.947} & \textbf{0.770} \\
    & \texttt{AutoBant} & 0.861 & 0.924 & 0.892 & 0.585 \\
    & \texttt{SimBant} & \textit{0.943} & \textit{0.949} & \textbf{0.946} & \textit{0.705}
    \\
    \bottomrule
    \end{tabular}   
    }
    \end{threeparttable}
\vspace{2mm}
\end{minipage}
\hfill
\begin{minipage}{0.4\textwidth}
    \centering
    \captionof{table}{\textsc{ResNet1d18} on ECG (1AVB).}
\label{tab:ecgavb} 
  \begin{threeparttable}
  \resizebox{\textwidth}{!}{
    \begin{tabular}{l|l|c|c|c|c}
    \toprule
    \multicolumn{1}{c|}{} & Algorithm & Sensitivity & Specificity & G-mean & f1-score
    \\

\midrule
     
    \multirow{6}{*}{\rotatebox[origin=c]{90}{\parbox[c]{3.5cm}{\large Without Attack}}} 
    & \textsc{Adam} & \textit{0.896} & 0.871 & \textit{0.884} & 0.335 \\
    & \textsc{FLTrust} & 0.857 & 0.883 & 0.870 & 0.344 \\
    & \textsc{Recess} & \textbf{0.909} & 0.870 & \textbf{0.889} & 0.337 \\
    & \textsc{Zeno} & 0.890 & 0.883 & \textit{0.886} & 0.353 \\
    & \textsc{CC} & 0.864 & 0.896 & 0.880 & 0.371 \\
    & \textsc{CC+fbm} & 0.888 & 0.868 & 0.877 & 0.353 \\
    & \textsc{CC+bucketing} & 0.869 & 0.887 & 0.877 & 0.351 \\
    & \textsc{Safeguard} & 0.877 & 0.883 & 0.880 & 0.350 \\
    & \textsc{VR Marina} & 0.825 & \textbf{0.922} & 0.872 & \textbf{0.420}\\
    & \texttt{Bant} & 0.831 & \textit{0.916} & 0.873 & \textit{0.408} \\
    & \texttt{AutoBant} & 0.825 & \textbf{0.922} & 0.872 & \textbf{0.420} \\
    & \textsc{FineTuned} & 0.851 & 0.894 & 0.872 & 0.362 \\
    \midrule
    \multirow{6}{*}{\rotatebox[origin=c]{90}{\parbox[c]{3.75cm}{\large Label Flip (60\%)}}}
    & \textsc{Adam} & 0.210 & 0.002 & 0.022 & 0.069 \\
    & \textsc{FLTrust} & 0.870 & \textit{0.903} & \textbf{0.886} & \textit{0.388} \\
    & \textsc{Recess} & 0.210 & 0.001 & 0.017 & 0.113 \\
    & \textsc{Zeno} & 0.400 & 0.039 & 0.125 & 0.309 \\
    & \textsc{CC} & 0.290 & 0.120 & 0.187 & 0.112 \\
    & \textsc{CC+fbm} & 0.725 & 0.866 & 0.792 & 0.265 \\
    & \textsc{CC+bucketing} & 0.740 & 0.818 & 0.778 & 0.305 \\
    & \textsc{Safeguard} & 0.100 & 0.001 & 0.010 & 0.069 \\
    & \textsc{VR Marina} & 0.220 & 0.002 & 0.023 & 0.125 \\
    & \texttt{Bant} & 0.812 & \textbf{0.934} & 0.871 & \textbf{0.455} \\
    & \texttt{AutoBant} & \textit{0.877} & 0.895 & \textbf{0.886} & 0.373 \\
    & \textsc{FineTuned} & \textbf{0.890} & 0.857 & \textit{0.873} & 0.311 \\
    \midrule
    \multirow{6}{*}{\rotatebox[origin=c]{90}{\parbox[c]{3.7cm}{\large Sign Flip (60\%)}}}
    & \textsc{Adam} & 0.597 & 0.685 & 0.640 & 0.119 \\
    & \textsc{FLTrust} & 0.007 & \textbf{0.980} & 0.080 & 0.008 \\
    & \textsc{Recess} & 0.058 & 0.874 & 0.226 & 0.026 \\
    & \textsc{Zeno} & 0.190 & 0.079 & 0.123 & 0.015 \\
    & \textsc{CC} & \textbf{1.000} & 0.002 & 0.041 & 0.070 \\
    & \textsc{CC+fbm} & 0.260 & 0.080 & 0.145 & 0.122 \\
    & \textsc{CC+bucketing} & 0.250 & 0.076 & 0.138 & 0.117 \\
    & \textsc{Safeguard} & \textit{0.981} & 0.030 & 0.170 & 0.070 \\
    & \textsc{VR Marina} & 0.130 & 0.058 & 0.087 & 0.065 \\
    & \texttt{Bant} & 0.669 & \textit{0.951} & \textit{0.797} & \textbf{0.447} \\
    & \texttt{AutoBant} & 0.916 & 0.845 & \textbf{0.879} & \textit{0.301} \\
    & \textsc{FineTuned} & 0.617 & 0.616 & 0.616 & 0.103 \\ 
    \midrule
    \multirow{6}{*}{\rotatebox[origin=c]{90}{\parbox[c]{4.1cm}{Random Gradients (60\%)}}}
    & \textsc{Adam} & 0.220 & 0.005 & 0.036 & 0.074 \\
    & \textsc{FLTrust} & \textbf{1.000} & 0.001 & 0.027 & 0.069 \\
    & \textsc{Recess} & 0.468 & 0.757 & 0.595 & 0.117 \\
    & \textsc{Zeno} & 0.220 & 0.061 & 0.116 & 0.098 \\
    & \textsc{CC} & \textit{0.909} & 0.227 & 0.455 & 0.080 \\
    & \textsc{CC+fbm} & 0.610 & 0.390 & 0.487 & \textit{0.362} \\
    & \textsc{CC+bucketing} & 0.650 & 0.416 & 0.520 & 0.340 \\
    & \textsc{Safeguard} & \textbf{1.000} & 0.020 & 0.140 & 0.071 \\
    & \textsc{VR Marina} & 0.230 & 0.057 & 0.115 & 0.103 \\
    & \texttt{Bant} & 0.805 & \textbf{0.919} & 0.860 & \textbf{0.405 }\\
    & \texttt{AutoBant} & 0.896 & 0.864 & \textbf{0.880} & 0.324 \\
    & \textsc{FineTuned} & 0.857 & \textit{0.890} & \textit{0.873} & 0.357 \\
    \midrule
    \multirow{6}{*}{\rotatebox[origin=c]{90}{\parbox[c]{3.35cm}{\large IPM (60 \%)}}}
    & \textsc{Adam} & \textbf{0.933} & 0.622 & 0.762 & 0.154 \\
    & \textsc{FLTrust} & 0.013 & \textbf{0.964} & 0.112 & 0.013 \\
    & \textsc{Recess} & \textit{0.922 }& 0.852 & \textbf{0.886} & 0.313 \\
    & \textsc{Zeno} & 0.260 & 0.169 & 0.210 & 0.098 \\
    & \textsc{CC} & 0.571 & 0.649 & 0.609 & 0.104 \\
    & \textsc{CC+fbm} & 0.700 & 0.480 & 0.580 & 0.130 \\
    & \textsc{CC+bucketing} & 0.610 & 0.374 & 0.478 & 0.142 \\
    & \textsc{Safeguard} & 0.050 & 0.001 & 0.008 & 0.078 \\
    & \textsc{VR Marina} & 0.200 & 0.036 & 0.085 & 0.115 \\
    & \texttt{Bant} & 0.721 & \textit{0.945} & 0.825 & \textbf{0.449} \\
    & \texttt{AutoBant} & 0.890 & 0.865 & \textit{0.877} & 0.323 \\
    & \textsc{FineTuned} & 0.857 & 0.893 & 0.875 & \textit{0.363 }\\
    \midrule
    \multirow{6}{*}{\rotatebox[origin=c]{90}{\parbox[c]{3.35cm}{\large IPM (80 \%)}}}
    & \textsc{Adam} & 0.312 & 0.580 & 0.425 & 0.050 \\
    & \textsc{FLTrust} & 0.201 & 0.751 & 0.389 & 0.051 \\
    & \textsc{Recess} & 0.558 & 0.508 & 0.533 & 0.076 \\
    & \textsc{Zeno} & 0.240 & 0.065 & 0.125 & 0.110 \\
    & \textsc{CC} & 0.857 & 0.310 & 0.515 & 0.084 \\
    & \textsc{CC+fbm} & 0.220 & 0.220 & 0.038 & 0.113 \\
    & \textsc{CC+bucketing} & 0.230 & 0.008 & 0.045 & 0.107 \\
    & \textsc{Safeguard} & 0.481 & 0.499 & 0.490 & 0.064 \\
    & \textsc{VR Marina} & 0.260 & 0.063 & 0.128 & 0.083 \\
    & \texttt{Bant} & \textit{0.857} & \textbf{0.884} & 0.870 & \textit{0.345} \\
    & \texttt{AutoBant} & 0.851 & 0.900 & \textit{0.875} & \textbf{0.375 }\\
    & \textsc{FineTuned} & \textbf{0.903} & \textit{0.876} & \textbf{0.889} & 0.344 \\
    \midrule
    \multirow{6}{*}{\rotatebox[origin=c]{90}{\parbox[c]{3.2cm}{\large ALIE (40\%)}}}
    & \textsc{Adam} & \textbf{1.000} & 0.001 & 0.031 & 0.069 \\
    & \textsc{FLTrust} & 0.220 & 0.032 & 0.085 & 0.150 \\
    & \textsc{Recess} & 0.220 & 0.055 & 0.110 & 0.065 \\
    & \textsc{Zeno} & 0.220 & 0.025 & 0.075 & 0.124 \\
    & \textsc{CC} & 0.520 & 0.443 & 0.480 & 0.115 \\
    & \textsc{CC+fbm} & 0.690 & 0.626 & 0.657 & \textit{0.350} \\
    & \textsc{CC+bucketing} & 0.700 & 0.585 & 0.640 & \textbf{0.380} \\
    & \textsc{Safeguard} & \textbf{1.000} & 0.000 & 0.000 & 0.069 \\
    & \textsc{VR Marina} & 0.200 & 0.001 & 0.010 & 0.103 \\
    & \texttt{Bant} & 0.935 & 0.850 &\textbf{ 0.892} & 0.314 \\
    & \texttt{AutoBant} & 0.818 & \textbf{0.888} & 0.852 & 0.339 \\
    & \texttt{SimBant} & \textit{0.968} & 0.818 & \textit{0.890} & 0.282 \\
    \bottomrule
    \end{tabular}   
    }
    \end{threeparttable}
\end{minipage}
\end{table}

\begin{table}[H]
\begin{minipage}{0.4\columnwidth}
    \centering
    \captionof{table}{\textsc{ResNet1d18} on ECG (PVC).}
\label{tab:ecgpvc} 
  \begin{threeparttable}
  \resizebox{\columnwidth}{!}{
    \begin{tabular}{l|l|c|c|c|c}
    \toprule
    \multicolumn{1}{c|}{} & Algorithm & Sensitivity & Specificity & G-mean & f1-score
    \\

\midrule
     
    \multirow{6}{*}{\rotatebox[origin=c]{90}{\parbox[c]{3.5cm}{\large Without Attack}}} 
    & \textsc{Adam} & \textbf{0.977} & 0.974 & \textit{0.975} & 0.790 \\
    & \textsc{FLTrust} & \textit{0.972} & 0.971 & 0.972 & 0.772 \\
    & \textsc{Recess} & \textit{0.972} & 0.961 & 0.967 & 0.720 \\
    & \textsc{Zeno} & \textbf{0.977} & 0.975 & \textbf{0.976} & \textit{0.801} \\
    & \textsc{CC} & \textit{0.972} & 0.959 & 0.965 & 0.707 \\
    & \textsc{CC+fbm} & 0.960 & 0.980 & 0.970 & 0.755 \\
    & \textsc{CC+bucketing} & 0.950 & \textbf{0.986} & 0.968 & 0.740 \\
    & \textsc{Safeguard} & \textbf{0.977} & 0.965 & 0.971 & 0.743 \\
    & \textsc{VR Marina} & 0.240 & 0.060 & 0.120 & 0.098 \\
    & \texttt{Bant} & 0.931 & \textit{0.981} & 0.955 & \textbf{0.810} \\
    & \texttt{AutoBant} & \textit{0.972} & 0.970 & 0.971 & 0.765 \\
    & \textsc{FineTuned} & 0.963 & 0.971 & 0.967 & 0.770 \\
    \midrule
    \multirow{6}{*}{\rotatebox[origin=c]{90}{\parbox[c]{3.75cm}{\large Label Flip (60\%)}}}
    & \textsc{Adam} & 0.639 & 0.710 & 0.673 & 0.180 \\
    & \textsc{FLTrust} & 0.220 & 0.059 & 0.114 & 0.096 \\
    & \textsc{Recess} & 0.931 & 0.937 & 0.934 & 0.596 \\
    & \textsc{Zeno} & 0.310 & 0.037 & 0.108 & 0.210 \\
    & \textsc{CC} & 0.005 & \textbf{0.999} & 0.068 & 0.009 \\
    & \textsc{CC+fbm} & 0.930 & 0.920 & 0.925 & 0.760 \\
    & \textsc{CC+bucketing} & 0.910 & 0.910 & 0.910 & 0.735 \\
    & \textsc{Safeguard} & 0.981 & 0.971 & \textbf{0.976} & \textit{0.782} \\
    & \textsc{VR Marina} & 0.200 & 0.036 & 0.085 & 0.115 \\
    & \texttt{Bant} & 0.870 & \textit{0.979} & 0.923 & 0.767 \\
    & \texttt{AutoBant} & \textbf{0.972} & 0.976 & \textit{0.974} &\textbf{ 0.805} \\
    & \textsc{FineTuned} & \textit{0.944} & 0.957 & 0.951 & 0.686 \\
    \midrule
    \multirow{6}{*}{\rotatebox[origin=c]{90}{\parbox[c]{3.7cm}{\large Sign Flip (60\%)}}}
    & \textsc{Adam} & 0.639 & 0.710 & 0.673 & 0.180 \\
    & \textsc{FLTrust} & 0.220 & 0.032 & 0.085 & 0.118 \\
    & \textsc{Recess} & 0.931 & 0.937 & 0.934 & 0.596 \\
    & \textsc{Zeno} & 0.190 & 0.069 & 0.115 & 0.010 \\
    & \textsc{CC} & 0.005 & \textbf{0.999} & 0.068 & 0.009 \\
    & \textsc{CC+fbm} & 0.250 & 0.072 & 0.135 & 0.117 \\
    & \textsc{CC+bucketing} & 0.230 & 0.055 & 0.113 & 0.096 \\
    & \textsc{Safeguard} & \textbf{0.981} & 0.971 & \textbf{0.976} & \textit{0.782} \\
    & \textsc{VR Marina} & 0.120 & 0.035 & 0.065 & 0.010 \\
    & \texttt{Bant} & 0.870 & \textit{0.979} & 0.923 & 0.767 \\
    & \texttt{AutoBant} & \textit{0.972} & 0.976 & \textit{0.974} & \textbf{0.805} \\
    & \textsc{FineTuned} & 0.944 & 0.957 & 0.951 & 0.686 \\
    \midrule
    \multirow{6}{*}{\rotatebox[origin=c]{90}{\parbox[c]{4.1cm}{Random Gradients (60\%)}}}
    & \textsc{Adam} & 0.220 & 0.005 & 0.035 & 0.120 \\
    & \textsc{FLTrust} & 0.220 & 0.063 & 0.118 & 0.110 \\
    & \textsc{Recess} & 0.255 & 0.845 & 0.464 & 0.122 \\
    & \textsc{Zeno} & 0.180 & 0.053 & 0.098 & 0.010 \\
    & \textsc{CC} & 0.250 & 0.048 & 0.110 & 0.130 \\
    & \textsc{CC+fbm} & 0.270 & 0.083 & 0.150 & 0.117 \\
    & \textsc{CC+bucketing} & 0.220 & 0.043 & 0.098 & 0.123 \\
    & \textsc{Safeguard} & 0.009 & \textbf{0.997} & 0.096 & 0.017 \\
    & \textsc{VR Marina} & 0.120 & 0.037 & 0.067 & 0.010 \\
    & \texttt{Bant} & \textit{0.944} & 0.969 & \textit{0.957} & \textit{0.747} \\
    & \texttt{AutoBant} & \textbf{0.963} & \textit{0.978} & \textbf{0.970} & \textbf{0.809} \\
    & \textsc{FineTuned} & \textbf{0.963} & 0.945 & 0.954 & 0.644 \\
    \midrule
    \multirow{6}{*}{\rotatebox[origin=c]{90}{\parbox[c]{3.35cm}{\large IPM (60 \%)}}}
    & \textsc{Adam} & 0.796 & \textit{0.976} & 0.882 & 0.708 \\
    & \textsc{FLTrust} & 0.240 & 0.046 & 0.106 & 0.114 \\
    & \textsc{Recess} & 0.944 & 0.938 & 0.941 & 0.608 \\
    & \textsc{Zeno} & 0.220 & 0.025 & 0.075 & 0.124 \\
    & \textsc{CC} & 0.944 & 0.933 & 0.939 & 0.590 \\
    & \textsc{CC+fbm} & 0.930 & 0.960 & 0.945 & 0.625 \\
    & \textsc{CC+bucketing} & 0.930 & 0.950 & 0.940 & 0.610 \\
    & \textsc{Safeguard} & 0.782 & 0.217 & 0.412 & 0.095 \\
    & \textsc{VR Marina} & 0.190 & 0.037 & 0.084 & 0.115 \\
    & \texttt{Bant} & 0.931 & \textbf{0.977} & 0.954 & \textbf{0.790} \\
    & \texttt{AutoBant} & \textit{0.949} & 0.971 & \textbf{0.960 }& \textit{0.762} \\
    & \textsc{FineTuned} & \textbf{0.954} & 0.965 & \textit{0.959} & 0.729 \\
    \midrule
    \multirow{6}{*}{\rotatebox[origin=c]{90}{\parbox[c]{3.35cm}{\large IPM (80 \%)}}}
    & \textsc{Adam} & \textit{0.958} & 0.005 & 0.069 & 0.093 \\
    & \textsc{FLTrust} & 0.005 & \textbf{0.991} & 0.068 & 0.008 \\
    & \textsc{Recess} & 0.426 & 0.502 & 0.462 & 0.079 \\
    & \textsc{Zeno} & 0.300 & 0.225 & 0.260 & 0.108\\
    & \textsc{CC} & 0.463 & 0.265 & 0.350 & 0.061 \\
    & \textsc{CC+fbm} & 0.220 & 0.010 & 0.047 & 0.114 \\
    & \textsc{CC+bucketing} & 0.200 & 0.006 & 0.035 & 0.105 \\
    & \textsc{Safeguard} & 0.065 & 0.831 & 0.232 & 0.031 \\
    & \textsc{VR Marina} & 0.210 & 0.063 & 0.115 & 0.086 \\
    & \texttt{Bant} & 0.931 & \textit{0.980} & \textit{0.955} & \textit{0.807} \\
    & \texttt{AutoBant} & \textbf{0.970} & 0.965 & \textbf{0.968} & \textbf{0.820} \\
    & \textsc{FineTuned} & \textit{0.963} & 0.937 & \textit{0.950} & 0.611 \\
    \midrule
    \multirow{6}{*}{\rotatebox[origin=c]{90}{\parbox[c]{3.2cm}{\large ALIE (40\%)}}}
    & \textsc{Adam} & 0.200 & 0.021 & 0.065 & 0.120 \\
    & \textsc{FLTrust} & 0.398 & 0.581 & 0.481 & 0.086 \\
    & \textsc{Recess} & 0.200 & 0.028 & 0.075 & 0.135 \\
    & \textsc{Zeno} & 0.190 & 0.063 & 0.110 & 0.078  \\
    & \textsc{CC} & 0.913 & 0.116 & 0.325 & 0.084 \\
    & \textsc{CC+fbm} & 0.880 & 0.821 & 0.850 & 0.580 \\
    & \textsc{CC+bucketing} & 0.890 & 0.870 & 0.880 & 0.610 \\
    & \textsc{Safeguard} & 0.218 & 0.881 & 0.438 & 0.126 \\
    & \textsc{VR Marina} & 0.240 & 0.064 & 0.124 & 0.105 \\
    & \texttt{Bant} & \textbf{0.954} & \textit{0.955} & \textit{0.954} & \textit{0.680} \\
    & \texttt{AutoBant} & \textbf{0.954} & \textbf{0.981} & \textbf{0.967} & \textbf{0.826} \\
    & \textsc{FineTuned} & \textit{0.917} & 0.954 & 0.935 & 0.658 \\
    \bottomrule
    \end{tabular}   
    }
    \end{threeparttable}
\end{minipage}
\hfill
\begin{minipage}{0.4\textwidth}
    \centering
    \captionof{table}{\textsc{ResNet1d18} on ECG (CLBBB).}
\label{tab:ecgclbbb} 
  \begin{threeparttable}
  \resizebox{\textwidth}{!}{
    \begin{tabular}{l|l|c|c|c|c}
    \toprule
    \multicolumn{1}{c|}{} & Algorithm & Sensitivity & Specificity & G-mean & f1-score
    \\

\midrule
     
    \multirow{6}{*}{\rotatebox[origin=c]{90}{\parbox[c]{3.5cm}{\large Without Attack}}} 
    & \textsc{Adam} & 0.979 & 0.957 & 0.968 & 0.508 \\
    & \textsc{FLTrust} & \textbf{0.990 }& 0.947 & 0.968 & 0.462 \\
    & \textsc{Recess} & 0.969 & \textbf{0.963} & 0.966 & \textbf{0.538} \\
    & \textsc{Zeno} & \textbf{0.990} & 0.957 & \textit{0.973} & 0.509 \\
    & \textsc{CC} & 0.979 & 0.945 & 0.962 & 0.445 \\
    & \textsc{CC+fbm} & 0.970 & 0.970 & 0.970 & 0.480 \\
    & \textsc{CC+bucketing} & 0.950 & 0.950 & 0.950 & 0.465 \\
    & \textsc{Safeguard} & \textbf{0.990} & \textit{0.961} & \textbf{0.975} & \textit{0.537} \\
    & \textsc{VR Marina} & 0.220 & 0.021 & 0.068 & 0.117 \\
    & \texttt{Bant} & \textbf{0.990} & 0.947 & 0.968 & 0.460 \\
    & \texttt{AutoBant} & \textit{0.989} & 0.936 & 0.962 & 0.415 \\
    & \textsc{FineTuned} & \textbf{0.990} & 0.953 & 0.971 & 0.491 \\
    \midrule
    \multirow{6}{*}{\rotatebox[origin=c]{90}{\parbox[c]{3.75cm}{\large Label Flip (60\%)}}}
    & \textsc{Adam} & 0.220 & 0.065 & 0.120 & 0.115 \\
    & \textsc{FLTrust} & 0.198 & 0.630 & 0.353 & 0.023 \\
    & \textsc{Recess} & 0.021 & 0.873 & 0.135 & 0.006 \\
    & \textsc{Zeno} & 0.190 & 0.084 & 0.127 & 0.005 \\
    & \textsc{CC} & 0.764 & 0.854 & 0.808 & 0.199 \\
    & \textsc{CC+fbm} & 0.860 & 0.800 & 0.830 & 0.210 \\
    & \textsc{CC+bucketing} & 0.790 & 0.830 & 0.810 & 0.214 \\
    & \textsc{Safeguard} & \textit{0.989} & 0.958 & \textbf{0.974} & \textbf{0.517} \\
    & \textsc{VR Marina} & 0.260 & 0.042 & 0.105 & 0.125 \\
    & \texttt{Bant} & \textbf{0.990} & 0.923 & 0.956 & 0.370 \\
    & \texttt{AutoBant} & \textbf{0.990} & \textbf{0.955} & \textit{0.972} & \textit{0.502} \\
    & \textsc{FineTuned} & \textbf{0.990} & \textit{0.952} & 0.971 & 0.487 \\
    \midrule
    \multirow{6}{*}{\rotatebox[origin=c]{90}{\parbox[c]{3.7cm}{\large Sign Flip (60\%)}}}
    & \textsc{Adam} & \textbf{1.000} & 0.001 & 0.027 & 0.044 \\
    & \textsc{FLTrust} & 0.865 & 0.747 & 0.804 & 0.134 \\
    & \textsc{Recess} & 0.220 & 0.102 & 0.150 & 0.044 \\
    & \textsc{Zeno} & 0.250 & 0.193 & 0.220 & 0.054  \\
    & \textsc{CC} & 0.260 & 0.177 & 0.215 & 0.117 \\
    & \textsc{CC+fbm} & 0.320 & 0.204 & 0.256 & 0.117 \\
    & \textsc{CC+bucketing} & 0.300 & 0.052 & 0.125 & 0.240 \\
    & \textsc{Safeguard} & 0.698 & 0.617 & 0.656 & 0.076 \\
    & \textsc{VR Marina} & 0.210 & 0.029 & 0.078 & 0.112 \\
    & \texttt{Bant} & \textit{0.979} & \textit{0.944} & \textit{0.962} & \textit{0.444} \\
    & \texttt{AutoBant} & 0.969 & 0.934 & 0.951 & 0.401 \\
    & \textsc{FineTuned} & 0.969 & \textbf{0.958} & \textbf{0.964} & \textbf{0.512} \\
    \midrule
    \multirow{6}{*}{\rotatebox[origin=c]{90}{\parbox[c]{4.1cm}{Random Gradients (60\%)}}}
    & \textsc{Adam} &\textbf{1.000} & 0.007 & 0.085 & 0.044 \\
    & \textsc{FLTrust} & \textbf{1.000} & 0.003 & 0.054 & 0.044 \\
    & \textsc{Recess} & 0.427 & 0.963 & 0.641 & 0.282 \\
    & \textsc{Zeno} & 0.310 & 0.054 & 0.130 & 0.210 \\
    & \textsc{CC} & 0.130 & 0.120 & 0.125 & 0.044 \\
    & \textsc{CC+fbm} & 0.310 & 0.201 & 0.250 & 0.130 \\
    & \textsc{CC+bucketing} & 0.190 & 0.170 & 0.180 & 0.054 \\
    & \textsc{Safeguard} & 0.052 & \textbf{0.987} & 0.227 & 0.064 \\
    & \textsc{VR Marina} & 0.210 & 0.030 & 0.080 & 0.112 \\
    & \texttt{Bant} & \textit{0.990} & 0.948 & \textit{0.969} & \textit{0.465} \\
    & \texttt{AutoBant} & \textit{0.990} & 0.931 & 0.960 & 0.396 \\
    & \textsc{FineTuned} & \textit{0.990}& \textit{0.971} & \textbf{0.980} & \textbf{0.607} \\
    \midrule
    \multirow{6}{*}{\rotatebox[origin=c]{90}{\parbox[c]{3.35cm}{\large IPM (60 \%)}}}
    & \textsc{Adam} & \textbf{0.979} & 0.957 & \textbf{0.968} & \textit{0.511} \\
    & \textsc{FLTrust} & \textbf{0.979 }& 0.296 & 0.538 & 0.060 \\
    & \textsc{Recess} & \textit{0.958} & 0.945 & 0.951 & 0.438 \\
    & \textsc{Zeno} & 0.330 & 0.027 & 0.095 & 0.215 \\
    & \textsc{CC} & \textit{0.958} & 0.950 & 0.954 & 0.465 \\
    & \textsc{CC+fbm} & 0.950 & 0.940 & 0.945 & 0.440 \\
    & \textsc{CC+bucketing} & 0.940 & 0.936 & 0.938 & 0.460 \\
    & \textsc{Safeguard} & \textbf{0.979} & 0.539 & 0.727 & 0.089 \\
    & \textsc{VR Marina} & 0.230 & 0.041 & 0.098 & 0.115 \\
    & \texttt{Bant} & 0.813 & \textbf{0.982} & 0.893 & \textbf{0.629} \\
    & \texttt{AutoBant} & \textbf{0.979} & 0.941 & 0.960 & 0.429 \\
    & \textsc{FineTuned} & 0.938 & 0.930 & 0.934 & 0.376 \\
    \midrule
    \multirow{6}{*}{\rotatebox[origin=c]{90}{\parbox[c]{3.35cm}{\large IPM (80 \%)}}}
    & \textsc{Adam} & 0.073 & 0.506 & 0.192 & 0.006 \\
    & \textsc{FLTrust} & 0.330 & 0.017 & 0.075 & 0.210 \\
    & \textsc{Recess} & 0.042 & \textbf{0.952} & 0.199 & 0.026 \\
    & \textsc{Zeno} & 0.230 & 0.141 & 0.180 & 0.048\\
    & \textsc{CC} & 0.073 & 0.884 & 0.254 & 0.024 \\
    & \textsc{CC+fbm} & 0.190 & 0.075 & 0.120 & 0.065 \\
    & \textsc{CC+bucketing} & 0.220 & 0.200 & 0.210 & 0.020 \\
    & \textsc{Safeguard} & \textbf{1.000} & 0.002 & 0.046 & 0.044 \\
    & \textsc{VR Marina} & 0.220 & 0.005 & 0.035 & 0.118 \\
    & \texttt{Bant} & \textit{0.990} & \textit{0.951} & \textbf{0.970} & \textbf{0.481} \\
    & \texttt{AutoBant} & 0.979 & 0.947 & 0.963 & \textit{0.459} \\
    & \textsc{FineTuned} & \textit{0.990} & 0.943 & \textit{0.966} & 0.441 \\
    \midrule
    \multirow{6}{*}{\rotatebox[origin=c]{90}{\parbox[c]{3.2cm}{\large ALIE (40\%)}}}
    & \textsc{Adam} & 0.250 & 0.176 & 0.210 & 0.123 \\
    & \textsc{FLTrust} & 0.220 & 0.089 & 0.140 & 0.065 \\
    & \textsc{Recess} & \textbf{1.000} & 0.005 & 0.071 & 0.044 \\
    & \textsc{Zeno} & 0.220 & 0.032 & 0.085 & 0.114 \\
    & \textsc{CC} & 0.460 & 0.599 & 0.525 & 0.160 \\
    & \textsc{CC+fbm} & 0.490 & 0.470 & 0.480 & 0.210 \\
    & \textsc{CC+bucketing} & 0.540 & 0.463 & 0.500 & 0.120 \\
    & \textsc{Safeguard} & 0.220 & 0.060 & 0.115 & 0.087 \\
    & \textsc{VR Marina} & 0.200 & 0.078 & 0.125 & 0.072 \\
    & \texttt{Bant} & \textit{0.979}& \textbf{0.968} & \textbf{0.973} & \textbf{0.578} \\
    & \texttt{AutoBant} & \textbf{1.000} & 0.814 & 0.902 & 0.198 \\
    & \textsc{FineTuned} & 0.958 & \textit{0.957} & \textit{0.958} & \textit{0.501} \\
    \bottomrule
    \end{tabular}   
    }
    \end{threeparttable}
\end{minipage}
\end{table}

\newpage
\subsection{Learning-to-rank Experiments} \label{sec:appendix_experiments_ltr}
In the main body, we introduced our approach and presented a subset of results for the Learning-to-Rank (LTR) task under Byzantine settings. Here, we expand on the experimental setup, covering all baseline models and proposed methods across the full spectrum of adversarial scenarios considered.

\paragraph{Problem Formulation.} The LTR task is defined over a set of queries $\mathcal{Q}$, where each query $q \in \mathcal{Q}$ is associated with a set of documents $D_q$. For each document $d_i \in D_q$, a feature vector $x_i$ and a relevance label $y_i \in \{0, \dots, r-1\}$ are provided. The objective is to learn a scoring function $f(x; \theta)$ such that, for any query $q$, the induced ordering of scores $s_i = f(x_i; \theta)$ approximates the ideal relevance ordering.

We use the Normalized Discounted Cumulative Gain at cutoff $k$ (NDCG@k) to assess ranking quality. First, we define the Discounted Cumulative Gain (DCG@k) for a ranking $\pi$ (a permutation of documents based on predicted scores) as:
$$
\mathrm{DCG}@k = \sum_{i=1}^{k} \frac{2^{y_{\pi(i)}} - 1}{\log_2(i + 1)},
$$
where $y_{\pi(i)}$ is the relevance label of the document ranked at position $i$. This formulation assigns higher weight to highly relevant documents that appear earlier in the ranking, with a logarithmic discount applied to lower positions.

\begin{figure}[h]
    \centering
    \includegraphics[width=\textwidth]{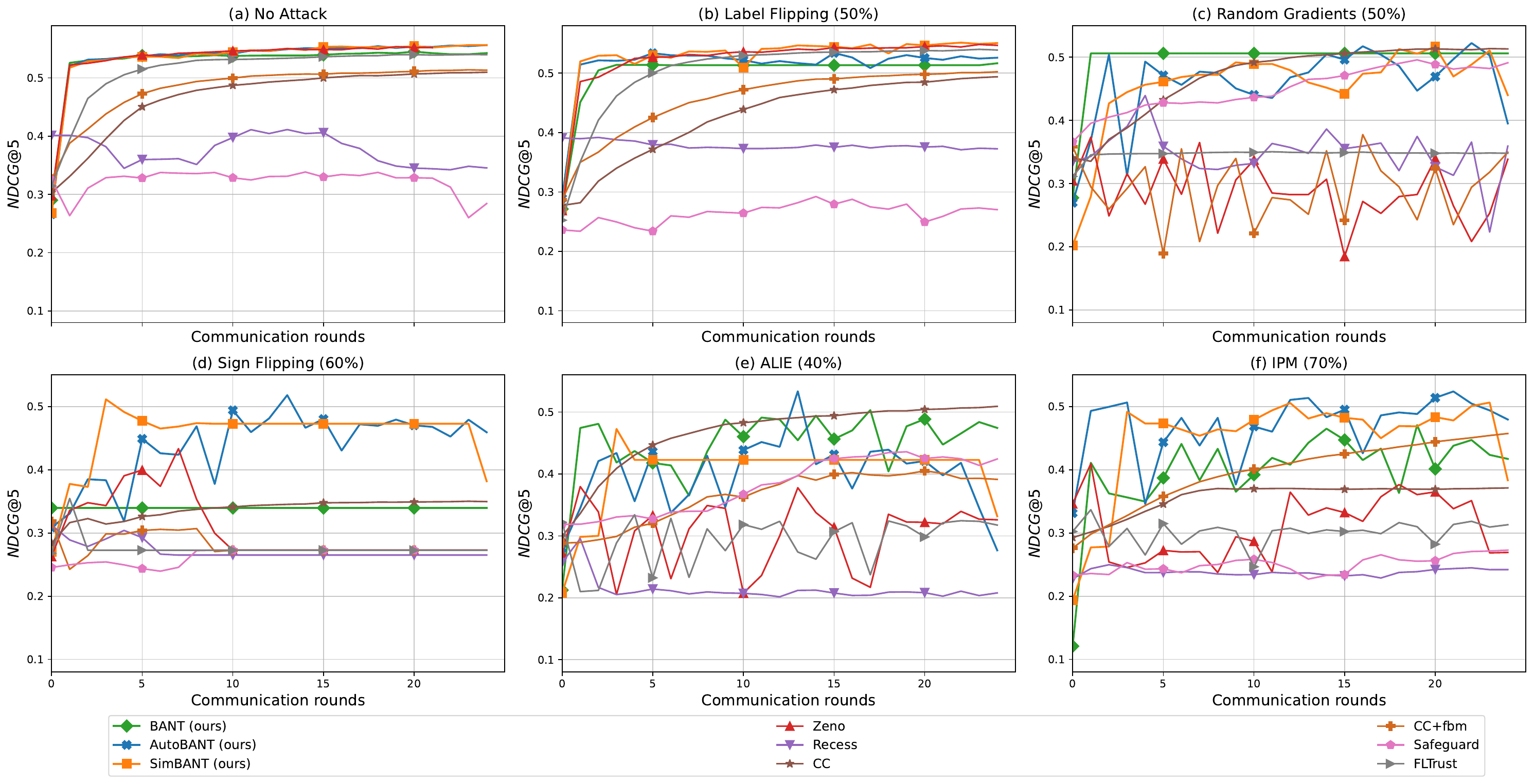}
    \caption{Test NDCG@5 for \textsc{Transformer} on the Learning-to-Rank task.}
    \label{fig:ltr_appendix}
\end{figure}

The Ideal DCG (IDCG@k) is the maximum possible DCG@k obtained by sorting documents in descending order of their true relevance labels.

The final evaluation metric, Normalized DCG (NDCG@k), is computed as:

\[
\mathrm{NDCG}@k = \frac{\mathrm{DCG}@k}{\mathrm{IDCG}@k}.
\]

This normalization bounds the metric between 0 and 1, where 1 indicates a perfect ranking.

\paragraph{Experimental Scope.} We evaluate models on the WEB30K dataset \citep{DBLP:journals/corr/QinL13}, a standard benchmark for LTR consisting of 30,000 queries with graded relevance labels. While the main section highlights our method's applicability to LTR, here we provide complete validation across:

\begin{itemize}
    \item Baseline defenses, including \textsc{Zeno}, \textsc{Recess}, \textsc{Centered Clip}, \textsc{Safeguard} and \textsc{FlTrust}.
    \item Our proposed methods: \texttt{Bant}, \texttt{AutoBant}, and \texttt{SimBant}.
    \item All adversarial scenarios introduced in the main paper: Label Flipping, Sign Flipping, Random Gradients, IPM, and ALIE.
\end{itemize}

For consistency with \textsc{Zeno}'s assumptions, we set the threshold for Byzantine tolerance to $|\mathcal{B}| = 0.5n$, where $n$ is the number of clients as it is unknown a priori. While \textsc{Zeno} might perform well when the fraction of adversarial clients is close to $50\%$, it suffers significant degradation (or divergence) when this assumption is violated. This sensitivity underscores the need for defenses that do not rely on tight prior knowledge of Byzantine ratios.

\allowdisplaybreaks
\section{Notation}\label{Notation}
The following sections will be dedicated to the theoretical proof of all aspects discussed in the main body of the paper. 

In order to facilitate the understanding of the proofs presented in the appendix, as well as to simplify the interaction with all the formulas throughout this paper, we provide a comprehensive list of notation used in this study in the form of the following table.
\begin{table}[H]
\caption{Notation Reference.} \label{table:notation}
\centering
\resizebox{\textwidth}{!}{
\begin{tabular}{|p{0.8cm}|p{3.2cm}|p{0.5cm}|p{3.5cm}|p{1.1cm}|p{3cm}|p{0.9cm}|p{3cm}|}
\hline
$N$ & Number of samples in the trial function & $\gamma$ & Learning rate (stepsize) in optimization & $\hat{f}$ & Trial loss function on the server & $\mathbb{I}_{[a > 0]}$ & Indicator function taking value 1 if $a > 0$, otherwise 0 \\
\hline
$d$ & Dimensionality of the parameter space $\mathbb{R}^d$ & $L$ & Smoothness constant \eqref{as1}& $f$ & Objective function $f(x)$ in distributed learning & $\hat{P}^t$ & Adaptive preconditioner matrix at iteration $t$ \\
\hline
$n = n(t)$ & Number of workers in the distributed system & $\mu$ & Strong convexity constant \eqref{as2stronglyconvex} & $f_i$ & Local objective function on $i$-th worker &  $\langle a, b \rangle_{\hat{P}}$ & Weighted inner product with $\hat{P}$ \\
\hline
$\mathcal{G}(t)$ & Set of honest workers & $\beta$ & Momentum parameter for weights $\omega_i^t$ & $f_1$ & Local objective function on the server& $\alpha$ & Lower bound on the preconditioner \eqref{asscaled}\\
\hline
$G(t) = |\mathcal{G}(t)|$ & Number of honest workers & $\delta$ & Approximation error in $\arg\min$ finding \eqref{alg1} & $g_i$ & Stochastic gradient of worker $i$& $\Gamma$ & Upper bound on the preconditioner \eqref{asscaled} \\
\hline
$\mathcal{B}(t)$ & Set of Byzantine workers & $\delta_1$ & $(\delta_1, \delta_2)-$heterogeneity parameter \eqref{as4} &$g_i^t$ & Gradient from worker $i$ at iteration $t$ & $\|\cdot\|_{\infty}$ & $\max\limits_{1 \leqslant i \leqslant d} |\cdot|_i$\\
\hline
$B(t) = |\mathcal{B}(t)|$ & Number of Byzantine workers & $\delta_2$ & $(\delta_1, \delta_2)-$heterogeneity parameter \eqref{as4} &  $x^*$ & Optimal solution of the objective function  & $\Delta_d^1$ & $d$-dimensional simplex constraint on weights \\
\hline
$t$ & Current iteration number & $\sigma^2$ & Variance of stochastic gradients \eqref{as3} &$\overline{x}^t$ & $\frac{1}{G}\sum\limits_{i \in \mathcal{G}}x_i^t$& $\omega_i^t$ & Weight assigned to worker $i$ at iteration $t$ \\
\hline
$T$ & Total number of iterations in training & $l$ & Local round length &$\mathcal{KL}(p \| q)$ & Kullback-Leibler divergence between distributions $p$ and $q$ & $[\cdot]_0$ & Non-negative projection: $\max\{0, \cdot\}$ 
\\
\hline
$\mathcal{X}$ & Domain$:~x\in \mathcal{X}$ &$\mathcal{F}$&function class $\mathcal{F}: \quad \xi \mapsto \nabla f(x; \xi), x \in \mathcal{X}$& $V^t$ & $\frac{1}{G}\sum\limits_{i \in \mathcal{G}}\|x_i^t-\overline{x}^t\|^2$  & $\varepsilon$ & Radius of balls in the covering net in Lemma \ref{lemmaShalevShwartz} \\
\hline
$\mathcal{W} = \mathcal{W}(t)$ & Number of workers & $G $ &$\underset{t\leqslant T}{\min}~ G(t)$ & $S$ & Bound on $\mathcal{X}$ in Lemma \ref{lemmaShalevShwartz} & $\|\cdot\|$ & If not specified other, $\|\cdot\|_2 = \sqrt{\langle \cdot, \cdot \rangle}$\\
\hline
\end{tabular}
}

\end{table}

\newpage
\section{General Inequalities and Lemmas}\label{B}

First, mention important inequalities that are used in further proofs. Consider a function \( f \) satisfying \text{Assumption~\ref{as1}}, \( g \) satisfying \text{Assumptions \ref{as2stronglyconvex}} and \( \varphi \) complying with \text{Assumption~\ref{as2convex}}. Then for any \( i \) in the real numbers and for all vectors \( x, y, x_i \) in \( \mathbb{R}^n \) with a positive scalar \( p \), the following inequalities hold.
\begin{align}
\notag\\ \hline \notag\\
\label{bi:young} \tag{Young} \left|\left\langle x, y \right\rangle\right| \quad & \leqslant \quad \frac{\|x\|^2}{2p} + \frac{p \|y\|^2}{2} \\
\notag\\ \hline \notag\\
\label{bi:norm} \tag{Norm}
\begin{split}
 -\left\langle x, y \right\rangle \quad &= \quad -\frac{\|x\|^2}{2} - \frac{\|y\|^2}{2} + \frac{\|x-y\|^2}{2} \\
 \|x + y\|^2 \quad &= \quad \|x\|^2 + \|y\|^2 + 2\left\langle x, y \right\rangle
\end{split}\\
\notag\\ \hline \notag\\
\label{bi:mustrconvex} \tag{$\mu$-Conv} \quad g(y) \quad & \geqslant \quad g(x) - \left\langle \nabla g(y), x - y \right\rangle - \frac{1}{2\mu} \| \nabla g(x) - \nabla g(y) \|^2 \\
\notag\\ \hline \notag\\
\label{bi:convex} \tag{Conv}
\begin{split}
\quad \varphi(y) \quad & \geqslant \quad \varphi(x) + \left\langle \nabla \varphi(x), y - x \right\rangle \\
0\quad& \geqslant\quad \left\langle \nabla \varphi(x) - \nabla \varphi(y), y - x\right\rangle 
\end{split}\\
\notag\\ \hline \notag\\
\label{bi:lsmoothness} \tag{Lip} 
\begin{split}
    \left\|\nabla f(x) - \nabla f(y)\right\|^2 \quad & \leqslant \quad L^2 \|x-y\|^2\\
    \quad f(x) \quad & \leqslant \quad f(y) + \left\langle \nabla f(y), x-y \right\rangle + \frac{L}{2} \|x-y\|^2\\
    f(x) \quad & \leqslant \quad f(y) - \left\langle \nabla f(x), y-x \right\rangle - \frac{1}{2L} \|\nabla f(x)-\nabla f(y)\|^2
\end{split}\\
\notag\\ \hline \notag\\
\label{bi:CauchySchwarz} \tag{CS}
\begin{split}
\left\|\sum_{i=1}^{n} x_i\right\|^2 \quad& \leqslant  \quad n \sum_{i=1}^{n} \|x_i\|^2  \\
\left\|x + y \right\|^2 \quad& \leqslant  \quad (1 + p)\|x\|^2 + \left(1 + \frac{1}{p}\right)\|y\|^2
\end{split}
\\
\notag\\ \hline \notag\\
\label{bi:Jensen} \tag{Jen}  \varphi \left(\frac{\sum_{i=1}^{n} w_i x_i}{\sum_{i=1}^{n} w_i}\right) \quad & \leqslant \quad \frac{\sum_{i=1}^{n} w_i \varphi(x_i)}{\sum_{i=1}^{n} w_i} \\
\notag\\ \hline \notag
\end{align}
This section delineates a series of lemmas that form the cornerstone of our subsequent proofs. These lemmas encapsulate critical properties and bounds that are instrumental in establishing the theorems elaborated later in this paper.

\newpage
The following lemma addresses a critical issue concerning the evaluation of the trial function and its deviation from $\nabla f_1$. As highlighted in the main part of our work, existing literature frequently overlooks the additional term in convergence that arises when employing a trial function. In this study, we rectify this oversight, thereby providing a more comprehensive understanding of the convergence behavior associated with trial functions.
\begin{lemma} \label{lemmaShalevShwartz}
Suppose Assumption \ref{as1} holds. Then for all $x \in \mathcal{X} \subset \mathbb{R}^d$ with probability of at least $1-\widetilde{\delta}$
over a sample of size $N$, the following estimate, linking the trial function with the objective function on the server, is valid:
\begin{equation*}
\begin{split}
\|\nabla f_1(x) - \nabla \hat{f}(x)\|_{2}^2 ~ \leqslant ~ \zeta(N) = \widetilde{\mathcal{O}}\left(\frac{1}{N}\right),
\end{split}
\end{equation*}

\begin{proof}
Given the norm inequality $\|\cdot\|_2 \leqslant \sqrt{d}\cdot \|\cdot\|_\infty$, we can recast the scalar product in the following manner:
\begin{eqnarray}
\label{lemmaShalevShwartz:1}
\|\nabla f_1(x) - \nabla \hat{f}(x)\|_{2}^2 &\leqslant& d \cdot \|\nabla f_1(x^t) - \nabla \hat{f}(x)\|_{\infty}^2.
\end{eqnarray}
To establish the uniform convergence of $\left\|\nabla f_1(x) - \nabla \hat{f}(x)\right\|_\infty^2$, we employ Theorem 5 from \citep{shalev2009stochastic}. This theorem provides a bound on the \(\ell_{\infty}\)-covering number of the function class \(\mathcal{F} = \{\xi \mapsto \nabla f_1(x; \xi) \mid x \in \mathcal{X}\}\). Given that \(\mathcal{X}\) resides within an \(\ell_2\)-sphere, let us define it bound by $S$, the covering number for \(\mathcal{X}\) using the Euclidean metric \(d_2(x_i, x_j) = \|x_i - x_j\|_2\) is constrained as follows for \(d > 3\):
$$
N(\varepsilon, \mathcal{X}, d_2) = \mathcal{O}\left(d^2 \left(\frac{S}{\varepsilon}\right)^d\right).
$$
In evaluating the covering numbers for \(\mathcal{F}\) under the \(\ell_\infty\) metric, where \(\left\|\nabla f_1(x_i; \cdot) - \nabla f_1(x_j; \cdot)\right\|_\infty = \sup_\xi |\nabla f_1(x_i; \xi) - \nabla f_1(x_j; \xi)|\), the \(L\)-smoothness property facilitates the following assertion:
$$
\forall x_i, x_j \in \mathcal{X} \hookrightarrow \left\|\nabla f_1(x_i; \cdot) - \nabla f_1(x_j; \cdot)\right\|_\infty \leqslant \left\|\nabla f_1(x_i; \cdot) - \nabla f_1(x_j; \cdot)\right\|_2 \leqslant L \|x_i - x_j\|.
$$
This indicates that an \(\varepsilon\)-net for \(\mathcal{X}\) in \(d_2\) space concurrently serves as an \(L\varepsilon\)-net for \(\mathcal{F}\) in \(d_{\infty}\) space:
$$
N(\varepsilon, \mathcal{F}, d_{\infty}) \leqslant N(\varepsilon/L, \mathcal{X}, d_2) = \mathcal{O}\left(d^2 \left(\frac{LS}{\varepsilon}\right)^d\right).
$$
Following this analysis, we derive an estimation consistent with the findings in \citep{shalev2009stochastic}:
$$
\|\nabla f_1(x) - \nabla \hat{f}(x)\|_{\infty}^2 = \widetilde{\mathcal{O}}\left(\frac{1}{N}\right).
$$
Defining the notation $$\zeta(N) \overset{\text{def}}{=} \widetilde{\mathcal{O}}\left(\frac{1}{N}\right),$$ and substituting this into \eqref{lemmaShalevShwartz:1} concludes the proof of the lemma.

\end{proof}
\end{lemma}

This lemma is technical in nature, and we significantly benefit from the assertion established in the previous lemma. Ultimately, we derive an important estimate for the scalar product, which appears in many subsequent proofs throughout this work.
\begin{lemma} \label{lemma2}
Suppose Assumption \ref{as1} holds. Then for all $x \in \mathbb{R}^d$ and $g_i = g_i(x, \xi_i)$, the following estimate is valid:
\begin{eqnarray*}
- \gamma \left\langle \nabla\hat{f}(x), \frac{1}{ G } \sum_{i \in \mathcal{G}} \nabla f_i(x) \right\rangle  & \leqslant & - \frac{\gamma}{2}\left\|\nabla f(x) \right\|^2 + \gamma\cdot\zeta(N) + \frac{3\gamma}{2} \left(\delta_1 + \delta_2 \|\nabla f(x^t)\|^2\right).
\end{eqnarray*}

\begin{proof}
We commence by examining the difference $\nabla f(x) - \nabla \hat{f}(x)$:
\begin{eqnarray*}
- \gamma \left\langle \nabla\hat{f}(x), \frac{1}{ G } \sum_{i \in \mathcal{G}} \nabla f_i(x) \right\rangle & = &\gamma \left\langle \nabla f(x) - \nabla \hat{f}(x),\frac{1}{ G } \sum_{i \in \mathcal{G}} \nabla f_i(x) \right\rangle \\
& & - \gamma \left\langle \nabla f(x),\frac{1}{ G } \sum_{i \in \mathcal{G}} \nabla f_i(x)\right\rangle.
\end{eqnarray*}
Next, we continue with further manipulations on the first term:
\begin{eqnarray*} 
& & \gamma\left\langle \nabla f(x) - \nabla \hat{f}(x), \frac{1}{ G } \sum_{i \in \mathcal{G}} \nabla f_i(x) \right\rangle \\
& &~~~~\overset{\eqref{bi:young}}{\leqslant} ~~~\frac{\gamma}{2} \left\|\nabla f(x) - \nabla \hat{f}(x)\right\|^2 + \frac{\gamma}{2}\left\|\frac{1}{ G } \sum_{i \in \mathcal{G}} \nabla f_1(x)\right\|^2\\
& &~~~~\overset{\eqref{bi:young}}{\leqslant} ~~~\frac{\gamma}{2}\left(\left\|\nabla f(x) - \nabla f_1(x)\right\|^2 + \left\|\nabla f_1(x) - \nabla \hat{f}(x)\right\|^2\right) + \frac{\gamma}{2}\left\|\frac{1}{ G } \sum_{i \in \mathcal{G}} \nabla f_i(x)\right\|^2\\
& &\overset{(\text{Lemma }\ref{lemmaShalevShwartz})}{\leqslant} \gamma \left( \zeta(N) + \delta_1 + \delta_2 \|\nabla f(x)\|^2 \right) + \frac{\gamma}{2}\left\|\frac{1}{ G } \sum_{i \in \mathcal{G}} \nabla f_i(x)\right\|^2, 
\end{eqnarray*} 
and with the second term, \begin{eqnarray*}
- \gamma\left\langle \nabla f(x), \frac{1}{ G }\sum\limits_{i\in \mathcal{G}} \nabla f_i(x)\right\rangle & \overset{\eqref{bi:norm}}{=}& -\frac{\gamma}{2}\|\nabla f(x)\|^2 -\frac{\gamma}{2} \left\|\frac{1}{ G }\sum\limits_{i \in \mathcal{G}} \nabla f_i(x)\right\|^2 \\
&& + \frac{\gamma}{2} \left\|\frac{1}{ G }\sum\limits_{i \in \mathcal{G}} (\nabla f_i(x) - \nabla f(x))\right\|^2 \\
&\overset{\eqref{bi:CauchySchwarz}}{\leqslant}& - \frac{\gamma}{2}\|\nabla f(x)\|^2 -\frac{\gamma}{2} \left\|\frac{1}{ G }\sum\limits_{i \in \mathcal{G}} \nabla f_i(x)\right\|^2 \\
&&+ \frac{\gamma}{2 G }\sum\limits_{i \in \mathcal{G}} \left\| (\nabla f_i(x) - \nabla f(x))\right\|^2\\
&\overset{(\text{Ass. }\ref{as4})}{\leqslant}& - \frac{\gamma}{2}\|\nabla f(x)\|^2 -\frac{\gamma}{2} \left\|\frac{1}{ G }\sum\limits_{i \in \mathcal{G}} \nabla f_i(x)\right\|^2 \\
&&+ \frac{\gamma}{2} \left(\delta_1 + \delta_2 \|\nabla f(x)\|^2\right).
\end{eqnarray*}
In summary, this supports the claim of the lemma.
\end{proof}
\end{lemma}
With this, we conclude the discussion of general statements. They are frequently used in our proofs in the upcoming sections of the Appendix. Next, we begin the examination of each method obtained individually.

\newpage 
\section{Proofs of Bant}\label{C}
In this section, we explore the theoretical underpinnings of the first proposed method, \texttt{Bant}. As outlined in Algorithm \ref{alg2}, this step diverges from the standard \textsc{SGD} approach primarily due to the distinct weight distribution that we subsequently allocate to the devices. Consequently, it is essential to conduct an analysis that takes this particular characteristic into account. To achieve final convergence rate, we demonstrate a supporting lemma that reinforces our findings. 
\begin{lemma}\label{lemma:scetch}  Under Assumptions \ref{as1}, \ref{as2convex}, \ref{as5}, the following holds for the iteration of \text{Algorithm \ref{alg2}}:
\begin{eqnarray*}
    \hat{f}(x^{t+1}) & \leqslant &  \hat{f}(x^t) - \frac{\gamma\beta}{n}\left\langle\nabla\hat{f}(x^t), \sum\limits_{i\in \mathcal{G}} g_i^t\right\rangle + \frac{L\gamma^2 \beta}{2n}\sum\limits_{i\in \mathcal{G}}\|g_i^t\|^2.
\end{eqnarray*}
\begin{proof}
 Actually, the update step of the algorithm \ref{alg2} is given by:
    \begin{equation*}
        x^{t+1} = x^t - \gamma\sum\limits_{i=1}^n \mathbb{I}_{[\theta_i^t> 0]}\omega_i^t g_i^t,
    \end{equation*}
where $g_i^t = g_i(x^t, \xi_i^t)$ and $\sum_{i=1}^n \omega_i^t = 1$. Applying Jensen's inequality for the convex function $\hat{f}$ (Assumption \ref{as2convex}) and denoting $\overline{\omega}_i^t = \frac{[\theta_i^t]_0}{\sum_{j=1}^n [\theta_j^t]_0}$:
\begin{eqnarray*}
    \hat{f}(x^{t+1}) & = & \hat{f}\Bigl(\sum\limits_{i=1}^n \omega_i^t\left[x^t -  \gamma\mathbb{I}_{[\theta_i^t> 0]} g_i^t\right]\Bigr) \\
    & \leqslant & \sum\limits_{i=1}^n \omega_i^t \hat{f}\left(x^t - \gamma\mathbb{I}_{[\theta_i^t> 0]}g_i^t\right) \\
    & = & \sum\limits_{i\in \mathcal{B}} \omega_i^t \hat{f}\left(x^t - \gamma\mathbb{I}_{[\theta_i^t> 0]}g_i^t\right) + \sum\limits_{i\in \mathcal{G}} \omega_i^t \hat{f}\left(x^t - \gamma\mathbb{I}_{[\theta_i^t> 0]}g_i^t\right) \\
    & \leqslant & \sum\limits_{i\in \mathcal{B}} (1 - \beta) \omega_i^{t-1} \hat{f}(x^t) + \sum\limits_{i\in \mathcal{B}} \beta\overline{\omega}_i^t \hat{f}\left(x^t - \gamma\mathbb{I}_{[\theta_i^t> 0]}g_i^t\right) \\ 
    & &+ \sum\limits_{i\in \mathcal{G}} (1 - \beta) \omega_i^{t-1} \hat{f}\left(x^t\right) + \sum\limits_{i\in \mathcal{G}} \beta\overline{\omega}_i^t \hat{f}\left(x^t - \gamma\mathbb{I}_{[\theta_i^t> 0]}g_i^t\right) \\ 
    & = & (1 - \beta)\hat{f}(x^t) + \sum\limits_{i\in \mathcal{B}} \beta\overline{\omega}_i^t \hat{f}\left(x^t - \gamma\mathbb{I}_{[\theta_i^t> 0]}g_i^t\right) + \sum\limits_{i\in \mathcal{G}} \beta\overline{\omega}_i^t \hat{f}\left(x^t - \gamma\mathbb{I}_{[\theta_i^t> 0]}g_i^t\right).
\end{eqnarray*}
 
In the inequality above, we make an estimation $\hat{f}\left(x^t - \gamma\mathbb{I}_{[\theta_i^t> 0]}g_i^t\right) \leqslant \hat{f}\left(x^t\right)$, since the indicator guarantees us that we do not increase the trial function $\hat{f}$ by performing a step. By eliminating the weights $\omega_i^{t-1}$ accumulated from past iterations, we can rearrange the coefficients between Byzantine and honest workers in such a way that honest workers have higher weights. 
To achieve this, we sort the honest workers by increasing values of $\hat{f}$ and assign them coefficients $\omega_i$ in decreasing order. This permutation ensures that honest workers have higher weights and Byzantine workers have lower weights. This operation is valid because if $\overline{\omega}$ for some Byzantine worker is higher than for a honest worker, then this Byzantine has a greater influence on $\hat f$, and changing the weights would worsen the overall influence of these two workers. Therefore, with new weights $\{\widetilde{\omega}_i^t\}_{i=1}^n$:
 \begin{eqnarray*}
    \hat{f}(x^{t+1}) & \leqslant & (1 - \beta)\hat{f}(x^t) + \sum\limits_{i\in \mathcal{B}} \beta\widetilde{\omega}_i^t \hat{f}\left(x^t - \gamma\mathbb{I}_{[\theta_i^t> 0]}g_i^t\right) \\
    &&+ \sum\limits_{i\in \mathcal{G}} \beta\widetilde{\omega}_i^t \hat{f}\left(x^t - \gamma\mathbb{I}_{[\theta_i^t> 0]}g_i^t\right) \\
    & \leqslant & (1 - \beta)\hat{f}(x^t) + \sum\limits_{i\in \mathcal{B}} \beta\widetilde{\omega}_i^t \hat{f}(x^t) + \sum\limits_{i\in \mathcal{G}} \beta\widetilde{\omega}_i^t \hat{f}\left(x^t - \gamma\mathbb{I}_{[\theta_i^t> 0]}g_i^t\right) \\ 
    & = & \hat{f}(x^t) + (1 - \beta)\left[\hat{f}(x^t) - \hat{f}(x^t)\right] + \sum\limits_{i\in \mathcal{B}} \beta\widetilde{\omega}_i^t \left[\hat{f}(x^t) - \hat{f}(x^t)\right] \\
    &&+ \sum\limits_{i\in \mathcal{G}} \beta\widetilde{\omega}_i^t \left[\hat{f}\left(x^t - \gamma\mathbb{I}_{[\theta_i^t> 0]}g_i^t\right) - \hat{f}(x^t)\right].
\end{eqnarray*}
Let us assign the coefficient $\nicefrac{1}{n}$ to all honest workers. This procedure is also valid. We sorted the weights and honest workers now have the greatest weights, thus, the sum of the coefficients of honest workers is at least $\nicefrac{ G }{n}$. Moreover, the honest workers with the stronger influence have the greater weights which allows to equalize the total weight $\nicefrac{ G }{n}$ between all $ G $ workers.
Thus, we get
    \begin{align*}
    \textstyle{
        \hat{f}(x^{t+1}) \leqslant \hat{f}(x^t) + \frac{\beta}{n}\sum\limits_{i\in \mathcal{G}} \left[\hat{f}\left(x^t - \gamma\mathbb{I}_{[\theta_i^t > 0]}g_i^t\right) - \hat{f}(x^t)\right].}
    \end{align*}
Now we can remove the indicator function because if $g_i^t$ minimizes the trial function, the indicator equals 1. If $g_i^t$ maximizes the trial function, the indicator excludes this gradient. However, we still account for it and maximize the trial function, thus:
    \begin{eqnarray*}
    \hat{f}(x^{t+1}) & \leqslant & \hat{f}(x^{t}) + \frac{\beta}{n}\sum\limits_{i\in \mathcal{G}} \left[\hat{f}\left(x^t - \gamma g_i^t\right) - \hat{f}(x^t)\right]  \\
    & \overset{\eqref{bi:lsmoothness}}{\leqslant} & \hat{f}(x^{t}) + \frac{\beta}{n} \sum\limits_{i\in \mathcal{G}} \Biggl[\hat{f}(x^t) - \gamma\left\langle\nabla\hat{f}(x^t), g_i^t\right\rangle + \frac{L\gamma^2}{2}\|g_i^t\|^2 - \hat{f}(x^t)\Biggr] \\
    & = & \hat{f}(x^t) - \frac{\gamma\beta}{n}\left\langle\nabla\hat{f}(x^t), \sum\limits_{i\in \mathcal{G}} g_i^t\right\rangle + \frac{L\gamma^2 \beta}{2n}\sum\limits_{i\in \mathcal{G}}\|g_i^t\|^2.
\end{eqnarray*}
\end{proof}
\end{lemma}

We are now prepared to present the final result for the convex case. This theorem was introduced in the main part of our work, however, we will reiterate its formulation once more.

\textbf{Theorem~\ref{ITRSGDa}.} \textit{Under Assumptions \ref{as1}, \ref{as2convex}, \ref{as3},
     \ref{as4} with $\delta_2 \leqslant \frac{1}{12}$, \ref{as5}, for solving the problem described in the equation \eqref{eq:setting_dist} after $T$ iterations of \text{Algorithm \ref{alg2}} with $\gamma \leqslant \frac{1}{13L}$, the following holds}:
\begin{eqnarray*}
    \frac{1}{T}\sum\limits_{t = 0}^{T-1} \mathbb E \|\nabla f(x^t)\|^2 &\leqslant &\frac{\mathbb E\left[\hat{f}(x^0) - \hat{f}(\hat{x}^*)\right]}{\gamma T}\cdot\frac{4 n}{\beta G } + 3\delta_1 + 6L \gamma\sigma^2 + 4\zeta(N).
\end{eqnarray*}
\begin{proof}
According to the lemma \eqref{lemma:scetch}:
\begin{eqnarray*}
\hat{f}(x^{t+1}) &\leqslant & \hat{f}(x^t) - \gamma\beta\left\langle\nabla\hat{f}(x^t), \frac{1}{n}\sum\limits_{i\in \mathcal{G}} \nabla g_i^t\right\rangle + \frac{L\gamma^2 \beta}{2n}\sum\limits_{i\in \mathcal{G}} \left\|g_i^t\right\|^2.
\end{eqnarray*}
Taking the expectation of both sides of the inequality:
\begin{eqnarray*}
\mathbb{E}\hat{f}(x^{t+1}) &\leqslant & \mathbb{E}\hat{f}(x^t) - \gamma\beta \cdot\frac{ G }{n}\left\langle\nabla\hat{f}(x^t), \frac{1}{ G }\sum\limits_{i\in \mathcal{G}} \nabla f_i(x^t)\right\rangle + \frac{L\gamma^2 \beta}{2n}\sum\limits_{i\in \mathcal{G}}\mathbb{E}\left\| g_i^t\right\|^2 \\
    &\overset{(\text{Lemma } \ref{lemma2})}{\leqslant} & \mathbb{E}\hat{f}(x^t) + \frac{\gamma\beta G }{n}\zeta(N) - \frac{\gamma\beta G }{2n}\left\|\nabla f(x^t) \right\|^2 \\
&&+ \frac{3\gamma\beta G }{2n} \left(\delta_1 + \delta_2 \|\nabla f(x^t)\|^2\right) +  \frac{L\gamma^2 \beta}{2n}\sum\limits_{i\in \mathcal{G}}\mathbb{E}\left\| g_i^t\right\|^2 \\
    & \overset{\eqref{bi:CauchySchwarz}}{\leqslant}& \mathbb{E}\hat{f}(x^t) + \frac{\gamma\beta G }{n}\zeta(N) - \frac{\gamma\beta G }{2n}\left\|\nabla f(x^t) \right\|^2 \\
    & & + \frac{3\gamma\beta G }{2n} \left(\delta_1 + \delta_2 \|\nabla f(x^t)\|^2\right) \\
& &+  \frac{3L \gamma ^2\beta}{2n}\left(\sum_{i \in \mathcal{G}} \mathbb{E}\left\|\nabla f(x^t) - f_i(x^t)\right\|^2 + \sum_{i \in \mathcal{G}} \mathbb{E}\left\|\nabla f_i(x^t) - g_i^t\right\|^2\right) \\
& &+ \frac{3L \gamma ^2\beta G }{2n}\left\|\nabla f(x^t) \right\|^2 \\
& \overset{(\text{Ass. }\ref{as3}, \ref{as4})}{\leqslant}& \mathbb{E}\hat{f}(x^t) + \frac{\gamma\beta G }{n}\zeta(N) - \frac{\gamma\beta G }{2n}\left\|\nabla f(x^t) \right\|^2 + \frac{3\gamma\beta G }{2n} \left(\delta_1 + \delta_2 \|\nabla f(x^t)\|^2\right) \\
&&+  \frac{3L \gamma ^2\beta G }{2n}\left(\delta_1 + \delta_2 \|\nabla f(x^t)\|^2 +  \sigma ^2\right) + \frac{3L \gamma ^2\beta G }{2n}\left\|\nabla f(x^t) \right\|^2 \\
    & \overset{(\mathrm{Ass. \ref{as3}})}{\leqslant}& \mathbb E [\hat{f}(x^{t})] - \frac{\gamma\beta G }{2n}\left[1 - 3L \gamma - (3 + 3L \gamma)\delta_2 \right]\|\nabla f(x^t)\|^2 \\
&& +\frac{\gamma\beta  G }{2n}(3 + 3L \gamma)\delta_1 +  \frac{3\gamma^2\beta L G }{2n}\sigma^2 + \frac{\gamma\beta  G }{n}\zeta(N).
\end{eqnarray*}
We first fix $\delta_2 \leqslant \frac{1}{12}$. Then by choosing $\gamma  \leqslant \frac{1}{13L} \leqslant \frac{1}{12L(1 + \delta_2)}$ and summing over the iterations, we get the bound:
\begin{eqnarray*}
    \frac{1}{T}\sum\limits_{t = 0}^{T-1} \mathbb E \|\nabla f(x^t)\|^2 &\leqslant& \frac{\mathbb E\left[\hat{f}(x^0) - \hat{f}(\hat{x}^*)\right]}{\gamma T}\cdot\frac{4 n}{\beta G } + 3\delta_1 + 6L \gamma\sigma^2 + 4\zeta(N).
\end{eqnarray*}
\end{proof}
We have obtained the final statement of the theorem. From this, we can derive the convergence rate:

\textbf{Corollary \ref{ITRSGDCorollary}} 
\textit{Under the assumptions of Theorem \ref{ITRSGDa}, for solving the problem \eqref{eq:setting_dist}, after $T$ iterations with  $\gamma \leqslant \min\left\{\frac{1}{13L}, \frac{\sqrt{2\mathbb E\left[\hat{f}(x^0) - \hat{f}(\hat{x}^*)\right] n}}{\sigma\sqrt{3L  G  \beta T}}\right\}$, the following holds:}
        \begin{align*}
              \frac{1}{T}\sum\limits_{i=1}^{T-1}\mathbb{E}\|\nabla f(x^t)\|^2 =\mathcal{O}\Biggl(\frac{\mathbb E\left[\hat{f}(x^0) - \hat{f}(\hat{x}^*)\right] Ln}{\beta  G T} 
            + \frac{\sigma\sqrt{\mathbb E\left[\hat{f}(x^0) - \hat{f}(\hat{x}^*)\right] L n}}{\sqrt{  \beta  G  T}} 
            + \delta_1 
            + \zeta(N) \Biggr).
        \end{align*}

\begin{proof}[Proof of Corollary~\ref{ITRSGDCorollary}]
    We proceed estimation, analogical to \text{Lemma 4 from \citep{stich2019unified}}. Using the result of \text{Theorem \ref{ITRSGDa}}, we choose the appropriate $\gamma \leqslant \min\left\{\frac{1}{13L}, \frac{\sqrt{2\mathbb E\left[\hat{f}(x^0) - \hat{f}(\hat{x}^*)\right] n}}{\sigma\sqrt{3L  G  \beta T}}\right\}$. In that way, we obtain
    \begin{align*}
        \frac{1}{T}\sum\limits_{i=1}^{T-1}\mathbb{E}\|\nabla f(x^t)\|^2  \leqslant & \frac{
        \mathbb E\left[\hat{f}(x^0) - \hat{f}(\hat{x}^*)\right] 52 Ln}{\beta  G T} + \frac{2\sigma\sqrt{6\mathbb E\left[\hat{f}(x^0) - \hat{f}(\hat{x}^*)\right] L n}}{\sqrt{\beta G T}} + 3 \delta_1 + 2 \zeta(N),
    \end{align*}
that ends the proof.
\end{proof}

We have obtained the result for the convex case. However, we wish to extend the theory to other cases.
Using the obtained result, let us proceed to the \(\mu\)-strongly convex case and derive an estimate for it.

\begin{theorem}
    \label{ITRSGDb}
    Under Assumptions \ref{as1}, \ref{as2stronglyconvex}, \ref{as3},
     \ref{as4} with $\delta_2 \leqslant \frac{1}{12}$, \ref{as5}, for solving the problem described in the equation \eqref{eq:setting_dist} after $T$ iterations of \text{Algorithm \ref{alg2}} with $\gamma \leqslant \frac{1}{13L}$, the following holds:
        \begin{eqnarray*}
            \mathbb{E} \left[\hat{f}(x^t) - \hat{f}(\hat{x}^*)\right]  & \leqslant &  \left(1 -  \frac{\gamma\beta G \mu}{4n}\right)^{t} \mathbb E\left[\hat{f}(x^0) - \hat{f}(\hat{x}^*)\right] + \frac{3}{\mu}\delta_1 + \frac{6L\gamma}{\mu}\sigma^2 + \frac{2}{\mu}\zeta(N).
        \end{eqnarray*}
\begin{proof}

From Theorem \ref{ITRSGDa}, we have
\begin{eqnarray*}
    \mathbb E \|\nabla f(x^t)\|^2 &\leqslant&  \frac{4 n}{\gamma \beta G }\mathbb{E}\left[ \hat{f}(x^{t}) - \hat{f}(\hat{x}^{t+1})\right]  + 3\delta_1 + 6L \gamma\sigma^2 + 4\zeta(N).
\end{eqnarray*}
Let us examine the left-hand side of the inequality in more detail:
\begin{eqnarray*}
\mathbb{E}\left\|\nabla f(x^t)\right\|^2  &=& \mathbb{E}\left\|\nabla f(x^t)\right\|^2 + \mathbb{E}\left\|\nabla \hat{f} (x^t) - \nabla f(x^t)\right\|^2 - \mathbb{E}\left\|\nabla \hat{f} (x^t) - \nabla f(x^t)\right\|^2 \\
 & \overset{\eqref{bi:CauchySchwarz}}{\geqslant} & \frac{1}{2} \mathbb{E}\left\|\nabla f(x^t) + \nabla \hat{f} (x^t) - \nabla f(x^t)\right\|^2 - \mathbb{E}\left\|\nabla \hat{f} (x^t) - \nabla f(x^t)\right\|^2 \\
 & \overset{(\text{Lemma }\ref{lemmaShalevShwartz})}{\geqslant} & \frac{1}{2} \mathbb{E}\left\|\nabla \hat{f}(x^t)\right\|^2 -2\zeta(N) -2\delta_1 - 2\delta_2 \|\nabla f(x^t)\|^2.
\end{eqnarray*}
Returning to the initial inequality,
\begin{eqnarray*}
\left(\frac{1}{2}-2\delta_2\right)\mathbb{E}\left\|\nabla \hat{f} (x^t) - \nabla \hat{f}(\hat{x}^*)\right\|^2 & \leqslant & \frac{4 n}{\gamma \beta G }\mathbb{E}\left[ \hat{f}(x^{0}) - \hat{f}(\hat{x}^*)\right] + 5\delta_1 + 6L \gamma\sigma^2 + 6\zeta(N) 
\end{eqnarray*}
Taking into account that $\delta_2 \leqslant \frac{1}{12}$ and due to \eqref{bi:mustrconvex}, we get:
\begin{eqnarray*}
\underbrace{\frac{1}{2}\mathbb{E}\left\|\nabla \hat{f} (x^t) - \nabla \hat{f}(\hat{x}^*)\right\|^2}_{\geqslant ~ \mu \left[ \hat{f}(x^{t}) - \hat{f}(\hat{x}^*)\right]} & \leqslant & \frac{6 n}{\gamma \beta G }\mathbb{E}\left[ \hat{f}(x^{0}) - \hat{f}(\hat{x}^*)\right] + 8\delta_1 + 9L \gamma\sigma^2 + 9\zeta(N) \\
\frac{6 n}{\gamma \beta G } \mathbb{E}\hat{f}(x^{t+ 1}) - \mu \mathbb{E} \hat{f}(\hat{x}^{*}) & \leqslant& \left(\frac{6 n}{\gamma \beta G } - \mu\right) \mathbb{E} \hat{f}(x^{t})+ 8\delta_1 + 9L \gamma\sigma^2 + 9\zeta(N).
\end{eqnarray*}
Defining $\gamma' = \frac{\gamma \beta G }{6 n},$
\begin{eqnarray*}
\mathbb{E}\hat{f}(x^{t+ 1}) -  \gamma' \mu \mathbb{E} \hat{f}(\hat{x}^{*}) & \leqslant & \left(1 - \gamma' \mu \right) \mathbb{E} \hat{f}(x^{t})+ 8\gamma'  \delta_1 + 9\gamma' L \gamma\sigma^2 + 9\gamma' \zeta(N).
\end{eqnarray*}
We add to both sides of inequality the term $(-1 + \gamma' \mu) \mathbb E \hat{f}(\hat{x}^*)$:
\begin{eqnarray*}
    \left[\mathbb{E}\hat{f}(x^{t+ 1}) - \mathbb{E} \hat{f}(\hat{x}^{*})\right]& \leqslant  &\left(1 -  \gamma' \mu\right) \left[\mathbb{E} \hat{f}(x^{t}) - \mathbb{E} \hat{f}(\hat{x}^*)\right] + 8\gamma'\delta_1 +  \frac{9L \gamma\gamma'}{ G } \sigma ^2  +  9\gamma' \zeta(N).
\end{eqnarray*}
Applying this inequality to the first term on the right side \(t\) times, we obtain:
\begin{eqnarray*}
\mathbb{E}\left[\hat{f}(x^t) - \hat{f}(\hat{x}^*)\right]  & \leqslant &\left(1 -  \gamma' \mu\right)^{t} \mathbb{E}\left[\hat{f}(x^{0}) - \hat{f}(\hat{x}^*)\right] \\
& & + \underbrace{\sum\limits_{i=0}^{t-1} \left(1 - \gamma' \mu\right)^i}_{\leqslant \frac{1}{\gamma' \mu}} \left[ 8\gamma'\delta_1 + 9L \gamma\gamma' \sigma ^2  + 9\gamma'\zeta(N)\right],\\
\mathbb{E}\left[\hat{f}(x^t) - \hat{f}(\hat{x}^*)\right] & \leqslant& \left(1 -  \gamma' \mu\right)^{t} \mathbb{E}\left[\hat{f}(x^{0}) - \hat{f}(\hat{x}^*)\right] +  \frac{8}{\mu} \delta_1 + \frac{9L\gamma}{\mu}\sigma^2 + \frac{9}{\mu}\zeta(N).
\end{eqnarray*}
\end{proof}
\end{theorem}

Similarly, from this estimate, we derive the final convergence rate for the $\mu$-strongly convex setting.
\begin{corollary} \label{ITRSGDCorollaryb}
    Under the assumptions of \text{Theorem} \ref{ITRSGDb}, for solving the problem \eqref{eq:setting_dist}, after $T$ iterations with special tunings of $\gamma$:
        \begin{align*}
        \mathbb E\left[\hat{f}(x^T) - \hat{f}(\hat{x}^*)\right] &=\mathcal{\widetilde{O}}\bigg( \mathbb E\left[\hat{f}(x^0) - \hat{f}(\hat{x}^*)\right]\exp\left[-\frac{\mu\beta G  T}{4Ln}\right] + \frac{Ln}{\mu^2\beta G T}\sigma^2 + \frac{1}{\mu}\delta_1 + \frac{1}{\mu}\zeta(N)\bigg).
    \end{align*}

\begin{proof}
    In \text{Theorem \ref{ITRSGDb}}, we obtain classical result for SGD. We use \text{Lemma 2 from \citep{stich2019unified}} and appropriate special tunings of $\gamma$:
    \begin{eqnarray*}
        \gamma &\leqslant &\min\left\{\frac{1}{13L}, \frac{4n\log\left(\max\left\{2, \frac{\mu^2 \beta  G \mathbb E\left[\hat{f}(x^0) - \hat{f}(\hat{x}^*)\right]}{36 Ln\sigma_*^2}\right\}\right)}{\mu\beta G  T}\right\}.
    \end{eqnarray*}    
We obtain the final convergence.
\end{proof}
\end{corollary}
With this, we conclude this section. In summary, we examine the proof for both the convex and strongly convex cases and obtain the final convergence estimates.
\newpage
\section{Proofs of AutoBant}\label{D}
Now, let us turn our attention to the second of our methods. As previously mentioned, \texttt{Bant} has certain moments to be discussed. It is important to note the adverse impact of the parameter $\beta$. The introduction of this momentum term aimed to protect honest clients from rapidly decreasing their trust scores due to unfavorable stochastic gradients, but it inadvertently enables Byzantine agents to maintain their weights despite their attacks.
To address this issue, we add an indicator to the algorithm for detecting Byzantine devices, although this limits its theoretical part in common non-convex scenarios.

Our goal is to learn how to circumvent this limitation. To achieve this, we tackle an additional subproblem related to weight assignment. This step represents a key distinction in our theoretical analysis. We assert that we can solve this weight distribution minimization subproblem to any predefined accuracy. Next, we present the theorem discussed in the main part of the paper, along with its complete proof.

\textbf{Theorem~\ref{TFMSGD}.} 
    \textit{Under Assumptions \ref{as1}, \ref{as2nonconvex}, \ref{as3},
     \ref{as4} with $\delta_2 \leqslant \frac{1}{12}$, \ref{as5}, for solving the problem described in the equation \eqref{eq:setting_dist} after $T$ iterations of \text{Algorithm \ref{alg1}} with $\gamma \leqslant \frac{1}{13L}$, the following holds:}
    \begin{eqnarray*}
    \frac{1}{T}\sum\limits_{t = 0}^{T-1} \mathbb E \|\nabla f(x^t)\|^2&\leqslant &\frac{4\mathbb E\left[\hat{f}(x^0) - \hat{f}(\hat{x}^*)\right]}{\gamma T} + 3\delta_1 + \frac{6L\gamma}{ G }\sigma^2 + 4\zeta(N) 
    .
\end{eqnarray*}
\begin{proof}
The iterative update formula for $x^{t+1}$ is given by
\begin{eqnarray*} 
x^{t+1} &=& x^t - \gamma \sum_{i = 1}^n \left(\arg\min_{\omega \in \Delta_1^n} \hat{f}\left[x^t - \gamma \sum_{i = 1}^n \omega_i g_i^t\right]\right) g_i^t,
\end{eqnarray*} 
which leads to an upper bound on $\hat{f}(x^{t+1})$:
\begin{eqnarray*} 
\hat{f}(x^{t+1}) & = & \min_{\omega \in \Delta_1^n} \hat{f}\left[x^t - \gamma \sum_{i = 1}^n \omega_i g_i^t\right]  \\
& \leqslant & \hat{f}\left[x^t - \frac{\gamma}{ G } \sum_{i \in \mathcal{G}} g_i^t\right] \\
& \overset{\eqref{bi:lsmoothness}}{\leqslant} &\hat{f}(x^t) - \left\langle \nabla\hat{f}(x^t), \frac{\gamma}{ G } \sum_{i \in \mathcal{G}} g_i^t \right\rangle + \frac{L \gamma^2}{2} \left\|\frac{1}{ G }\sum_{i \in \mathcal{G}} g_i^t\right\|^2 
\end{eqnarray*}
Taking the expectation,
\begin{eqnarray*}
\mathbb{E} \hat{f}(x^{t+1})&\leqslant &\mathbb{E}  \hat{f}(x^t) - \left\langle \nabla\hat{f}(x^t), \frac{\gamma}{ G } \sum_{i \in \mathcal{G}} \nabla f_i(x^t) \right\rangle  + \frac{L \gamma^2}{2} \mathbb{E} \left\|\frac{1}{ G }\sum_{i \in \mathcal{G}} g_i^t\right\|^2 \\ 
    &\overset{(\text{Lemma } \ref{lemma2})}{\leqslant} & \mathbb{E}\hat{f}(x^t) + \gamma\zeta(N) - \frac{\gamma}{2}\left\|\nabla f(x^t) \right\|^2 \\
&&+ \frac{3\gamma}{2} \left(\delta_1 + \delta_2 \|\nabla f(x^t)\|^2\right) +  \frac{L\gamma^2}{2}\mathbb{E}\left\| \frac{1}{ G }\sum_{i \in \mathcal{G}} g_i^t\right\|^2\\
    &\overset{\eqref{bi:CauchySchwarz}}{\leqslant} & \mathbb{E}\hat{f}(x^t) + \gamma\zeta(N) - \frac{\gamma}{2}\left\|\nabla f(x^t)\right\|^2 + \frac{3\gamma}{2} \left(\delta_1 + \delta_2 \|\nabla f(x^t)\|^2\right)\\
&&+  \frac{3L \gamma ^2}{2}\left\|\frac{1}{ G }\sum_{i \in \mathcal{G}}  (\nabla f(x^t) - \nabla f_i (x^t))\right\|^2 \\
&&+ \frac{3L \gamma ^2}{2}\mathbb{E}\left\|\frac{1}{ G }\sum_{i \in \mathcal{G}}  (\nabla f_i (x^t) - g_i^t)\right\|^2 +\frac{3L \gamma ^2}{2} \left\|\frac{1}{ G }\sum_{i \in \mathcal{G}} \nabla f(x^t) \right\|^2.
\end{eqnarray*}
Due to the fact that $\mathbb{E} g_i^t = \nabla f_i(x^t)$ and $\mathbb{E} \langle \nabla f_i (x^t) - g_i^t, \nabla f_j (x^t) - g_j^t \rangle = 0$,
\begin{eqnarray*}
     \mathbb{E} \hat{f}(x^{t+1})   & \overset{\eqref{bi:CauchySchwarz}}{\leqslant}& \mathbb{E}\hat{f}(x^t) + \gamma\zeta(N) - \frac{\gamma}{2}\left\|\nabla f(x^t) \right\|^2 + \frac{3\gamma}{2} \left(\delta_1 + \delta_2 \|\nabla f(x^t)\|^2\right) \\
&&+  \frac{3L \gamma ^2}{2} \left(\frac{1}{ G }\sum_{i \in \mathcal{G}} \left\|\nabla f(x^t) - f_i(x^t)\right\|^2 + \frac{1}{ G ^2}\sum_{i \in \mathcal{G}} \mathbb{E}\left\|\nabla f_i(x^t) - g_i^t\right\|^2\right) \\
& & + \frac{3L \gamma ^2}{2} \left\|\nabla f(x^t) \right\|^2  \\
& \overset{(\text{Ass. }\ref{as3}, \ref{as4})}{\leqslant}& \mathbb{E}\hat{f}(x^t) + \gamma\zeta(N) - \frac{\gamma}{2}\left\|\nabla f(x^t) \right\|^2 + \frac{3\gamma}{2} \left(\delta_1 + \delta_2 \|\nabla f(x^t)\|^2\right) \\
&&+  \frac{3L \gamma ^2}{2}\left(\delta_1 + \delta_2 \|\nabla f(x^t)\|^2 +  \frac{\sigma ^2}{ G }\right) + \frac{3L \gamma ^2}{2} \left\|\nabla f(x^t) \right\|^2 \\
    &=& \mathbb E [\hat{f}(x^{t})] - \frac{\gamma}{2}\left[1 - 3 L\gamma - (3 + 3 L\gamma)\delta_2 \right]\|\nabla f(x^t)\|^2 \\
&& +\frac{2\gamma}{2}(1 + 3 L\gamma)\delta_1 +  \frac{3L\gamma^2}{2 G }\sigma^2 + \gamma\zeta(N).
\end{eqnarray*}
We first fix $\delta_2 \leqslant \frac{1}{12}$. By choosing $\gamma  \leqslant \frac{1}{13L} \leqslant \frac{1}{12L(1 + \delta_2)}$, and summing over the iterations, we get
\begin{eqnarray*}
    \frac{1}{T}\sum\limits_{t = 0}^{T-1} \mathbb E \|\nabla f(x^t)\|^2 &\leqslant& \frac{4\mathbb E\left[\hat{f}(x^0) - \hat{f}(\hat{x}^*)\right]}{\gamma T} + 3\delta_1 + \frac{6L\gamma}{ G }\sigma^2 + 4\zeta(N).
\end{eqnarray*}
\end{proof}
We have successfully proven the obtained result. Let us also recall that we formulated the final estimate in Corollary \ref{TFMSGDCorollary}. We will omit the proof of Corollary \ref{TFMSGDCorollary} since it entirely replicates the proof of Corollary \ref{ITRSGDCorollary}. 

With this, we conclude our proof of the foundational versions of the algorithms. We establish all the formulated statements and derive convergence estimates for the strongly convex, convex, and non-convex cases. The subsequent sections of the Appendix are dedicated to exploring extensions that hold significant importance in our study.

\newpage
\section{Scaled methods}\label{E}

In this section, we provide a detailed analysis of our Byzantine-robust methods extended to adaptive methods, as mentioned in the main part. Specifically, we consider the application of our techniques to methods like \textsc{Adam} and \textsc{RMSProp}. We describe the formal description of \textsc{Scaled Bant} and \textsc{Scaled AutoBant} methods, which utilize a diagonal preconditioner \((\hat{P}^t)^{-1}\), which scales a gradient to \((\hat{P}^t)^{-1} g_i^t\), and the step is performed using this scaled gradient.
From iteration to iteration, the matrix $P^t$ changes, e.g. the following rule can be used.
\begin{equation}
    \label{eq:prec1}
    (P^t)^2 = \beta_t(P^{t-1})^2 + (1-\beta_t)(H^t)^2.
\end{equation}


This update scheme is satisfied by \textsc{Adam}-based methods with $(H^t)^2 = \text{diag} (g^t \odot g^t)$ and by \textsc{AdaHessian} \citep{yao2021adahessian} with $(H^t)^2 = \text{diag} (z^t \odot \nabla^2 f(x^t))^2$, where $\odot$ denotes the component-wise product between two vectors, and $z^t$ are from Rademacher distribution, i.e. all components from vector are independent and equal to $\pm 1$ with probability \nicefrac{1}{2}. 

We want the preconditioner being a positive define matrix, thus, it is typical to modify $P^t$ a bit:
\begin{equation}
\label{eq:prec2}
    \hspace{-0.2cm}
    (\hat{P}^t)_{ii} = \max\{e, |P^t|_{ii}\},
\end{equation}
where $e$ is a (small) positive parameter. 
There are also other possible update rules, one of which is
\begin{equation*}
    P^t = \beta_t (P^{t-1}) + (1-\beta_t)H^t.
\end{equation*}
For example, such a rule is extended in \textsc{OASIS} \citep{jahani2021doubly} with $\beta_t \equiv \beta$ and $H^t = \text{diag} (z^t \odot \nabla^2 f(x^t))$. 
We also note additional details in the construction of a positively defined preconditioner.
We can also alternatively define $\hat{P}^t$ as $(\hat{P})_{ii} = |P|_{ii} + e$. Most importantly, both of these approaches construct diagonal matrices with positive elements. We introduce the crucial for our analysis assumption. 
\begin{assumption} \label{asscaled}
For any \( t \geqslant 1 \), we have \( \alpha I \preccurlyeq \hat{P}^t \preccurlyeq \Gamma I \).
\end{assumption} 
The correct proof of this statement, as well as a more detailed description of the diagonal preconditioner, is provided in \citep{sadiev2024stochastic}. We mention that for \textsc{AdaHessian} and \textsc{OASIS} preconditioners, $\Gamma = \sqrt{d} L$, and for \textsc{Adam} and \textsc{RMSProp}, under the condition $\|\nabla f(x)\| \leqslant M, \Gamma = M$. Thus, having constructed a diagonal preconditioner, we can proceed to the analysis of scaled methods.

\subsection{Scaled Bant} \label{E:Bant}

Before presenting the results for \textsc{Scaled Bant} (Algorithm \ref{alg4}), we need to provide the preliminary analysis.
To derive the final estimation, let us prove auxiliary lemmas.
\begin{lemma}\label{lemma:scaled}
If the diagonal preconditioner $\hat{P}$ is such that $\alpha I\preccurlyeq \hat{P} \preccurlyeq \Gamma I$, the following estimates are valid:
\begin{enumerate}[label=(\alph*), ref=\thecorollary(\alph*), leftmargin=0.8cm]
        \item \label{lemma:scaled1}
        \parbox{\linewidth}{
\begin{eqnarray*}
\frac{1}{\Gamma}\|g\|^2 \leqslant & \left\|g\right\|_{\hat{P}^{-1}}^2 &\leqslant \frac{1}{\alpha}\|g\|^2,
\end{eqnarray*}}
\item \label{lemma:scaled2}
        \parbox{\linewidth}{\begin{eqnarray*}
\frac{1}{\Gamma^2}\|g\|^2  \leqslant & \left\|\hat{P}^{-1}g\right\|^2 &\leqslant  \frac{1}{\alpha^2}\|g\|^2,
\end{eqnarray*}
}
\end{enumerate}
where $\langle h, g\rangle_{ \hat{P}^{-1}} \overset{def}{=} \left\langle h, \hat{P}^{-1}g \right\rangle$, $h, g \in \mathbb{R}^d$
\begin{proof}
\begin{eqnarray*}  
I \preccurlyeq \frac{1}{\alpha} \hat{P} ~ \Rightarrow ~ \left\|\hat{P}^{-1}g\right\|^2 = \left\langle I \hat{P}^{-1}g, \hat{P}^{-1}g \right\rangle &\leqslant& \frac{1}{\alpha}\left\langle g, \hat{P}^{-1}g \right\rangle ~ \overset{def}{=} ~ \frac{1}{\alpha}\left\langle g,g \right\rangle_{ \hat{P}^{-1}} = \frac{1}{\alpha}\|g\|_{ \hat{P}^{-1}}^2.\\
\left\langle g, \hat{P}^{-1}g \right\rangle ~ &\leqslant& ~ \frac{1}{\alpha}\left\langle g,g \right\rangle = \frac{1}{\alpha}\|g\|^2.\\
\hat{P} \preccurlyeq \Gamma I  ~ \Rightarrow ~ \left\|\hat{P}^{-1}g\right\|^2 = \left\langle I \hat{P}^{-1}g, \hat{P}^{-1}g \right\rangle &\geqslant& \frac{1}{\Gamma}\left\langle g, \hat{P}^{-1}g \right\rangle ~ \overset{def}{=} ~ \frac{1}{\Gamma}\left\langle g,g \right\rangle_{ \hat{P}^{-1}} = \frac{1}{\Gamma}\|g\|_{ \hat{P}^{-1}}^2.\\
\left\langle g, \hat{P}^{-1}g \right\rangle ~ &\geqslant& ~ \frac{1}{\Gamma}\left\langle g,g \right\rangle = \frac{1}{\Gamma}\|g\|^2.
\end{eqnarray*}
\end{proof}
\end{lemma}

\begin{algorithm}[ht]
\caption{Scaled Bant}\label{alg4}
\begin{algorithmic}[1]
\STATE \textbf{Input:} Starting point $x^0 \in \mathbb{R}^d$
\STATE \textbf{Parameters:} Stepsize $\gamma > 0$, momentum parameter $\beta \in [0, 1]$
\FOR{$t = 0, 1, 2, \ldots, T-1$}
    \STATE Server sends $x^t$ to each worker
    \FORALL{workers $i = 0, 1, 2, \ldots, n$ \textbf{in parallel}}
        \STATE Generate $\xi_i^t$ independently
        \STATE Compute stochastic gradient $g_i(x^t, \xi_i)$
        \STATE Send $g_i^t = g_i(x^t, \xi_i)$ to server
    \ENDFOR
    \STATE $\omega^{t} = (1-\beta)\omega_i^{t-1} + \beta\frac{[\hat{f} (x^t) - \hat{f} (x^t - \gamma \left(\hat{P}^t\right)^{-1} g_i^t)]_0}{\sum\nolimits_{j = 1}^n [\hat{f} (x^t) - \hat{f} (x^t - \gamma \left(\hat{P}^t\right)^{-1} g_j^t)]_0}$ 
    \IF {each $\left[\hat{f} (x^t) - \hat{f} (x^t - \gamma \left(\hat{P}^t\right)^{-1} g_i^t)\right]_0 = 0$}
        \STATE $\omega_i^t = (1-\beta)\omega_i^{t-1} + \beta\frac{1}{n}$
    \ENDIF
    \STATE $x^{t+1} = x^t - \gamma\left(\hat{P}^t\right)^{-1}\sum\nolimits_{i=1}^n \mathbb{I}_{[\hat{f} (x^t) - \hat{f} (x^t - \gamma \left(\hat{P}^t\right)^{-1} g_i^t) > 0]} \omega_i^{t} g_i^t$ 
    \STATE $\hat{P}^t$ is the function of $\hat{P}^{t-1}$ and $H^t$, e.g., as \eqref{eq:prec1} + \eqref{eq:prec2}
\ENDFOR
\STATE \textbf{Output:} $\frac{1}{T}\sum\limits_{t = 0}^{T-1} x^t$
\end{algorithmic}
\end{algorithm}

\begin{lemma}[Scaled version of Lemma \ref{lemma:scetch}]\label{lemma:scetchscaled}  Under Assumptions \ref{as1}, \ref{as2convex}, \ref{as5}, \ref{asscaled}, the following holds for the iteration of \text{Algorithm \ref{alg4}}:
\begin{eqnarray*}
    \hat{f}(x^{t+1}) &  \leqslant & \hat{f}(x^t) - \gamma\beta\frac{1}{n}\sum\limits_{i\in \mathcal{G}}\left\langle\nabla\hat{f}(x^t), (\hat{P}^t)^{-1} g_i^t\right\rangle + \frac{L\gamma^2 \beta}{2n}\sum\limits_{i\in \mathcal{G}}\|(\hat{P}^t)^{-1}g_i^t\|^2
\end{eqnarray*}
\begin{proof}
Our analysis implies the same as was done in Lemma \ref{lemma:scetch}. Since the only thing that has changed is that we additionally scale the gradient at each step, the analysis is similar and we obtain the final estimate.
\end{proof}
\end{lemma}
\begin{lemma} [Scaled version of Lemma \ref{lemma2}]\label{lemma2:scaled}
Suppose Assumption \ref{as1} holds. Then for all $x \in \mathbb{R}^d$ and $g_i = g_i(x, \xi_i)$, the following estimate is valid:
\begin{eqnarray*}
- \gamma \left\langle \nabla\hat{f}(x), \frac{1}{ G } \sum_{i \in \mathcal{G}} \nabla f_i(x) \right\rangle^2 & \leqslant & - \frac{\gamma}{2\Gamma}\left\|\nabla f(x) \right\|^2 + \frac{\gamma}{2\alpha}\zeta(N) + \frac{\gamma}{2\alpha} \left(\delta_1 + \delta_2 \|\nabla f(x^t)\|^2\right).
\end{eqnarray*}

\begin{proof}
We commence by examining the difference $\nabla f(x) - \nabla \hat{f}(x)$:
\begin{eqnarray*}
- \gamma \left\langle \nabla\hat{f}(x), \frac{1}{ G } \sum_{i \in \mathcal{G}} \nabla f_i(x) \right\rangle_{\hat{P}^{-1}} & = & \gamma \left\langle \nabla f(x) - \nabla \hat{f}(x),\frac{1}{ G } \sum_{i \in \mathcal{G}} \nabla f_i(x) \right\rangle_{\hat{P}^{-1}} \\
& & - \gamma \left\langle \nabla f(x),\frac{1}{ G } \sum_{i \in \mathcal{G}} \nabla f_i(x)\right\rangle_{\hat{P}^{-1}}.
\end{eqnarray*}
Next, we continue with further manipulations on the first term:
\begin{eqnarray*} 
& & \gamma \left\langle \nabla f(x) - \nabla \hat{f}(x), \frac{1}{ G } \sum_{i \in \mathcal{G}} \nabla f_i(x) \right\rangle_{\hat{P}^{-1}} \\
& & ~~~~~~\overset{\eqref{bi:young}}{\leqslant}~~~~~ \frac{\gamma}{2} \left\|\nabla f(x) - \nabla \hat{f}(x)\right\|_{\hat{P}^{-1}}^2  + \frac{\gamma}{2}\left\|\frac{1}{ G } \sum_{i \in \mathcal{G}} \nabla f_i(x)\right\|_{\hat{P}^{-1}}^2 \\
& &\overset{(\text{Lemma }\ref{lemma:scaled1})}{\leqslant} \frac{\gamma}{2\alpha^2} \left\|\nabla f(x) - \nabla \hat{f}(x)\right\|^2  + \frac{\gamma}{2}\left\|\frac{1}{ G } \sum_{i \in \mathcal{G}} \nabla f_i(x)\right\|_{\hat{P}^{-1}}^2\\
& &~~\overset{(\text{Lemma }\ref{lemmaShalevShwartz})}{\leqslant}~~ \frac{\gamma}{2\alpha} \zeta(N)  + \frac{\gamma}{2}\left\|\frac{1}{ G } \sum_{i \in \mathcal{G}} \nabla f_i(x)\right\|_{\hat{P}^{-1}}^2,
\end{eqnarray*} 
and with the second term, \begin{eqnarray*}
- \gamma\left\langle \nabla f(x), \frac{1}{ G }\sum\limits_{i\in \mathcal{G}} \nabla f_i(x)\right\rangle_{\hat{P}^{-1}} & \overset{\eqref{bi:norm}}{=}& -\frac{\gamma}{2}\|\nabla f(x)\|_{\hat{P}^{-1}}^2 -\frac{\gamma}{2} \left\|\frac{1}{ G }\sum\limits_{i \in \mathcal{G}} \nabla f_i(x)\right\|_{\hat{P}^{-1}}^2 \\
&& + \frac{\gamma}{2} \left\|\frac{1}{ G }\sum\limits_{i \in \mathcal{G}} (\nabla f_i(x) - \nabla f(x))\right\|_{\hat{P}^{-1}}^2 \\
&\overset{\eqref{bi:CauchySchwarz}}{\leqslant}& - \frac{\gamma}{2}\|\nabla f(x)\|_{\hat{P}^{-1}}^2 -\frac{\gamma}{2} \left\|\frac{1}{ G }\sum\limits_{i \in \mathcal{G}} \nabla f_i(x)\right\|_{\hat{P}^{-1}}^2 \\
&&+ \frac{\gamma}{2 G }\sum\limits_{i \in \mathcal{G}} \left\| (\nabla f_i(x) - \nabla f(x))\right\|_{\hat{P}^{-1}}^2\\
&\overset{(\text{Lemma }\ref{lemma:scaled})}{\leqslant}& - \frac{\gamma}{2\Gamma}\|\nabla f(x)\|^2 -\frac{\gamma}{2} \left\|\frac{1}{ G }\sum\limits_{i \in \mathcal{G}} \nabla f_i(x)\right\|_{\hat{P}^{-1}}^2 \\
&&+ \frac{\gamma}{2\alpha G }\sum\limits_{i \in \mathcal{G}} \left\| (\nabla f_i(x) - \nabla f(x))\right\|^2\\
&\overset{(\text{Ass. }\ref{as4})}{\leqslant}& - \frac{\gamma}{2\Gamma}\|\nabla f(x)\|^2 -\frac{\gamma}{2} \left\|\frac{1}{ G }\sum\limits_{i \in \mathcal{G}} \nabla f_i(x)\right\|_{\hat{P}^{-1}}^2 \\
&&+ \frac{\gamma}{2\alpha} \left(\delta_1 + \delta_2 \|\nabla f(x)\|^2\right).
\end{eqnarray*}
    Summing up substantiates the claim of the lemma.
\end{proof}
\end{lemma}

We are now ready to write out the main results for the scaled methods.
\begin{theorem}\label{TheoremScaledITR}
    Under \text{Assumptions \ref{as1}, \ref{as2convex}, \ref{as3}, \ref{as4} with $\delta_2 \leqslant \frac{2 \Gamma - \alpha}{\alpha+\frac{4\Gamma ^2}{\alpha^2}}$, \ref{as5}, \ref{asscaled}}, for solving the problem \eqref{eq:setting_dist}, after $T$ iteration of \text{Algorithm \ref{alg4}} with $\gamma \leqslant \frac{\alpha}{12L}$, it holds that
    \begin{eqnarray*}
    \frac{1}{T}\sum\limits_{t = 0}^{T-1} \mathbb E \|\nabla f(x^t)\|^2 \leqslant \frac{4\mathbb{E}\left[ \hat{f}(x^{0}) - \hat{f}(\hat{x}^*)\right]}{\gamma T}\cdot\frac{n\Gamma}{\beta G } + \frac{3\Gamma}{\alpha}\delta_1  + \frac{6L \gamma\Gamma }{\alpha^2}\sigma^2 + \frac{2\Gamma}{\alpha}\zeta(N).
\end{eqnarray*}

\begin{proof}[Proof of Theorem~\ref{TheoremScaledITR}]

Similar to the aforementioned Lemma \ref{lemma:scetch} (with the preconditioner added):
\begin{eqnarray*}
\hat{f}(x^{t+1})  & \leqslant&  \hat{f}(x^t) - \gamma\beta\frac{1}{n}\sum\limits_{i\in \mathcal{G}}\left\langle\nabla\hat{f}(x^t), (\hat{P}^t)^{-1} g_i^t\right\rangle + \frac{L\gamma^2 \beta}{2n}\sum\limits_{i\in \mathcal{G}}\|(\hat{P}^t)^{-1}g_i^t\|^2 \\
&\overset{(\text{Lemma }\ref{lemma:scaled})}{\leqslant}& \hat{f}(x^t) - \gamma\beta\frac{1}{n}\sum\limits_{i\in \mathcal{G}}\left\langle\nabla\hat{f}(x^t), g_i^t\right\rangle_{(\hat{P}^t)^{-1}} + \frac{L\gamma^2 \beta}{2n\alpha^2} \sum\limits_{i\in \mathcal{G}}\|g_i^t\|^2. 
\end{eqnarray*}

Taking the expectation of both sides of the inequality:
\begin{eqnarray*}
\mathbb{E}\hat{f}(x^{t+1}) &\leqslant & \mathbb{E}\hat{f}(x^t) - \gamma\beta \cdot\frac{ G }{n}\left\langle\nabla\hat{f}(x^t), \frac{1}{ G }\sum\limits_{i\in \mathcal{G}} \nabla f_i(x^t)\right\rangle_{(\hat{P}^t)^{-1}} \\
& & + \frac{L\gamma^2 \beta}{2n\alpha^2}\sum\limits_{i\in \mathcal{G}}\mathbb{E}\left\| g_i^t\right\|^2 \\
    &\overset{(\text{Lemma } \ref{lemma2:scaled})}{\leqslant} & \mathbb{E}\hat{f}(x^t) + \frac{\gamma\beta G }{2n\alpha}\zeta(N) - \frac{\gamma\beta G }{2n\Gamma}\left\|\nabla f(x^t) \right\|^2 \\
&&+ \frac{\gamma\beta G }{2n\alpha} \left(\delta_1 + \delta_2 \|\nabla f(x^t)\|^2\right) +  \frac{L\gamma^2 \beta}{2n\alpha ^2}\sum\limits_{i\in \mathcal{G}}\mathbb{E}\left\| g_i^t\right\|^2 \\
    & \overset{\eqref{bi:CauchySchwarz}}{\leqslant}& \mathbb{E}\hat{f}(x^t) + \frac{\gamma\beta G }{2n\alpha}\zeta(N) - \frac{\gamma\beta G }{2n\Gamma}\left\|\nabla f(x^t) \right\|^2 + \frac{\gamma\beta G }{2n\alpha} \left(\delta_1 + \delta_2 \|\nabla f(x^t)\|^2\right) \\
&&+  \frac{3L \gamma ^2\beta}{2n\alpha^2}\left(\sum_{i \in \mathcal{G}} \mathbb{E}\left\|\nabla f(x^t) - f_i(x^t)\right\|^2 + \sum_{i \in \mathcal{G}} \mathbb{E}\left\|\nabla f_i(x^t) - g_i^t\right\|^2\right) \\
& & + \frac{3L \gamma ^2\beta G }{2n\alpha^2}\left\|\nabla f(x^t) \right\|^2 \\
& \overset{(\text{Ass. }\ref{as3}, \ref{as4})}{\leqslant}& \mathbb{E}\hat{f}(x^t) + \frac{\gamma\beta G }{2n\alpha}\zeta(N) - \frac{\gamma\beta G }{2n\Gamma}\left\|\nabla f(x^t) \right\|^2 + \frac{\gamma\beta G }{2n\alpha} \left(\delta_1 + \delta_2 \|\nabla f(x^t)\|^2\right) \\
&&+  \frac{3L \gamma ^2\beta G }{2n\alpha^2}\left(\delta_1 + \delta_2 \|\nabla f(x^t)\|^2 +  \sigma ^2\right) + \frac{3L \gamma ^2\beta G }{2n\alpha^2}\left\|\nabla f(x^t) \right\|^2 \\
    & \overset{(\mathrm{Ass. \ref{as3}})}{\leqslant}& \mathbb E [\hat{f}(x^{t})] - \frac{\gamma\beta G }{2n\Gamma}\left[1 - 3L \gamma\frac{\Gamma}{\alpha^2} - (1 + 3L \gamma)\frac{\Gamma}{\alpha^2}\delta_2 \right]\|\nabla f(x^t)\|^2 \\
&& +\frac{\gamma\beta  G }{2n\alpha}(1 + \frac{3L \gamma}{\alpha})\delta_1 +  \frac{3\gamma^2\beta L G }{2n\alpha^2}\sigma^2 + \frac{\gamma\beta  G }{2n\alpha}\zeta(N).
\end{eqnarray*}
Now we have to choose $\gamma$.
We want $\left[1 - 3L \gamma\frac{\Gamma}{\alpha^2} - (1 + 3L \gamma)\frac{\Gamma}{\alpha^2}\delta_2 \right] \geqslant \frac{1}{2}$. Then
\begin{eqnarray*}
\gamma &\leqslant&\frac{1-\frac{2\Gamma}{\alpha^2}\delta_2}{6L\frac{\Gamma}{\alpha^2}(1+\delta_2)}.
\end{eqnarray*}
Let us choose $\delta_2 \leqslant \frac{2 \Gamma - \alpha}{\alpha+\frac{4\Gamma ^2}{\alpha^2}}$, then we have 
$\frac{1}{2} \leqslant \frac{1-\frac{2\Gamma}{\alpha^2}\delta_2}{\frac{\Gamma}{\alpha}(1+\delta_2)}$. Thus,
\begin{eqnarray*}
\gamma  &\leqslant&\frac{\alpha}{12L} ~~\leqslant~~\frac{1-\frac{2\Gamma}{\alpha^2}\delta_2}{6L\frac{\Gamma}{\alpha^2}(1+\delta_2)}.
\end{eqnarray*}

Using $\gamma \leqslant \frac{\alpha}{12L}$ and summing over the iterations, we get the bound:
\begin{eqnarray*}
    \frac{1}{T}\sum\limits_{t = 0}^{T-1} \mathbb{E} \|\nabla f(x^t)\|^2 &\leqslant& \frac{\mathbb E\left[\hat{f}(x^0) - \hat{f}(\hat{x}^*)\right]}{\gamma T}\cdot\frac{4 n\Gamma}{\beta G } + \frac{3\Gamma}{\alpha}\delta_1  + \frac{6L \gamma\Gamma }{\alpha^2}\sigma^2 + \frac{2\Gamma}{\alpha}\zeta(N).
\end{eqnarray*}
\end{proof}
\end{theorem}

Now we provide the final convergence rate.
\begin{corollary}\label{ITRScaledCorollary}
    Under assumptions of \text{Theorem \ref{TheoremScaledITR}} for solving the problem \eqref{eq:setting_dist}, after $T$ iterations of Algorithm \ref{E:Bant} with special tunings of $\gamma$, the following holds:
        \begin{align*}
            \frac{1}{T}\sum\limits_{i=1}^{T-1}\|\nabla f(x^t)\|^2 =& \mathcal{O}\left( \frac{\mathbb{E} \left[ \hat{f}(x^{0}) - \hat{f}(\hat{x}^*)\right] L n}{\beta G T}\cdot\frac{\Gamma}{\alpha} + \frac{\sigma\sqrt{\mathbb{E} \left[ \hat{f}(x^{0}) - \hat{f}(\hat{x}^*)\right]Ln}}{\sqrt{\beta  G  T}}\cdot \frac{\Gamma}{\alpha} +\left(\delta_1+\zeta(N)\right)\cdot\frac{\Gamma}{\alpha}\right) .
        \end{align*}
\end{corollary}

The proof of \text{Corollary~\ref{ITRScaledCorollary}} completely mirrors the proof of \text{Corollary \ref{ITRSGDCorollary}}.

\begin{remark}
    We get a result similar to \text{Corollary \ref{ITRSGDCorollary}}, but with the deterioration that each term is multiplied by an additional constant $\nicefrac{\Gamma}{\alpha} > 1$. This result suits us, since in \citep{sadiev2024stochastic} the result for the \textsc{Scaled SARAH} method corresponds similarly to the result for the classical \textsc{SARAH} \citep{nguyen2017sarah} method. 
\end{remark}

\subsection{Scaled AutoBant} \label{E:Bant2}
Now, let us consider the second algorithm for scaled methods - \textsc{Scaled AutoBant} (Algorithm \ref{alg3}). This section presents an algorithm which is an adaptive version of Algorithm \ref{alg1} taking into account the diagonal preconditioner. Now we provide an estimate for the convergence of the \textsc{Scaled AutoBant} method.

\begin{algorithm}[ht]
\caption{Scaled AutoBant}\label{alg3}
\begin{algorithmic}[1]
\STATE \textbf{Input:} Starting point $x^0 \in \mathbb{R}^d$
\STATE \textbf{Parameters:} Stepsize $\gamma > 0$, error accuracy $\delta$
\FOR{$t = 0, 1, 2, \ldots, T-1$}
    \STATE Server sends $x^t$ to each worker
    \FORALL{workers $i = 0, 1, 2, \ldots, n$ \textbf{in parallel}}
        \STATE Generate $\xi_i^t$ independently
        \STATE Compute stochastic gradient $g_i(x^t, \xi_i)$
        \STATE Send $g_i^t = g_i(x^t, \xi_i)$ to server
    \ENDFOR
    \STATE $\omega^{t} \approx \arg\underset{\omega\in\Delta_1^n}{\min} \hat{f}\left(x^t - \gamma\left(\hat{P}^t\right)^{-1}\sum\nolimits_{i=1}^n \omega_i g_i^t\right)$ 
    \STATE $x^{t+1} = x^t - \gamma\left(\hat{P}^t\right)^{-1}\sum\nolimits_{i=1}^n \omega_i^{t} g_i^t$
    \STATE $\hat{P}^t$ is the function of $\hat{P}^{t-1}$ and $H^t$, e.g., as \eqref{eq:prec1} + \eqref{eq:prec2}
\ENDFOR
\STATE \textbf{Output:} $\frac{1}{T}\sum\limits_{t = 0}^{T-1} x^t$
\end{algorithmic}
\end{algorithm}

\begin{theorem}\label{TheoremScaledTFM}
    Under \text{Assumptions \ref{as1}, \ref{as2nonconvex} \ref{as3}, \ref{as4} with $\delta_2 \leqslant 0.25$,  \ref{as5}, \ref{asscaled}}, for solving the problem \eqref{eq:setting_dist}, after $T$ iterations of \text{Algorithm \ref{alg3}} with $\gamma \leqslant \frac{\alpha}{12L}$, the following holds:
    \begin{eqnarray*}
    \frac{1}{T}\sum\limits_{t = 0}^{T-1} \mathbb E \|\nabla f(x^t)\|^2 &\leqslant& \frac{\mathbb{E} \left[ \hat{f}(x^{0}) - \hat{f}(\hat{x}^*)\right]\cdot 4\Gamma}{\gamma T} + \frac{3\Gamma}{\alpha}\delta_1  + \frac{6L \gamma\Gamma }{\alpha^2 G }\sigma^2 + \frac{2\Gamma}{\alpha}\zeta(N).
\end{eqnarray*}

\begin{proof}[Proof of Theorem \ref{TheoremScaledTFM}]
Note we estimate the trial function value:
\begin{eqnarray*}
\hat{f}(x^{t+1}) &= &\underset{\omega \in \Delta _1 ^n}{\min} \hat{f}\left(x^t - \gamma (\hat{P}^t)^{-1} \sum\limits_{i=1}^n \omega_i g_i^t\right) \\
&\leqslant & \hat{f}\left(x^t - \gamma (\hat{P}^t)^{-1} \frac{1}{ G } \sum\limits_{i \in \mathcal{G}} g_i^t\right) \\
    &\overset{\eqref{bi:lsmoothness}}{\leqslant}& \hat{f}\left(x^t\right) - \gamma \left\langle \nabla \hat{f}(x^t), (\hat{P}^t)^{-1} \frac{1}{ G } \sum_{i \in \mathcal{G}} g_i^t \right\rangle \\
    & & + \frac{L\gamma^2}{2} \left\| (\hat{P}^t)^{-1} \frac{1}{ G } \sum_{i \in \mathcal{G}} g_i^t \right\|^2 \\
& \overset{(\text{Lemma }\ref{lemma:scaled})}{\leqslant} &\hat{f}(x^t) - \gamma \left\langle \nabla \hat{f}(x^t), \frac{1}{ G } \sum_{i \in \mathcal{G}} g_i^t \right\rangle_{(\hat{P}^t)^{-1}}  + \frac{L\gamma^2}{2\alpha^2} \left\| \frac{1}{ G } \sum_{i \in \mathcal{G}} g_i^t \right\|^2.
\end{eqnarray*}
Taking the expectation of both sides of the inequality:
\begin{eqnarray*}
\mathbb{E}\hat{f}(x^{t+1}) &\leqslant & \mathbb{E}\hat{f}(x^t) - \gamma \left\langle\nabla\hat{f}(x^t), \frac{1}{ G }\sum\limits_{i\in \mathcal{G}} \nabla f_i(x^t)\right\rangle_{(\hat{P}^t)^{-1}} \\
& & + \frac{L\gamma^2 }{2\alpha^2}\mathbb{E}\left\|\frac{1}{ G }\sum\limits_{i\in \mathcal{G}} g_i^t\right\|^2 \\
    &\overset{(\text{Lemma } \ref{lemma2:scaled})}{\leqslant} & \mathbb{E}\hat{f}(x^t) + \frac{\gamma}{2\alpha}\zeta(N) - \frac{\gamma}{2\Gamma}\left\|\nabla f(x^t) \right\|^2 \\
&&+ \frac{\gamma}{2\alpha} \left(\delta_1 + \delta_2 \|\nabla f(x^t)\|^2\right)+ \frac{L\gamma^2 }{2\alpha^2}\mathbb{E}\left\|\frac{1}{ G }\sum\limits_{i\in \mathcal{G}} g_i^t\right\|^2 \\
    & \overset{\eqref{bi:CauchySchwarz}}{\leqslant}& \mathbb{E}\hat{f}(x^t) + \frac{\gamma}{2\alpha}\zeta(N) - \frac{\gamma}{2\Gamma}\left\|\nabla f(x^t) \right\|^2 + \frac{\gamma}{2\alpha} \left(\delta_1 + \delta_2 \|\nabla f(x^t)\|^2\right) \\
&&+  \frac{3L \gamma ^2}{2\alpha^2}\left( \mathbb{E}\left\|\frac{1}{ G }\sum\limits_{i\in \mathcal{G}}(\nabla f(x^t) - f_i(x^t))\right\|^2 \right. \\
& & \left. + \mathbb{E}\left\|\frac{1}{ G }\sum\limits_{i\in \mathcal{G}}\nabla (f_i(x^t) - g_i^t)\right\|^2\right) + \frac{3L \gamma ^2}{2\alpha^2}\left\|\nabla f(x^t) \right\|^2.
\end{eqnarray*}
Due to the fact that $\mathbb{E} g_i^t = \nabla f_i(x^t)$ and $\mathbb{E} \langle \nabla f_i (x^t) - g_i^t, \nabla f_j (x^t) - g_j^t \rangle = 0$,
\begin{eqnarray*}   
\mathbb{E}\hat{f}(x^{t+1}) &\leqslant&\mathbb{E}\hat{f}(x^t) + \frac{\gamma}{2\alpha}\zeta(N) - \frac{\gamma}{2\Gamma}\left\|\nabla f(x^t) \right\|^2 + \frac{\gamma}{2\alpha} \left(\delta_1 + \delta_2 \|\nabla f(x^t)\|^2\right) \\
&&+  \frac{3L \gamma ^2}{2 G \alpha^2} \sum\limits_{i\in \mathcal{G}}\mathbb{E}\left\|(\nabla f(x^t) - f_i(x^t))\right\|^2 \\
& & + \frac{3L \gamma ^2}{2 G ^2\alpha^2}\sum\limits_{i\in \mathcal{G}}\mathbb{E}\left\|\nabla (f_i(x^t) - g_i^t)\right\|^2 + \frac{3L \gamma ^2}{2\alpha^2}\left\|\nabla f(x^t) \right\|^2
\\
& \overset{(\text{Ass. }\ref{as3}, \ref{as4})}{\leqslant}& \mathbb{E}\hat{f}(x^t) + \frac{\gamma}{2\alpha}\zeta(N) - \frac{\gamma}{2\Gamma}\left\|\nabla f(x^t) \right\|^2 + \frac{\gamma}{2\alpha} \left(\delta_1 + \delta_2 \|\nabla f(x^t)\|^2\right) \\
&&+  \frac{3L \gamma ^2}{2\alpha^2}\left(\delta_1 + \delta_2 \|\nabla f(x^t)\|^2 +  \frac{1}{ G }\sigma ^2\right) + \frac{3L \gamma ^2}{2\alpha^2}\left\|\nabla f(x^t) \right\|^2 \\
    & \overset{(\mathrm{Ass. \ref{as3}})}{\leqslant}& \mathbb E [\hat{f}(x^{t})] - \frac{\gamma}{2\Gamma}\left[1 - 3L \gamma\frac{\Gamma}{\alpha^2} - (1 + 3L \gamma)\frac{\Gamma}{\alpha^2}\delta_2 \right]\|\nabla f(x^t)\|^2 \\
&& +\frac{\gamma}{2\alpha}(1 + \frac{3L \gamma}{\alpha})\delta_1 +  \frac{3L \gamma^2}{2\alpha^2 G }\sigma^2 + \frac{\gamma }{2\alpha}\zeta(N).
\end{eqnarray*}
Now we have to choose $\gamma$:
We want $\left[1 - 3L \gamma\frac{\Gamma}{\alpha^2} - (1 + 3L \gamma)\frac{\Gamma}{\alpha^2}\delta_2 \right] \geqslant \frac{1}{2}$. Then
\begin{eqnarray*}
\gamma &\leqslant&\frac{1-\frac{2\Gamma}{\alpha^2}\delta_2}{6L\frac{\Gamma}{\alpha^2}(1+\delta_2)}.
\end{eqnarray*}
Let us choose $\delta_2 \leqslant \frac{2 \Gamma - \alpha}{\alpha+\frac{4\Gamma ^2}{\alpha^2}}$, then we have 
$\frac{1}{2} \leqslant \frac{1-\frac{2\Gamma}{\alpha^2}\delta_2}{\frac{\Gamma}{\alpha}(1+\delta_2)}$. Thus,
\begin{eqnarray*}
\gamma  &\leqslant&\frac{\alpha}{12L} ~~\leqslant~~\frac{1-\frac{2\Gamma}{\alpha^2}\delta_2}{6L\frac{\Gamma}{\alpha^2}(1+\delta_2)}.
\end{eqnarray*}

Using $\gamma \leqslant \frac{\alpha}{12L}$, summing over the iterations and taking the expected value at the initial point and at the optimum point, we get the bound:
\begin{eqnarray*}
    \frac{1}{T}\sum\limits_{t = 0}^{T-1} \mathbb E \|\nabla f(x^t)\|^2 &\leqslant& \frac{\mathbb{E} \left[\hat{f}(x^0) - \hat{f}(\hat{x}^*)\right]\cdot 4\Gamma}{\gamma T} + \frac{3\Gamma}{\alpha}\delta_1  + \frac{6L \gamma\Gamma }{\alpha^2 G }\sigma^2 + \frac{2\Gamma}{\alpha}\zeta(N).
\end{eqnarray*}
\end{proof}
\end{theorem}
Now, let us present the final convergence rate for this algorithm.
\begin{corollary}\label{TFMScaledCorollary}
    Under assumptions of \text{Theorem \ref{TheoremScaledTFM}}, for solving the problem \eqref{eq:setting_dist}, after $T$ iterations of Algorithm \eqref{alg3} with special tunings of $\gamma$, the following holds:
    \begin{align*}
    \frac{1}{T}\sum\limits_{i=1}^{T-1}\|\nabla f(x^t)\|^2 =& \mathcal{O}\Biggl( \frac{\mathbb{E}\left[ \hat{f}(x^{0}) - \hat{f}(\hat{x}^*)\right] L G}{T}\cdot\frac{\Gamma}{\alpha} + \frac{\sigma\sqrt{\mathbb{E}\left[ \hat{f}(x^{0}) - \hat{f}(\hat{x}^*)\right]LG}}{\sqrt{T}}\cdot \frac{\Gamma}{\alpha} +\left(\delta_1+\zeta(N)\right)\cdot\frac{\Gamma}{\alpha}\Biggr)
    \end{align*}
\end{corollary}
\begin{remark}
    The boundary is the same as for \texttt{AutoBant}, with the only aggravation that several summands are multiplied by $\frac{\Gamma}{\alpha} > 1$. This result suits us for the same reason as the result of the \textsc{Scaled Bant} method. 
\end{remark}

The proof of \text{Corollary~\ref{TFMScaledCorollary}} completely mirrors the proof of \text{Corollary \ref{TFMSGDCorollary}}.

\newpage
\section{Local methods} \label{sec:local_appendix}

As highlighted in the main section, the significant expense associated with communication remains a critical concern in various fields. The communication bottleneck can act as a substantial barrier, limiting efficiency and hindering progress. To address this challenge, many researches are turning to local approaches, which focus on minimizing the need for exchanging the information \citep{woodworth2020minibatch, koloskova2020unified, khaled2020tighter, gorbunov2021local}. The idea is that each device performs a predefined number of local steps without utilizing information from other devices, and at the end of such a round, the server performs a mutual update. In this section we adapt our \texttt{AutoBant} algorithm to this scenario. Below we present the formal description of the \textsc{Local AutoBant} method (Algorithm \ref{alg:local_tfm}).
\begin{algorithm}[ht]
\caption{Local AutoBant}\label{alg:local_tfm}
\begin{algorithmic}[1]
\STATE \textbf{Input:} Starting point $x^0 \in \mathbb{R}^d$, local round length $l$
\STATE \textbf{Parameters:} Stepsize $\gamma > 0$, error accuracy $\delta$
\FOR{$t = 0, 1, 2, \ldots, T-1$}
    \IF{$t=0$}
        \STATE Server sends $x^0$ to each worker
    \ENDIF
    \FORALL{workers $i = 0, 1, 2, \ldots, n$ \textbf{in parallel}}
        \STATE Generate $\xi_i^t$ independently
        \STATE Compute stochastic gradient $g_i^t = g_i(x^t, \xi_i)$
        \IF{$t \neq t_{k\cdot l} \left(\text{for some~} k=\overline{0,\left\lfloor\nicefrac{T}{l}\right\rfloor}\right) \wedge t \neq T-1$}
            \STATE $x_i^{t+1} = x_i^{t} - \gamma g_i^t$
        \ELSE
            \STATE Send $x_i^t - \gamma g_i^t$ to server
        \ENDIF
    \ENDFOR
    \IF{$t = t_{k\cdot l} \left(\text{for some~} k=\overline{0,\left\lfloor\nicefrac{T}{l}\right\rfloor}\right) \vee t = T-1$}      
        \STATE $\omega^{t} \approx \arg\underset{\omega\in\Delta_1^n}{\min} \hat{f}\left(\sum\limits_{i=1}^n \omega_i\left(x_i^t - \gamma g_i^t\right)\right)$ 
        \STATE $x^{t+1} = \sum\limits_{i=1}^n \omega_i^t\left(x_i^t - \gamma g_i^t\right)$
        \STATE Server sends $x^{t+1}$ to each worker
    \ENDIF
\ENDFOR
\STATE \textbf{Output:} $\frac{1}{T}\sum\limits_{t = 0}^{T-1} x^t$
\end{algorithmic}
\end{algorithm}

In the convergence analysis of Algorithm \ref{alg:local_tfm}, we assume that at each local round, at least one device (including the server) acts as an honest worker. This implies that this device computes an honest stochastic gradient at each iteration of the round. This requirement is a natural extension of the assumption made in the analysis of the basic version of our methods, where we required at least one honest device (including the server) at each iteration (or in a local round of length 1). 

The analysis has some specific details. The key component is estimating how far devices can "move apart" from each other during a local round. We begin with this estimation and present the following lemma.
\begin{lemma}\label{lemma_differences_between_nodes}
    Under Assumptions \ref{as1}, \ref{as2convex}, \ref{as3}, \ref{as4}, \ref{as5}, at each iteration $t$ of Algorithm \ref{alg:local_tfm} with $\gamma\leqslant \frac{1}{4(l-1)L}$, the following estimate is valid:
    \begin{align*}
        \mathbb E V^t \leqslant \frac{9\delta_2\gamma}{2L}\sum\limits_{j=t_{k\cdot l}}^{t-1}\mathbb E\left\|\nabla f(\overline{x}^j)\right\|^2 + \frac{9\delta_1\gamma}{2L}(l-1) + 3\gamma^2\sigma^2(l-1),
    \end{align*}
    where $t_{k\cdot l}$ for some $k=\overline{0,\left\lfloor\nicefrac{T}{l}\right\rfloor}$ is the past to $t$-th iteration aggregation round, $V^t = \frac{1}{ G }\sum\limits_{i\in \mathcal{G}} \left\|x_i^t - \overline{x}^t\right\|^2$ and $\overline{x}^t = \frac{1}{ G }\sum\limits_{i\in \mathcal{G}} x_i^t$.
    \begin{proof}
    Utilizing notation $V^t = \frac{1}{ G }\sum\limits_{i\in \mathcal{G}} \left\|x_i^t - \overline{x}^t\right\|^2$ and $\overline{x}^t = \frac{1}{ G }\sum\limits_{i\in \mathcal{G}} x_i^t$ we mention, that for all iterations $t+1$, such that $t+1 = t_{k\cdot l}$ we have $V^{t+1} = \frac{1}{ G }\sum\limits_{i\in \mathcal{G}} \left\|x_i^{t+1} - \frac{1}{ G }\sum\limits_{i\in \mathcal{G}} x_i^{t+1}\right\|^2 = \frac{1}{ G }\sum\limits_{i\in \mathcal{G}} \left\|x^{t+1} - \frac{1}{ G }\sum\limits_{i\in \mathcal{G}} x^{t+1}\right\|^2 = 0$. For the rest iterations we write the step of the local update and use \eqref{bi:norm}:
        \begin{eqnarray*}
            \mathbb E \left\|x_i^{t+1} - \overline{x}^{t+1}\right\|^2 &=& \mathbb E \left\|x_i^{t} - \overline{x}^{t}\right\|^2 + \gamma^2\mathbb E\left\|g_i^{t} - \frac{1}{ G }\sum\limits_{i\in \mathcal{G}} g_i^t\right\|^2 \\
            & & - 2\gamma \mathbb E \left\langle x_i^{t} - \overline{x}^{t}, g_i^{t} - \frac{1}{ G }\sum\limits_{i\in \mathcal{G}} g_i^t\right\rangle \\
            &=& \mathbb E \left\|x_i^{t} - \overline{x}^{t}\right\|^2 + \gamma^2\mathbb E\left\|g_i^{t} - \frac{1}{ G }\sum\limits_{i\in \mathcal{G}} g_i^t\right\|^2 \\
            & & - 2\gamma \left\langle x_i^{t} - \overline{x}^{t}, \nabla f_i(x_i^{t}) - \frac{1}{ G }\sum\limits_{i\in \mathcal{G}} \nabla f_i(x_i^t)\right\rangle.
        \end{eqnarray*}
        Taking average over $i\in \mathcal{G}$,
        \begin{eqnarray}
            \notag \mathbb E V^{t+1} &=& \mathbb E V^t + \frac{\gamma^2}{ G }\sum\limits_{i\in \mathcal{G}}\mathbb E\left\|g_i^{t} - \frac{1}{ G }\sum\limits_{i\in \mathcal{G}} g_i^t\right\|^2 - \frac{2\gamma}{ G }\sum\limits_{i\in \mathcal{G}}\left\langle x_i^{t} - \overline{x}^{t}, \nabla f_i(x_i^{t})\right\rangle \\
            \notag& & + 2\gamma\left\langle \overline{x}^{t} - \overline{x}^{t}, \frac{1}{ G }\sum\limits_{i\in \mathcal{G}} \nabla f_i(x_i^t)\right\rangle\\
            \label{ldbn:ineq1}&=& \mathbb E V^t + \frac{\gamma^2}{ G }\sum\limits_{i\in \mathcal{G}}\mathbb E\left\|g_i^{t} - \frac{1}{ G }\sum\limits_{i\in \mathcal{G}} g_i^t\right\|^2 - \frac{2\gamma}{ G }\sum\limits_{i\in \mathcal{G}}\left\langle x_i^{t} - \overline{x}^{t}, \nabla f_i(x_i^{t})\right\rangle.
        \end{eqnarray}
        Now we need to estimate the second term. We start with \eqref{bi:norm}:
        \begin{eqnarray*}
            \mathbb E\left\|g_i^{t} - \frac{1}{ G }\sum\limits_{i\in \mathcal{G}} g_i^t\right\|^2 &=& \mathbb E\left\|g_i^{t} - \frac{1}{ G }\sum\limits_{i\in \mathcal{G}} \nabla f_i(x_i^t)\right\|^2 + \mathbb E\left\|\frac{1}{ G }\sum\limits_{i\in \mathcal{G}} \nabla f_i(x_i^t) - \frac{1}{ G }\sum\limits_{i\in \mathcal{G}} g_i^t\right\|^2 \\
            \notag& & + 2\mathbb E\left\langle g_i^{t} - \frac{1}{ G }\sum\limits_{i\in \mathcal{G}} \nabla f_i(x_i^t), \frac{1}{ G }\sum\limits_{i\in \mathcal{G}} \nabla f_i(x_i^t) - \frac{1}{ G }\sum\limits_{i\in \mathcal{G}} g_i^t\right\rangle \\
            \notag&\overset{(i)}{=}& \mathbb E\left\|g_i^{t} - \nabla f_i(x_i^t)\right\|^2 + \mathbb E\left\|\nabla f_i(x_i^t) - \frac{1}{ G }\sum\limits_{i\in \mathcal{G}} \nabla f_i(x_i^t)\right\|^2\\
            \notag& & + 2\mathbb E\left\langle g_i^{t} -  \nabla f_i(x_i^t), \nabla f_i(x_i^t) - \frac{1}{ G }\sum\limits_{i\in \mathcal{G}} \nabla f_i(x_i^t)\right\rangle \\
            & & + \mathbb E\left\|\frac{1}{ G }\sum\limits_{i\in \mathcal{G}} \nabla f_i(x_i^t) - \frac{1}{ G }\sum\limits_{i\in \mathcal{G}} g_i^t\right\|^2 \\
            \notag& & + 2\mathbb E\left\langle g_i^{t} - \frac{1}{ G }\sum\limits_{i\in \mathcal{G}} \nabla f_i(x_i^t), \frac{1}{ G }\sum\limits_{i\in \mathcal{G}} \nabla f_i(x_i^t) - \frac{1}{ G }\sum\limits_{i\in \mathcal{G}} g_i^t\right\rangle \\
            &\overset{(ii)}{=}& \mathbb E\left\|g_i^{t} -  \nabla f_i(x_i^t)\right\|^2 + \left\|\nabla f_i(x_i^t) - \frac{1}{ G }\sum\limits_{i\in \mathcal{G}} \nabla f_i(x_i^t)\right\|^2 \\
            & & + \mathbb E\left\|\frac{1}{ G }\sum\limits_{i\in \mathcal{G}} \nabla f_i(x_i^t) - \frac{1}{ G }\sum\limits_{i\in \mathcal{G}} g_i^t\right\|^2 \\
            \notag& & + 2\mathbb E\left\langle g_i^{t} - \frac{1}{ G }\sum\limits_{i\in \mathcal{G}} \nabla f_i(x_i^t), \frac{1}{ G }\sum\limits_{i\in \mathcal{G}} \nabla f_i(x_i^t) - \frac{1}{ G }\sum\limits_{i\in \mathcal{G}} g_i^t\right\rangle,
        \end{eqnarray*}
        where (\textit{i}) was made \eqref{bi:norm}, applied to the first norm and (\textit{ii}) by taking expectation of the first scalar product and obtaining it equal to zero. Next, averaging over $i\in \mathcal{G}$ and transforming the scalar product,
        \begin{eqnarray}
            \notag&&\frac{1}{ G }\sum\limits_{i\in \mathcal{G}}\mathbb E\left\|g_i^{t} - \frac{1}{ G }\sum\limits_{i\in \mathcal{G}} g_i^t\right\|^2 \\
            \notag&& =  \frac{1}{ G }\sum\limits_{i\in \mathcal{G}}\mathbb E\left\|g_i^{t} -  \nabla f_i(x_i^t)\right\|^2 +  \frac{1}{ G }\sum\limits_{i\in \mathcal{G}}\left\|\nabla f_i(x_i^t) - \frac{1}{ G }\sum\limits_{i\in \mathcal{G}} \nabla f_i(x_i^t)\right\|^2\\
            \notag&&\quad+ \mathbb E\left\|\frac{1}{ G }\sum\limits_{i\in \mathcal{G}} \nabla f_i(x_i^t) - \frac{1}{ G }\sum\limits_{i\in \mathcal{G}} g_i^t\right\|^2 - 2\mathbb E\left\|\frac{1}{ G }\sum\limits_{i\in \mathcal{G}} \nabla f_i(x_i^t) - \frac{1}{ G }\sum\limits_{i\in \mathcal{G}} g_i^t\right\|^2\\
            \notag&&\leqslant\frac{1}{ G }\sum\limits_{i\in \mathcal{G}}\mathbb E\left\|g_i^{t} -  \nabla f_i(x_i^t)\right\|^2 +  \frac{1}{ G }\sum\limits_{i\in \mathcal{G}}\left\|\nabla f_i(x_i^t) - \frac{1}{ G }\sum\limits_{i\in \mathcal{G}} \nabla f_i(x_i^t)\right\|^2\\
            \label{ldbn:ineq2}&&\overset{(i)}{\leqslant} \frac{1}{ G }\sum\limits_{i\in \mathcal{G}}\left\|\nabla f_i(x_i^t) - \frac{1}{ G }\sum\limits_{i\in \mathcal{G}} \nabla f_i(x_i^t)\right\|^2 + \sigma^2,
        \end{eqnarray}
        where (\textit{i}) was made according to Assumption \ref{as3}. To estimate the norm, we again use \eqref{bi:norm}:
        \begin{eqnarray*}
            &&~\frac{1}{ G }\sum\limits_{i\in \mathcal{G}}\left\|\nabla f_i(x_i^t) - \frac{1}{ G }\sum\limits_{i\in \mathcal{G}} \nabla f_i(x_i^t)\right\|^2 \\
            &&~=~ \frac{1}{ G }\sum\limits_{i\in \mathcal{G}}\left\|\nabla f_i(x_i^t) - \nabla f(\overline{x}^t)\right\|^2 + \frac{1}{ G }\sum\limits_{i\in \mathcal{G}}\left\|\nabla f(\overline{x}^t) - \frac{1}{ G }\sum\limits_{i\in \mathcal{G}} \nabla f_i(x_i^t)\right\|^2 \\
            & & \quad~~+ \frac{2}{ G }\sum\limits_{i\in \mathcal{G}} \left\langle \nabla f_i(x_i^t) - \nabla f(\overline{x}^t), \nabla f(\overline{x}^t) - \frac{1}{ G }\sum\limits_{i\in \mathcal{G}} \nabla f_i(x_i^t) \right\rangle\\
            &&~=~ \frac{1}{ G }\sum\limits_{i\in \mathcal{G}}\left\|\nabla f_i(x_i^t) - \nabla f(\overline{x}^t)\right\|^2 + \frac{1}{ G }\sum\limits_{i\in \mathcal{G}}\left\|\nabla f(\overline{x}^t) - \frac{1}{ G }\sum\limits_{i\in \mathcal{G}} \nabla f_i(x_i^t)\right\|^2\\
            & & \quad~~ - \frac{2}{ G }\sum\limits_{i\in \mathcal{G}}\left\|\nabla f(\overline{x}^t) - \frac{1}{ G }\sum\limits_{i\in \mathcal{G}} \nabla f_i(x_i^t)\right\|^2\\
            &&~\leqslant ~\frac{1}{ G }\sum\limits_{i\in \mathcal{G}}\left\|\nabla f_i(x_i^t) - \nabla f(\overline{x}^t)\right\|^2 \\
            &&\overset{\eqref{bi:CauchySchwarz}}{\leqslant} \frac{2}{ G }\sum\limits_{i\in \mathcal{G}}\left\|\nabla f_i(x_i^t) - \nabla f_i(\overline{x}^t)\right\|^2 + \frac{2}{ G }\sum\limits_{i\in \mathcal{G}}\left\|\nabla f_i(\overline{x}^t) - \nabla f(\overline{x}^t)\right\|^2\\
            &&~\overset{(i)}{\leqslant}~ \frac{2}{ G }\sum\limits_{i\in \mathcal{G}}\left\|\nabla f_i(x_i^t) - \nabla f_i(\overline{x}^t)\right\|^2 + 2(\delta_1 + \delta_2 \left\|\nabla f(\overline{x}^t)\right\|^2)\\
            &&\overset{\eqref{bi:lsmoothness}}{\leqslant} \frac{2}{ G }\sum\limits_{i\in \mathcal{G}} \left(2L\left(f_i(\overline{x}^t) - f_i(x_i^t) - \left\langle \nabla f_i(x_i^t), \overline{x}^t - x_i^t\right\rangle\right)\right) \\
            & &\quad~~+ 2(\delta_1 + \delta_2 \|\nabla f(\overline{x}^t)\|^2),
        \end{eqnarray*}
        where (\textit{i}) was made according to Assumption \ref{as4}.
        Combining it with \eqref{ldbn:ineq1} and \eqref{ldbn:ineq2},
        \begin{eqnarray*}
            \mathbb E V^{t+1} &\leqslant& \mathbb E V^{t} - \frac{2\gamma}{ G }\sum\limits_{i\in \mathcal{G}}\left\langle x_i^{t} - \overline{x}^{t}, \nabla f_i(x_i^{t})\right\rangle + \frac{4L \gamma^2}{ G }\sum\limits_{i\in \mathcal{G}}\left\langle x_i^{t} - \overline{x}^{t}, \nabla f_i(x_i^{t})\right\rangle \\
            & & + \frac{4L \gamma^2}{ G }\sum\limits_{i\in \mathcal{G}} \left(f_i(\overline{x}^t) - f_i(x_i^t)\right) + 2\gamma^2(\delta_1 + \delta_2 \|\nabla f(\overline{x}^t)\|^2)+ \gamma^2\sigma^2\\
            &=&  \mathbb E V^{t} - \frac{2\gamma(1 - 2L \gamma)}{ G }\sum\limits_{i\in \mathcal{G}} \left\langle x_i^{t} - \overline{x}^{t}, \nabla f_i(x_i^{t})\right\rangle + \frac{4L \gamma^2}{ G }\sum\limits_{i\in \mathcal{G}} \left(f_i(\overline{x}^t) - f_i(x_i^t)\right) \\
            & & + 2\gamma^2(\delta_1 + \delta_2 \|\nabla f(\overline{x}^t)\|^2)+ \gamma^2\sigma^2.
        \end{eqnarray*}
        Taking $\gamma \leqslant\frac{1}{4L}$ and applying  \eqref{bi:convex} to scalar product,
        \begin{eqnarray}
            \notag\mathbb E V^{t+1} &\leqslant& \mathbb E V^{t} + \frac{\gamma}{ G }\sum\limits_{i\in \mathcal{G}} \left(f_i(\overline{x}^{t}) - f_i(x_i^t)\right) + \frac{4L \gamma^2}{ G }\sum\limits_{i\in \mathcal{G}} \left(f_i(\overline{x}^t) - f_i(x_i^t)\right) \\
            \notag& & + 2\gamma^2(\delta_1 + \delta_2 \|\nabla f(\overline{x}^t)\|^2) + \gamma^2\sigma^2\\
            \notag&=&\mathbb E V^{t} + \frac{\gamma(1 + 4L \gamma)}{ G } \sum\limits_{i\in \mathcal{G}} \left(f_i(\overline{x}^t) - f_i(x_i^t)\right) + 2\gamma^2(\delta_1 + \delta_2 \|\nabla f(\overline{x}^t)\|^2) + \gamma^2\sigma^2\\
            \notag&\overset{\eqref{bi:lsmoothness}}{\leqslant}& \mathbb E V^{t} + \frac{\gamma(1 + 4L \gamma)}{ G } \sum\limits_{i\in \mathcal{G}} \left(\left\langle\nabla f_i(x_i^t), \overline{x}^t - x_i^t \right\rangle + \frac{L}{2} \|\overline{x}^t - x_i^t\|^2\right) \\
            \notag & & + 2\gamma^2(\delta_1 + \delta_2 \|\nabla f(\overline{x}^t)\|^2) + \gamma^2\sigma^2\\
            \notag&=& \mathbb E V^{t} + \frac{\gamma(1 + 4L \gamma)}{ G } \sum\limits_{i\in \mathcal{G}} \left(\underbrace{\left\langle\nabla f_i(x_i^t) - \nabla f_i(\overline{x}^t), \overline{x}^t - x_i^t \right\rangle}_{\leqslant 0 ~~\eqref{bi:convex}} + \left\langle\nabla f_i(\overline{x}^t), \overline{x}^t - x_i^t \right\rangle \right.\\
            \notag& & \left.+ \frac{L}{2} \|\overline{x}^t - x_i^t\|^2\right) + 2\gamma^2(\delta_1 + \delta_2 \|\nabla f(\overline{x}^t)\|^2) + \gamma^2\sigma^2\\
            \notag&\leqslant& \mathbb E V^{t} + \frac{\gamma(1 + 4L \gamma)}{ G } \sum\limits_{i\in \mathcal{G}} \left(\left\langle\nabla f_i(\overline{x}^t) - \nabla f(\overline{x}^t), \overline{x}^t - x_i^t \right\rangle + \left\langle\nabla f(\overline{x}^t), \overline{x}^t - x_i^t \right\rangle \right.\\
            \notag& & \left. + \frac{L}{2} \|\overline{x}^t - x_i^t\|^2\right) + 2\gamma^2(\delta_1 + \delta_2 \|\nabla f(\overline{x}^t)\|^2) + \gamma^2\sigma^2\\
            \notag&\overset{(i)}{\leqslant}& \mathbb E V^{t} + \frac{\gamma(1 + 4L \gamma)}{ G } \sum\limits_{i\in \mathcal{G}} \left(\frac{1}{2L}\|\nabla f_i(\overline{x}^t) - \nabla f(\overline{x}^t)\|^2 + L\|\overline{x}^t - x_i^t\|^2 \right.\\
            \label{ldbn:ineq4}& & \left. + \frac{L}{2}\|\overline{x}^t - x_i^t\|^2 + f(\overline{x}^t) - f(x_i^t)\right) + 2\gamma^2(\delta_1 + \delta_2 \|\nabla f(\overline{x}^t)\|^2) + \gamma^2\sigma^2,
        \end{eqnarray}
        where (\textit{i}) was made with applying \eqref{bi:young} to the first scalar product and applying \eqref{bi:lsmoothness} to second one. Using \eqref{bi:Jensen} we derive $\frac{1}{G}\sum\limits_{i\in\mathcal{G}}\left(f(\overline{x}^t) - f(x_i^t)\right) = \frac{1}{G}\sum\limits_{i\in\mathcal{G}}\left(f\left(\frac{1}{G}\sum\limits_{i\in\mathcal{G}} x_i^t\right) - f(x_i^t)\right) \leqslant \frac{1}{G}\sum\limits_{i\in\mathcal{G}}\left(\frac{1}{G}\sum\limits_{i\in\mathcal{G}} f(x_i^t) - f(x_i^t)\right) = 0$. In that way, we proceed \eqref{ldbn:ineq4} using this fact and taking expectation together with applying Assumption \ref{as4}:
        \begin{eqnarray*}
             \notag\mathbb E V^{t+1} &\leqslant& \mathbb E V^t + \frac{3L \gamma(1 + 4L \gamma)}{2}\mathbb E V^{t} + \frac{\gamma(1+4L \gamma)}{2L}\delta_2 \mathbb E\left\|\nabla f(\overline{x}^t)\right\|^2 +  \frac{\gamma(1+4L \gamma)}{2L}\delta_1 \\
             & &+ 2\gamma^2\left(\delta_1 + \delta_2 \mathbb E\left\|\nabla f(\overline{x}^t)\right\|^2\right) + \gamma^2\sigma^2 
        \end{eqnarray*}
        Going into recursion up to the past aggregation round, which was on the iteration $t_{k\cdot l}$ for some $k=\overline{0,\left\lfloor\nicefrac{T}{l}\right\rfloor}$, together with using $\gamma\leqslant\frac{1}{4L}$ choice, we obtain
        \begin{eqnarray*}
            \notag\mathbb E V^{t} & \leqslant & (1 + 3L \gamma)\notag\mathbb E V^{t-1} + \delta_2 \gamma \left(2\gamma + \frac{1}{L}\right)\mathbb E\left\|\nabla f(\overline{x}^{t-1})\right\|^2 + \delta_1 \gamma \left(2\gamma + \frac{1}{L}\right) +  \gamma^2\sigma^2  \\
            &\leqslant& (1 + 3L \gamma)^{t-t_{k\cdot l}}\notag\mathbb E V^{t_{k\cdot l}} + \delta_2 \gamma \left(2\gamma + \frac{1}{L}\right)\sum\limits_{j=t_{k\cdot l}}^{t-1}(1 + 3L \gamma)^{j-t_{k\cdot l}}\mathbb E\left\|\nabla f(\overline{x}^j)\right\|^2 \\
            & &+ \left(\delta_1 \gamma \left(2\gamma + \frac{1}{L}\right) +  \gamma^2\sigma^2\right)\sum\limits_{j=t_{k\cdot l}}^{t-1} \left(1+3L \gamma\right)^{j-t_{k\cdot l}} \\
            &\leqslant& \delta_2\gamma \left(2\gamma + \frac{1}{L}\right)(1+3L \gamma)^{l-1}\sum\limits_{j=t_{k\cdot l}}^{t-1}\mathbb E\left\|\nabla f(\overline{x}^j)\right\|^2 \\
            & & + \left(\delta_1 \gamma \left(2\gamma + \frac{1}{L}\right) + \gamma^2\sigma^2\right)(l-1)(1 + 3L \gamma)^{l-1}.
        \end{eqnarray*}
        Now we tune $\gamma\leqslant\frac{1}{4(l-1)L}$. Note it is smallest of all previous $\gamma$, since $l \geqslant 2$ and consequently all previous transitions hold true. In that way, using $\left(1 + 3L \gamma\right)^{l-1} \leqslant \left(1 + \frac{3}{4(l-1)}\right)^{l-1} \leqslant \left(1 + \frac{1}{l-1}\right)^{l-1} \leqslant 3$,
        \begin{eqnarray*}
            \mathbb E V^{t} &\leqslant& 3\delta_2\gamma \left(\frac{1}{2(l-1)L} + \frac{1}{L}\right)\sum\limits_{j=t_{k\cdot l}}^{t-1}\mathbb E\left\|\nabla f(\overline{x}^j)\right\|^2 \\
            & & + 3\left(\delta_1 \gamma \left(\frac{1}{2(l-1)L} + \frac{1}{L}\right) + \gamma^2\sigma^2\right)(l-1)\\
            &\leqslant& \frac{9\delta_2\gamma}{2L}\sum\limits_{j=t_{k\cdot l}}^{t-1}\mathbb E\left\|\nabla f(\overline{x}^j)\right\|^2 + \frac{9\delta_1\gamma}{2L}(l-1) + 3\gamma^2\sigma^2(l-1).
        \end{eqnarray*}
        This proves the second statement and ends the proof of the lemma.
    \end{proof}
\end{lemma}

Now we move to descent lemma in the local setup.

\begin{lemma}\label{DescentLemmaLocalTFM}
    Under Assumptions \ref{as1}, \ref{as2nonconvex}, \ref{as3}, \ref{as4}, \ref{as5}, at each iteration $t$ of \text{Algorithm \ref{alg:local_tfm}}, the following estimation is valid:
    \begin{align*}
        \notag\mathbb E\hat{f}\left(\overline{x}^{t+1}\right) \leqslant& \mathbb E \hat{f}\left(\overline{x}^{t}\right) - \frac{\gamma}{2}\left(\left(1-4\delta_2\right) - 3L\gamma(1 + 2\delta_2)\right) \mathbb E\left\|\nabla f\left(\overline{x}^t\right)\right\|^2 \\
        &+ L^2\gamma(1+3L \gamma)\frac{1}{ G }\sum\limits_{i\in \mathcal{G}}\mathbb E\left\|x_i^t - \overline{x}^t\right\|^2 + \frac{3L\gamma^2}{2 G }\sigma^2 + \gamma(2 + 3L\gamma)\delta_1 + \gamma\zeta(N).
    \end{align*}

\begin{proof}
To begin with, we consider iterations, when Algorithm \ref{alg:local_tfm} performs aggregations, i.e. $t = t_{k\cdot l}$ for some $k=\overline{0,\left\lfloor\nicefrac{T}{l}\right\rfloor}$. The update formula for such iterations is given by
\begin{align*} 
x^{t+1} = x_i^{t+1} = \sum\limits_{i=1}^n \left(\arg\min_{\omega \in \Delta_1^n} \hat{f}\left[\sum_{i = 1}^n \omega_i\left(x_i^t - \gamma g_i^t\right)\right]\right) \left(x_i^t - \gamma g_i^t\right),
\end{align*} 
which leads to an upper bound on $\hat{f}(x^{t+1})$:
\begin{align*}
    \hat{f}(x^{t+1}) ~ & ~ ~ \leqslant ~ \min_{\omega \in \Delta_1^n} \hat{f}\left[\sum_{i = 1}^n \omega_i\left(x^t - \gamma g_i^t\right)\right].
\end{align*}
Using this estimate and additional notation $\overline{x}^t = \frac{1}{ G }\sum\limits_{i\in \mathcal{G}} x_i^t$ we proceed to the average per honest devices during this local round point estimate:
\begin{eqnarray} 
\notag\hat{f}\left(\overline{x}^{t+1}\right) & = & \hat{f}(
x^{t+1}) \leqslant ~ \min_{\omega \in \Delta_1^n} \hat{f}\left[\sum_{i = 1}^n \omega_i\left(x_i^t - \gamma g_i^t\right)\right] \\
\label{lemmalocaltfm:ineq1} &\leqslant & \hat{f}\left[\frac{1}{ G }\sum\limits_{i\in \mathcal{G}} x_i^t - \frac{\gamma}{ G } \sum_{i \in \mathcal{G}} g_i^t\right] \\
\notag& \overset{\eqref{bi:lsmoothness}}{\leqslant} & \hat{f}\left(\overline{x}^t\right) - \left\langle \nabla\hat{f}\left(\overline{x}^t\right), \frac{\gamma}{ G } \sum_{i \in \mathcal{G}} g_i^t \right\rangle + \frac{L \gamma^2}{2} \left\|\frac{1}{ G }\sum_{i \in \mathcal{G}} g_i^t\right\|^2.
\end{eqnarray}
Taking expectation, we obtain
\begin{eqnarray*}
\mathbb E \hat{f}\left(\overline{x}^{t+1}\right) & \leqslant & \mathbb E \hat{f}\left(\overline{x}^{t}\right) - \left\langle \nabla \hat{f}\left(\overline{x}^t\right), \frac{\gamma}{ G } \sum_{i \in \mathcal{G}} \nabla f_i(x_i^t) \right\rangle + \frac{L \gamma^2}{2} \mathbb E\left\|\frac{1}{ G }\sum_{i \in \mathcal{G}} g_i^t\right\|^2\\
\notag&=& \mathbb E \hat{f}\left(\overline{x}^{t}\right) - \left\langle \nabla f\left(\overline{x}^t\right), \frac{\gamma}{ G } \sum_{i \in \mathcal{G}} \nabla f_i(x_i^t) \right\rangle \\
\notag& & - \left\langle \nabla f\left(\overline{x}^t\right) - \nabla \hat{f}\left(\overline{x}^t\right), \frac{\gamma}{ G } \sum_{i \in \mathcal{G}} \nabla f_i(x_i^t) \right\rangle + \frac{L \gamma^2}{2} \mathbb E \left\|\frac{1}{ G }\sum_{i \in \mathcal{G}} g_i^t\right\|^2\\
\notag& \overset{\eqref{bi:norm}, \eqref{bi:young}}{\leqslant} & \mathbb E \hat{f}\left(\overline{x}^{t}\right) - \frac{\gamma}{2}\left\|\nabla f\left(\overline{x}^t\right)\right\|^2 - \frac{\gamma}{2}\left\|\frac{1}{ G } \sum_{i \in \mathcal{G}} \nabla f_i(x_i^t) \right\|^2 \\
\notag & & + \frac{\gamma}{2}\left\|\nabla f\left(\overline{x}^t\right) - \frac{1}{ G } \sum_{i \in \mathcal{G}} \nabla f_i(x_i^t)\right\|^2 + \frac{\gamma}{2}\left\|\frac{1}{ G } \sum_{i \in \mathcal{G}} \nabla f_i(x_i^t) \right\|^2 \\
\notag& & + \frac{\gamma}{2}\left\|\nabla\hat{f}\left(\overline{x}^t\right) - \nabla f\left(\overline{x}^t\right)\right\|^2 + \frac{L \gamma^2}{2} \mathbb E\left\|\frac{1}{ G }\sum_{i \in \mathcal{G}} g_i^t\right\|^2 \\
&\overset{\eqref{bi:CauchySchwarz}}{\leqslant}& \mathbb E \hat{f}\left(\overline{x}^{t}\right) - \frac{\gamma}{2}\left\|\nabla f\left(\overline{x}^t\right)\right\|^2 + \frac{\gamma}{2}\left\|\nabla f\left(\overline{x}^t\right) - \frac{1}{ G } \sum_{i \in \mathcal{G}} \nabla f_i(x_i^t)\right\|^2 \\ 
\notag& & + \gamma\left\|\nabla\hat{f}\left(\overline{x}^t\right) - \nabla f_1\left(\overline{x}^t\right)\right\|^2 + \gamma\left\|\nabla f_1\left(\overline{x}^t\right) - \nabla f\left(\overline{x}^t\right)\right\|^2 \\
\notag & & + \frac{L \gamma^2}{2} \mathbb E\left\|\frac{1}{ G }\sum_{i \in \mathcal{G}} g_i^t\right\|^2\\
&\overset{(\text{Lemma~} \ref{lemmaShalevShwartz})}{\leqslant}& \mathbb E \hat{f}\left(\overline{x}^{t}\right) - \frac{\gamma}{2}\left\|\nabla f\left(\overline{x}^t\right)\right\|^2 + \frac{\gamma}{2}\left\|\nabla f\left(\overline{x}^t\right) - \frac{1}{ G } \sum_{i \in \mathcal{G}} \nabla f_i(x_i^t)\right\|^2 \\ 
\notag& & + \gamma\left\|\nabla f_1\left(\overline{x}^t\right) - \nabla f\left(\overline{x}^t\right)\right\|^2 + \frac{L \gamma^2}{2} \mathbb E\left\|\frac{1}{ G }\sum_{i \in \mathcal{G}} g_i^t\right\|^2 + \gamma\zeta(N) \\
\notag&\overset{\eqref{bi:CauchySchwarz}}{\leqslant}& \mathbb E \hat{f}\left(\overline{x}^{t}\right) - \frac{\gamma}{2}\left\|\nabla f\left(\overline{x}^t\right)\right\|^2 + \frac{\gamma}{2}\left\|\nabla f\left(\overline{x}^t\right) - \frac{1}{ G } \sum_{i \in \mathcal{G}} \nabla f_i(x_i^t)\right\|^2 \\
&& + \gamma\left\|\nabla f_1\left(\overline{x}^t\right) - \nabla f\left(\overline{x}^t\right)\right\|^2 + \frac{3L \gamma^2}{2}\mathbb E \left\|\frac{1}{ G }\sum_{i \in \mathcal{G}} g_i^t - \frac{1}{ G }\sum_{i \in \mathcal{G}} \nabla f_i(x_i^t)\right\|^2 \\
&& + \frac{3L \gamma^2}{2}\mathbb E \left\|\frac{1}{ G }\sum_{i \in \mathcal{G}} \nabla f_i(x_i^t) -\nabla f(\overline{x}^t)\right\|^2 + \frac{3L \gamma^2}{2}\mathbb E\left\|\nabla f(\overline{x}^t)\right\|^2+ \gamma\zeta(N).
\end{eqnarray*} 
Now we use Assumption \ref{as3} with the fact that $\mathbb{E} g_i^t = \nabla f_i(x^t)$ and $\mathbb{E} \langle \nabla f_i (x^t) - g_i^t, \nabla f_j (x^t) - g_j^t \rangle = 0$ to bound $\mathbb E \left\|\frac{1}{ G }\sum_{i \in \mathcal{G}} \left(g_i^t - \nabla f_i(x_i^t)\right)\right\|^2\leqslant\frac{\sigma^2}{ G }$. Taking expectation again, we move to
\begin{eqnarray*}
    \mathbb E\hat{f}\left(\overline{x}^{t+1}\right) &\leqslant& \mathbb E \hat{f}\left(\overline{x}^{t}\right) - \frac{\gamma}{2}\left(1 - 3L\gamma\right)\mathbb E\left\|\nabla f\left(\overline{x}^t\right)\right\|^2 \\
    & & + \frac{\gamma}{2}(1 + 3L\gamma) \mathbb E \left\|\frac{1}{ G }\sum_{i \in \mathcal{G}} \nabla f_i(x_i^t) - \nabla f(\overline{x}^t)\right\|^2\\
    & & + \gamma\mathbb E\left\|\nabla f_1\left(\overline{x}^t\right) - \nabla f\left(\overline{x}^t\right)\right\|^2 + \frac{3L\gamma^2}{2 G }\sigma^2 + \gamma\zeta(N)\\
    &\overset{\eqref{bi:CauchySchwarz}}{\leqslant}& \mathbb E \hat{f}\left(\overline{x}^{t}\right) - \frac{\gamma}{2}\left(1 - 3L\gamma\right)\mathbb E\left\|\nabla f\left(\overline{x}^t\right)\right\|^2 \\
    & & + \gamma(1 + 3L\gamma) \mathbb E \left\|\frac{1}{ G }\sum_{i \in \mathcal{G}} \nabla f_i(x_i^t) - \frac{1}{ G }\sum_{i \in \mathcal{G}} \nabla f_i(\overline{x}^t)\right\|^2\\
    & & + \gamma(1 + 3L\gamma) \mathbb E \left\|\frac{1}{ G }\sum_{i \in \mathcal{G}} \nabla f_i(\overline{x}^t) - \nabla f(\overline{x}^t)\right\|^2 + \gamma\mathbb E\left\|\nabla f_1\left(\overline{x}^t\right) - \nabla f\left(\overline{x}^t\right)\right\|^2 \\
    & & + \frac{3L\gamma^2}{2 G }\sigma^2 + \gamma\zeta(N)\\
    & \overset{\eqref{bi:CauchySchwarz}}{\leqslant}& \mathbb E \hat{f}\left(\overline{x}^{t}\right) -\frac{\gamma}{2}\left(1 - 3L\gamma\right)\mathbb E\left\|\nabla f\left(\overline{x}^t\right)\right\|^2 \\
    & & + \frac{\gamma(1 + 3L\gamma)}{ G }\sum_{i \in \mathcal{G}}\mathbb E \left\|\nabla f_i(x_i^t) - \nabla f_i(\overline{x}^t)\right\|^2 \\
    & & + \frac{\gamma(1 + 3L\gamma)}{ G }\sum_{i \in \mathcal{G}}\mathbb E \left\|\nabla f_i(\overline{x}^t) - \nabla f(\overline{x}^t)\right\|^2 + \gamma\mathbb E\left\|\nabla f_1\left(\overline{x}^t\right) - \nabla f\left(\overline{x}^t\right)\right\|^2 \\
    & & + \frac{3L\gamma^2}{2 G }\sigma^2 + \gamma\zeta(N).
\end{eqnarray*}
Using Assumption \ref{as4} to bound $\left\|\nabla f_i(\overline{x}^t) - \nabla f(\overline{x}^t)\right\|^2 \leqslant\delta_1 + \delta_2\left\|\nabla f(\overline{x}^t)\right\|^2$ for all $i\in\mathcal{G}$ we get
\begin{eqnarray}
    \notag\mathbb E\hat{f}\left(\overline{x}^{t+1}\right) &\leqslant& \mathbb E \hat{f}\left(\overline{x}^{t}\right) -\frac{\gamma}{2}\left(\left(1-4\delta_2\right) - 3L\gamma(1 + 2\delta_2)\right) \mathbb E\left\|\nabla f\left(\overline{x}^t\right)\right\|^2 \\
    \notag&&+ \frac{\gamma(1 + 3L\gamma)}{ G }\sum_{i \in \mathcal{G}}\mathbb E \left\|\nabla f_i(x_i^t) - \nabla f_i(\overline{x}^t)\right\|^2 + \frac{3L\gamma^2}{2 G }\sigma^2 \\
    \notag &&+ \gamma(2 + 3L\gamma)\delta_1 + \gamma\zeta(N)\\
    \notag&\overset{\eqref{bi:lsmoothness}}{\leqslant}& \mathbb E \hat{f}\left(\overline{x}^{t}\right) -\frac{\gamma}{2}\left(\left(1-4\delta_2\right) - 3L\gamma(1 + 2\delta_2)\right) \mathbb E\left\|\nabla f\left(\overline{x}^t\right)\right\|^2 \\
    \notag& & + L^2\gamma(1+3L \gamma)\frac{1}{ G }\sum\limits_{i\in \mathcal{G}}\mathbb E\left\|x_i^t - \overline{x}^t\right\|^2\\
    \label{lemmalocaltfm:ineq2}& & + \frac{3L\gamma^2}{2 G }\sigma^2 + \gamma(2 + 3L\gamma)\delta_1 + \gamma\zeta(N).
\end{eqnarray}
Now we want to give the estimate for iterations $t\neq t_{k\cdot l}$. The update rule combined with \eqref{bi:lsmoothness} gives
\begin{eqnarray*}
    \hat{f}(\overline{x}^{t+1}) &=& \hat{f}\left[\frac{1}{ G }\sum\limits_{i=1}^n x_i^t - \frac{\gamma}{ G }\sum\limits_{i=1}^n g_i^t\right] \leqslant \hat{f}\left[\frac{1}{ G }\sum\limits_{i=1}^n x_i^t - \frac{\gamma}{ G }\sum\limits_{i=1}^n g_i^t\right].
\end{eqnarray*}
Mention, this estimate is coincide with the \eqref{lemmalocaltfm:ineq1}. Thus, proceeding analogically to the $t = t_{k\cdot l}$ case, we obtain \eqref{lemmalocaltfm:ineq2}. In that way, \eqref{lemmalocaltfm:ineq2} delivers the result of the lemma.
\end{proof}
\end{lemma}

Now we pass to the main theorem of this section.

\begin{theorem}\label{TheoremLocalTFM}
    Under Assumptions \ref{as1}, \ref{as2convex}, \ref{as3}, \ref{as4} with $\delta_2 \leqslant \frac{1}{8}$, \ref{as5}, for solving the problem \eqref{eq:setting_dist}, for Algorithm \ref{alg:local_tfm} with $\gamma \leqslant \frac{1}{25(l-1)L}$, the following holds:
    \begin{eqnarray*}
        \frac{1}{T}\sum\limits_{t=0}^{T-1}\mathbb E\left\|\nabla f\left(\overline{x}^t\right)\right\|^2 &\leqslant& \frac{5\mathbb E\left[\hat{f}\left(x^{0}\right) - \hat{f}\left(\hat{x}^{*}\right)\right]}{\gamma T} + 20L^2\gamma^2(l-1)\sigma^2 + \frac{8 L\gamma}{ G }\sigma^2 \\
        & &+ 13\delta_1 + 30 L\gamma(l-1)\delta_1 + 5\zeta(N).
    \end{eqnarray*}
    \begin{proof}
        To begin with, we combine the result of Lemma \ref{DescentLemmaLocalTFM} with result of Lemma \ref{lemma_differences_between_nodes} to obtain
        \begin{eqnarray*}
            \notag\mathbb E\hat{f}\left(\overline{x}^{t+1}\right) &\leqslant& \mathbb E \hat{f}\left(\overline{x}^{t}\right) -\frac{\gamma}{2}\left(\left(1-4\delta_2\right) - 3L\gamma(1 + 2\delta_2)\right) \mathbb E\left\|\nabla f\left(\overline{x}^t\right)\right\|^2 \\
            & & + \frac{9\delta_2 L\gamma^2(1+3L \gamma)}{2}\sum\limits_{j=t_{k\cdot l}}^{t-1}\mathbb E\left\|\nabla f(\overline{x}^j)\right\|^2 + \frac{9\delta_1 L\gamma^2(1+3L \gamma)(l-1)}{2}\\
            & & + 3L^2\gamma^3(1+3L \gamma)\sigma^2 (l-1) + \frac{3L\gamma^2}{2 G }\sigma^2 \\
            & & + \gamma(2 + 3L\gamma)\delta_1 + \gamma\zeta(N).
        \end{eqnarray*}
        Summing over all iterations and using $\sum\limits_{t=0}^{T-1}\sum\limits_{j=t_{k\cdot l}}^{t-1}\mathbb E\left\|\nabla f(\overline{x}^j)\right\|^2 \leqslant (l-1) \sum\limits_{t=0}^{T-1} \mathbb E\left\|\nabla f(\overline{x}^t)\right\|^2$,
        \begin{eqnarray*}
            & & E\left[\hat{f}\left(x^{T}\right) - \hat{f}\left(x^{0}\right)\right] \\
            &&\leqslant -\frac{\gamma}{2}\left(\left(1-4\delta_2\right) - 3L\gamma(1 + 2\delta_2)\right) \sum\limits_{t=0}^{T-1}\mathbb E\left\|\nabla f\left(\overline{x}^t\right)\right\|^2 \\
            & & \quad+ \frac{9\delta_2 L\gamma^2(1+3L \gamma)(l-1)}{2}\sum\limits_{t=0}^{T-1}\mathbb E\left\|\nabla f(\overline{x}^t)\right\|^2 + \frac{9\delta_1 L\gamma^2(1+3L \gamma)(l-1)T}{2}\\
            & & \quad+ 3L^2\gamma^3(1+3L \gamma)\sigma^2 (l-1) T + \frac{3L\gamma^2 T}{2 G }\sigma^2 + \gamma(2 + 3L\gamma)\delta_1 T + \gamma\zeta(N) T \\
            &&= -\frac{\gamma}{2}\left(1-4\delta_2 - 3L\gamma(1 + 2\delta_2) - 9L\gamma\delta_2(1+3L \gamma)(l-1)\right) \sum\limits_{t=0}^{T-1}\mathbb E\left\|\nabla f\left(\overline{x}^t\right)\right\|^2\\
            & & \quad+ 3L^2\gamma^3(1+3L \gamma)\sigma^2 (l-1) T + \frac{3L\gamma^2 T}{2 G }\sigma^2 + \frac{9\delta_1 L\gamma^2(1+3L \gamma)(l-1)T}{2} \\
            & & \quad+ \gamma(2+3L \gamma)\delta_1 T + \gamma\zeta(N) T.
        \end{eqnarray*}
        Now we want to estimate the coefficient before the $\sum\limits_{t=0}^{T-1}\mathbb E\left\|\nabla f\left(\overline{x}^t\right)\right\|^2$ term. Let us take $\delta_2 \leqslant \frac{1}{8}$ and $\gamma\leqslant\frac{1}{25(l-1)L}$ (it is the smallest $\gamma$ from all we choose before, thus, all previous transitions holds true). Thus,
        \begin{eqnarray*}
            1-4\delta_2 - 3L\gamma(1 + 2\delta_2) - 9L\gamma\delta_2(1+3L \gamma)(l-1)&\geqslant& \frac{1}{2} - \frac{15}{4}L\gamma \\
            & & -\frac{9}{8}L\gamma\left(1 + \frac{3}{25(l-1)}\right)(l-1)\\
            &\overset{l\geqslant 2}{\geqslant}& \frac{1}{2} - \frac{15}{4}L\gamma -\frac{63}{50}L\gamma(l-1)\\
            &\geqslant& \frac{1}{2} - \frac{3}{20} - \frac{63}{1250} \geqslant \frac{1}{5}
        \end{eqnarray*}
        In that way,
        \begin{eqnarray*}
            \frac{\gamma}{5}\sum\limits_{t=0}^{T-1}\mathbb E\left\|\nabla f\left(\overline{x}^t\right)\right\|^2 &\leqslant& \mathbb E\left[\hat{f}\left(x^{0}\right) - \hat{f}\left(\hat{x}^{*}\right)\right] + 4L^2\gamma^3(l-1)\sigma^2 T + \frac{3L\gamma^2 T}{2 G }\sigma^2 \\
            & & + \frac{5}{2}\gamma\delta_1 T + 6L\gamma^2\delta_1(l-1)T + \gamma\zeta(N),\\
            \frac{1}{T}\sum\limits_{t=0}^{T-1}\mathbb E\left\|\nabla f\left(\overline{x}^t\right)\right\|^2 &\leqslant& \frac{5\mathbb E\left[\hat{f}\left(x^{0}\right) - \hat{f}\left(\hat{x}^{*}\right)\right]}{\gamma T} + 20L^2\gamma^2(l-1)\sigma^2 + \frac{8 L\gamma}{ G }\sigma^2 \\
            & &+ 13\delta_1 + 30 L\gamma(l-1)\delta_1 + 5\zeta(N),
        \end{eqnarray*}
        that ends the proof of the theorem.
    \end{proof}
\end{theorem}
\begin{remark}
    In this remark we want to explain why $\frac{1}{T}\sum\limits_{t=0}^{T-1}\mathbb{E}\left\|\nabla f\left(\frac{1}{ G }\sum\limits_{i\in \mathcal{G}} x_i^t\right)\right\|^2$, that we choose as a criterion in Theorem \ref{TheoremLocalTFM} is consistent. When we consider iterations, when Algorithm \ref{alg:local_tfm} performs aggregation, we assign weights of honest devices, i.e. devices, which send honest stochastic gradient $g$ with $\mathbb E[g] = \nabla f(x)$ equal to $\frac{1}{ G }$, and weights of other devices (who acts as Byzantine) is equal to $0$. And this layout gives an estimate that is an upper bound of true $\omega$ realization \eqref{lemmalocaltfm:ineq1}. In other words, we say, that if algorithm at each aggregation round take average of points from devices, which all previous local round act like honest it would not be better than for convergence, than real iteration of the algorithm. In such a way, we a not interesting in the points of devices, who perform even one Byzantine-like iteration in the local round. Now it is clear why we have to weaken the assumption about at least one honest device at each iteration, not necessary the same: we request at least one honest device at each local round, not necessary the same in different rounds.
\end{remark}

\begin{corollary}\label{ITRSGDlocalcorollary}
    Under the assumptions of \text{Theorem} \ref{TheoremLocalTFM}, for solving the problem \eqref{eq:setting_dist}, after $T$ iterations of Algorithm \ref{alg:local_tfm} with $\gamma \leqslant \min\left\{\frac{1}{25(l-1)L}, \frac{\sqrt{5\mathbb E\left[\hat{f}\left(x^{0}\right) - \hat{f}\left(\hat{x}^{*}\right)\right] G}}{\sigma\sqrt{9LT}}\right\}$, the following holds:
\begin{align*}
        \frac{1}{T}\sum\limits_{t=0}^{T-1}\mathbb{E} \left\|\nabla f(\overline{x}^t)\right\|^2
        =&  \mathcal{O}\left(\frac{\mathbb E\left[\hat{f}\left(x^{0}\right) - \hat{f}\left(\hat{x}^{*}\right)\right] lL}{T} + \frac{\sigma\sqrt{\mathbb E\left[\hat{f}\left(x^{0}\right) - \hat{f}\left(\hat{x}^{*}\right)\right] L}}{\sqrt{T G }} + \delta_1 + \zeta(N)\right).
    \end{align*}
\end{corollary}
We note that we do not provide the version for the first of our algorithms, \texttt{Bant} in this and the following sections. The reason is that its examination completely replicates the implementation of the technique discussed in this section for the proofs presented in Section \ref{C}. We consider it unnecessary to repeat this procedure; however, we are confident in the practical applicability of the technique to both of our methods: \texttt{Bant} and \texttt{AutoBant}.
\newpage
\section{Partial participation} \label{sec:partial_appendix}

In the previous section, we discussed an important regime in distributed systems known as the local approach. Another widely used scenario in distributed learning is partial participation, which can be advantageous in various practical setups \citep{yang2021achieving, wang2022unified, li2022federated}. In this section, we extend our \texttt{AutoBant} algorithm to support partial participation in federated learning. Below, we provide a formal description of the \texttt{AutoBant} with \textsc{Partial Participation} method (Algorithm \ref{alg:partial_participation}).

\begin{algorithm}[ht]
\caption{\texttt{AutoBant} with \textsc{Partial Participation
}}
\label{alg:partial_participation}
\begin{algorithmic}[1]
\STATE \textbf{Input:} Starting point $x^0 \in \mathbb{R}^d$
\STATE \textbf{Parameters:} Stepsize $\gamma > 0$, error accuracy $\delta$
\FOR{$t = 0, 1, 2, \ldots, T-1$}
    \STATE Define a set of active workers $\mathcal{W}(t) = \mathcal{G}(t) \cup \mathcal{B}(t)$;~ $n(t) = |\mathcal{W}(t)|$
    \STATE Server sends $x^t$ to each worker from $\mathcal{W}(t)$
    \FORALL{workers $i \in \mathcal{W}(t)$ in parallel}
        \STATE Generate $\xi_i^t$ independently
        \STATE Compute stochastic gradient $g_i^t = g_i(x^t, \xi_i^t)$
        \STATE Send $g_i^t$ to server
    \ENDFOR
    \STATE $\omega^{t} \approx \arg\underset{\omega\in\Delta_1^{n(t)}}{\min} \hat{f}\left(x^t - \gamma\sum\limits_{i \in \mathcal{W}(t)} \omega_i g_i^t\right)$ 
    \STATE $x^{t+1} = x^t - \gamma\sum\limits_{i \in \mathcal{W}(t)} \omega_i^{t} g_i^t$ 
\ENDFOR
\STATE \textbf{Output:} $\frac{1}{T}\sum\limits_{t = 0}^{T-1} x^t$
\end{algorithmic}
\end{algorithm}

The key change is that at each iteration we have a set $\mathcal{W}(t)$, which is a subset of the full list of accessible devices. In that way, we impose a stricter assumption compared to the regular training regime, namely, we require the presence of at least one honest device (including server) in each of the $\mathcal{W}(t)$ sets.

Now we present the main theorem of this section.

\begin{theorem} \label{theorempartial}
    Under Assumptions \ref{as1}, \ref{as2nonconvex}, \ref{as3}, \ref{as4} with $\delta_2 \leqslant 0.25$, \ref{as5}, for solving the problem \eqref{eq:setting_dist}, after $T$ iteration of \text{Algorithm \ref{alg1}} in partial participation scenario with $\gamma \leqslant \frac{1}{15L}$, it implies
    \begin{eqnarray*}
        \frac{1}{T}\sum\limits_{t = 0}^{T-1} \mathbb E \|\nabla f(x^t)\|^2&\leqslant &\frac{4\mathbb E\left[\hat{f}\left(x^{0}\right) - \hat{f}\left(\hat{x}^{*}\right)\right]}{\gamma T} + 3\delta_1 + \frac{6L\gamma}{\widetilde{G}}\sigma^2 + 4\zeta(N),
    \end{eqnarray*}
    where $\widetilde{G} = \underset{t\leqslant T}{\min}  G(t)$.
\begin{proof}
    The iterative update formula for $x^{t+1}$ is given by
    \begin{eqnarray*} 
    x^{t+1} &=& x^t - \gamma \sum_{i\in\mathcal{W}(t)} \left(\arg\min_{\omega \in \Delta_1^n} \hat{f}\left[x^t - \gamma \sum_{i\in\mathcal{W}(t)} \omega_i g_i^t\right]\right) g_i^t,
    \end{eqnarray*} 
    which leads to an upper bound on $\hat{f}(x^{t+1})$:
    \begin{eqnarray*} 
    \hat{f}(x^{t+1}) & = & \min_{\omega \in \Delta_1^n} \hat{f}\left[x^t - \gamma \sum_{i\in\mathcal{W}(t)} \omega_i g_i^t\right] \\
    & \leqslant & \hat{f}\left[x^t - \frac{\gamma}{ G(t) } \sum_{i \in \mathcal{G}(t)} g_i^t\right]  \\
    & \overset{\eqref{bi:lsmoothness}}{\leqslant} &\hat{f}(x^t) - \left\langle \nabla\hat{f}(x^t), \frac{\gamma}{ G(t) } \sum_{i \in \mathcal{G}(t)} g_i^t \right\rangle + \frac{L \gamma^2}{2} \left\|\frac{1}{ G(t) }\sum_{i \in \mathcal{G}(t)} g_i^t\right\|^2.
    \end{eqnarray*}
    Taking the expectation,
    \begin{eqnarray*}
    \mathbb{E} \hat{f}(x^{t+1})&\leqslant &\mathbb{E}  \hat{f}(x^t) - \left\langle \nabla\hat{f}(x^t), \frac{\gamma}{ G(t) } \sum_{i \in \mathcal{G}(t)} \nabla f_i(x^t) \right\rangle  \\
    & & + \frac{L \gamma^2}{2} \mathbb{E} \left\|\frac{1}{ G(t) }\sum_{i \in \mathcal{G}(t)} g_i^t\right\|^2\\ 
        &\overset{(\text{Lemma } \ref{lemma2})}{\leqslant} & \mathbb{E}\hat{f}(x^t) + \gamma\zeta(N) - \frac{\gamma}{2}\left\|\nabla f(x) \right\|^2 \\
    &&+ \frac{3\gamma}{2} \left(\delta_1 + \delta_2 \|\nabla f(x^t)\|^2\right) +  \frac{L\gamma^2}{2}\mathbb{E}\left\| \frac{1}{ G(t) }\sum_{i \in \mathcal{G}(t)} g_i^t\right\|^2 \\
        &\overset{\eqref{bi:CauchySchwarz}}{\leqslant} & \mathbb{E}\hat{f}(x^t) + \gamma\zeta(N) - \frac{\gamma}{2}\left\|\nabla f(x)\right\|^2 + \frac{3\gamma}{2} \left(\delta_1 + \delta_2 \|\nabla f(x^t)\|^2\right)\\
    &&+  \frac{3L \gamma ^2}{2}\left\|\frac{1}{ G(t) }\sum_{i \in \mathcal{G}(t)}  (\nabla f(x^t) - \nabla f_i (x^t))\right\|^2 \\
    & & + \frac{3L \gamma ^2}{2}\mathbb{E}\left\|\frac{1}{ G(t) }\sum_{i \in \mathcal{G}(t)}  (\nabla f_i (x^t) - g_i^t)\right\|^2 \\
    && +\frac{3L \gamma ^2}{2} \left\|\frac{1}{ G(t) }\sum_{i \in \mathcal{G}(t)} \nabla f(x^t) \right\|^2.
    \end{eqnarray*}
    Due to the fact that $\mathbb{E} g_i^t = \nabla f_i(x^t)$ and $\mathbb{E} \langle \nabla f_i (x^t) - g_i^t, \nabla f_j (x^t) - g_j^t \rangle = 0$,
    \begin{eqnarray*}
         \mathbb{E} \hat{f}(x^{t+1})   & \overset{\eqref{bi:CauchySchwarz}}{\leqslant}& \mathbb{E}\hat{f}(x^t) + \gamma\zeta(N) - \frac{\gamma}{2}\left\|\nabla f(x) \right\|^2 + \frac{3\gamma}{2} \left(\delta_1 + \delta_2 \|\nabla f(x^t)\|^2\right) \\
    &&+  \frac{3L \gamma ^2}{2} \left(\frac{1}{ G(t) }\sum_{i \in \mathcal{G}(t)} \left\|\nabla f(x^t) - f_i(x^t)\right\|^2 \right.\\
    & & \left. + \frac{1}{ (G(t)) ^2}\sum_{i \in \mathcal{G}(t)} \mathbb{E}\left\|\nabla f_i(x^t) - g_i^t\right\|^2\right) + \frac{3L \gamma ^2}{2} \left\|\nabla f(x^t) \right\|^2   \\
    & \overset{(\text{Ass. }\ref{as3}, \ref{as4})}{\leqslant}& \mathbb{E}\hat{f}(x^t) + \gamma\zeta(N) - \frac{\gamma}{2}\left\|\nabla f(x) \right\|^2 + \frac{3\gamma}{2} \left(\delta_1 + \delta_2 \|\nabla f(x^t)\|^2\right) \\
    &&+  \frac{3L \gamma ^2}{2}\left(\delta_1 + \delta_2 \|\nabla f(x^t)\|^2 +  \frac{\sigma ^2}{ G(t) }\right) + \frac{3L \gamma ^2}{2} \left\|\nabla f(x^t) \right\|^2  \\
        &=& \mathbb E [\hat{f}(x^{t})] - \frac{\gamma}{2}\left[1 - 3 L\gamma - (3 + 3 L\gamma)\delta_2 \right]\|\nabla f(x^t)\|^2 \\
    && +\frac{2\gamma}{2}(1 + 3 L\gamma)\delta_1 +  \frac{3L\gamma^2}{2 G(t) }\sigma^2 + \gamma\zeta(N).
    \end{eqnarray*}
    We first fix $\delta_2 \leqslant \frac{1}{12}$. Finally, by choosing $\gamma \leqslant \frac{1}{13L} \leqslant \frac{1}{12L(1 + \delta_2)}$, utilizing $G(t) \geqslant \underset{t\leqslant T}{\min}~G(t) = \widetilde{G}$ and summing over the iterations, we obtain the constraint in Theorem:
    \begin{eqnarray*}
        \frac{1}{T}\sum\limits_{t = 0}^{T-1} \mathbb E \|\nabla f(x^t)\|^2 &\leqslant& \frac{4\mathbb E\left[\hat{f}\left(x^{0}\right) - \hat{f}\left(\hat{x}^{*}\right)\right]}{\gamma T} + 3\delta_1 + \frac{6L\gamma}{\widetilde{G}}\sigma^2 + 4\zeta(N).
    \end{eqnarray*}
\end{proof}
\end{theorem}
\begin{corollary}\label{ITRSGDpartialcorollary}
    Under the assumptions of \text{Theorem} \ref{theorempartial}, for solving the problem \eqref{eq:setting_dist}, after $T$ iterations with $\gamma \leqslant \min\left\{\frac{1}{13L}, \frac{\sqrt{2\mathbb E\left[\hat{f}\left(x^{0}\right) - \hat{f}\left(\hat{x}^{*}\right)\right] \widetilde{G}}}{\sigma\sqrt{3LT}}\right\}$, the following holds:
\begin{align*}
        \frac{1}{T}\sum\limits_{t=0}^{T-1}\mathbb{E} \left\|\nabla f(x^t)\right\|^2 =&   \mathcal{O}\left(\frac{\mathbb E\left[\hat{f}\left(x^{0}\right) - \hat{f}\left(\hat{x}^{*}\right)\right] L}{T} + \frac{\sigma\sqrt{\mathbb E\left[\hat{f}\left(x^{0}\right) - \hat{f}\left(\hat{x}^{*}\right)\right] L}}{\sqrt{T \widetilde{G} }} + \delta_1 + \zeta(N)\right).
    \end{align*}
\end{corollary}
\begin{remark}
    We adopted our method to the scenario of partial participation. It is natural that it requires stronger assumption: we need at least one honest client from that taking participation at each iteration, but not always the same.
\end{remark}

\section{SimBant} \label{sec:fine_tuned}
We also propose another method for dealing with Byzantine attacks - \texttt{SimBant}. Here we use the concept of the trial function in a different way: the key is that we look not at the loss function $\hat{f}$, as in Algorithms \ref{alg2}, \ref{alg1}, but at the output of the model $m$ on the trial data $\hat{D}$. In particular, the calculation of trust scores $w_i$ is now based on how the outputs of the model parameters obtained on the server (a guaranteed honest device, without losing generality, we can assume that it has index 1) and on the device are similar to each other. The server is validated in pairs with each device and the trust scores are calculated for each device relative to the prediction classes based on how similar they are between the server and the device. And only on the basis of the trust scores, we give a weights to devices' trained models and calculate the new state of the general model. Let us denote by $g^t$ the model that the server sends to the devices at step $t$, $g_i^t, i\in\{2,\ldots, n\}$ - devices' trained models, $g_1^t$ the trained server model. Then the trust score for $i$-th device is function $\alpha_i\rightarrow sim(m(x^t - \gamma g_i^t, \mathcal{\hat{D}}),m(x^t - \gamma g_1^t, \mathcal{\hat{D}}))$, which measures the similarity of predictions of server and device models, where $x^t - \gamma g_i^t$ and $x^t - \gamma g_1^t$ are new parameters on the $i$-th device and the server, respectively, $sim$ is responsible for the measure of the similarity of the outputs.
We do not provide the theory for this method but validate it in practice (see Section \ref{sec:experiments}).
\restylefloat{algorithm}
\begin{algorithm}[ht]
\caption{\texttt{SimBant}}\label{alg:finetuned}
\begin{algorithmic}[1]
\STATE \textbf{Input:} Starting point $x^0 \in \mathbb{R}^d$
\STATE \textbf{Parameters:} Stepsize $\gamma > 0$, error accuracy $\delta$
\FOR{$t = 0, 1, 2, \ldots, T-1$}
    \STATE Server sends $x^t$ to each worker
    \FORALL{workers $i = 1, 2, \ldots, n$ in parallel}
        \STATE Generate $\xi_i^t$ independently
        \STATE Compute stochastic gradient $g_i^t = g_i(x^t, \xi_i^t)$
        \STATE Send $g_i^t$ to server
    \ENDFOR
    \STATE $\omega_i^{t} = (1-\beta)\omega_i^{t-1} + \beta\frac{sim(m(x^t - \gamma g_i^t, \hat{D}),m(x^t - \gamma g_1^t, \hat{D}))}{\sum\nolimits_{j = 1}^n sim(m(x^t - \gamma g_j^t, \hat{D}),m(x^t - \gamma g_1^t, \hat{D}))}$ 
    \STATE $x^{t+1} = x^t - \gamma\sum\nolimits_{i=1}^n \omega_i^{t} g_i^t$ 
\ENDFOR
\STATE \textbf{Output:} $\frac{1}{T}\sum\limits_{t = 0}^{T-1} x^t$
\end{algorithmic}
\end{algorithm}

\end{document}